%% file: main.tex
\pdfoutput=1
\documentclass{article}

\PassOptionsToPackage{numbers, compress}{natbib}
\usepackage[final]{neurips_2022}

\usepackage[utf8]{inputenc} 
\usepackage[T1]{fontenc}    
\usepackage{hyperref}       
\usepackage{url}            
\usepackage{booktabs}       
\usepackage{amsfonts}       
\usepackage{nicefrac}       
\usepackage{microtype}      
\usepackage{xcolor}         

\usepackage{enumitem}
\usepackage{multirow}
\usepackage{makecell}
\usepackage{chngpage}
\usepackage{graphicx}

\input{math_commands.tex}

\usepackage{amsthm}
\theoremstyle{plain}
\newtheorem{theorem}{Theorem}[section]
\newtheorem{proposition}[theorem]{Proposition}
\newtheorem{lemma}[theorem]{Lemma}
\newtheorem{fact}[theorem]{Fact}
\newtheorem{corollary}[theorem]{Corollary}
\theoremstyle{definition}
\newtheorem{definition}[theorem]{Definition}

\newtheorem{example}[theorem]{Example}
\newtheorem{remark}[theorem]{Remark}

\title{Rethinking Lipschitz Neural Networks and Certified Robustness: A Boolean Function Perspective}

\author{%
    \textbf{Bohang Zhang}$^{1}$\quad \textbf{Du Jiang}$^{1}$
    \quad  \textbf{Di He}$^{1}$\quad \textbf{Liwei Wang}$^{1,2,\dag}$\\
    \hspace{-5pt}$^1$Key Laboratory of Machine Perception, MOE, School of Artificial Intelligence, Peking University\\
    \quad $^2$Center for Data Science, Peking University \quad $^\dag$Corresponding author\\
    \footnotesize\texttt{ \{zhangbohang,dujiang,dihe\}@pku.edu.cn},\quad\texttt{ wanglw@cis.pku.edu.cn}
}

\begin{document}

\maketitle

\vspace{-5pt}
\begin{abstract}
    Designing neural networks with bounded Lipschitz constant is a promising way to obtain certifiably robust classifiers against adversarial examples. However, the relevant progress for the important $\ell_\infty$ perturbation setting is rather limited, and a principled understanding of how to design expressive $\ell_\infty$ Lipschitz networks is still lacking. In this paper, we bridge the gap by studying certified $\ell_\infty$ robustness from a novel perspective of representing Boolean functions. We derive two fundamental impossibility results that hold for any standard Lipschitz network: one for robust classification on finite datasets, and the other for Lipschitz function approximation. These results identify that networks built upon norm-bounded affine layers and Lipschitz activations intrinsically lose expressive power even in the two-dimensional case, and shed light on how recently proposed Lipschitz networks (e.g., GroupSort and $\ell_\infty$-distance nets) bypass these impossibilities by leveraging \emph{order statistic functions}. Finally, based on these insights, we develop a unified Lipschitz network that generalizes prior works, and design a practical version that can be efficiently trained (making certified robust training free). Extensive experiments show that our approach is scalable, efficient, and consistently yields better certified robustness across multiple datasets and perturbation radii than prior Lipschitz networks. Our code is available at \texttt{\href{https://github.com/zbh2047/SortNet}{https://github.com/zbh2047/SortNet}}.
\end{abstract}

\section{Introduction}
\label{sec:intro}
Modern neural networks, despite their great success in various applications \cite{he2016deep,devlin2019bert}, typically suffer from a severe drawback of lacking adversarial robustness. In classification tasks, given an image $\vx$ correctly classified by a neural network, there often exists a small adversarial perturbation $\boldsymbol\delta$, making the perturbed image $\vx + \boldsymbol\delta$ look indistinguishable to humans but fool the network to predict a wrong class with high confidence \cite{szegedy2013intriguing,biggio2013evasion}. 

It is well-known that the adversarial robustness of a neural network is closely related to its Lipschitz continuity \citep{cisse2017parseval,tsuzuku2018lipschitz} (see Section \ref{sec:preliminaries}). Accordingly, training neural networks with bounded Lipschitz constant has been considered a promising way to address the problem. A variety of works studied Lipschitz architectures for the ordinary Euclidean norm \citep{tsuzuku2018lipschitz,trockman2021orthogonalizing,leino21gloro,singla2021skew}, and recent works even established state-of-the-art (deterministic) certified $\ell_2$ robustness \citep{singla2022improved,meunier2022dynamical}. However, when it comes to the more critical (realistic) $\ell_\infty$ perturbation setting, the progress seems to be rather limited. In fact, standard Lipschitz ReLU networks have been shown to perform poorly in terms of $\ell_\infty$ robustness \citep{tsuzuku2018lipschitz,huster2018limitations,anil2019sorting}. While other more advanced Lipschitz networks have been proposed \citep{anil2019sorting,cohen2019universal}, the achieved results are still not satisfactory even on simple datasets like MNIST. Until recently, Zhang et al. \citep{zhang2021towards,zhang2022boosting} designed a quite \emph{unusual} Lipschitz network based on a heuristic choice of the \emph{$\ell_\infty$-distance function}, which surprisingly established state-of-the-art certified $\ell_\infty$ robustness on multiple datasets over prior works. Yet, it remains unclear why previous Lipschitz networks typically failed, and what is the essential reason behind the success of this particular $\ell_\infty$-distance network structure. 

\textbf{Theoretical contributions}. In this work, we systematically investigate how to design expressive Lipschitz neural networks (w.r.t. $\ell_\infty$-norm) through the novel lens of representing discrete Boolean functions, which provides a deep understanding on the aforementioned problems. Specifically, we first figure out a fundamental limitation of standard Lipschitz networks in representing a class of logical operations called \emph{symmetric Boolean functions}  (SBF), which comprises the basic logical AND/OR as special cases. We prove that for any non-constant SBF of $d$ variables, there exists a finite dataset of size $\mathcal O(d)$ such that the certified $\ell_\infty$ robust radius must vanish as $\mathcal O(1/d)$ for any classifier induced by a standard Lipschitz network. Remarkably, since logical AND/OR operations correspond to perhaps the most basic classifiers, our result indicates an intrinsic difficulty of such networks in fitting high-dimensional real-world datasets with guaranteed certified $\ell_\infty$ robustness.

Our analysis can be readily extended into the Lipschitz function approximation setting. We point out the relationship between \emph{monotonic} SBF and the \emph{order statistics} (which are 1-Lipschitz functions), and then prove that any $d$-dimensional order statistic (including the max/min function) on a compact domain cannot be approximated by standard Lipschitz networks with error $\mathcal O(1-1/d)$, regardless of the network size. This impossibility result is significant in that: $(\mathrm{i})$ it applies to all Lipschitz activations (thus extending prior works \citep{anil2019sorting,huster2018limitations}), $(\mathrm{ii})$ it resolves an open problem raised recently in \citep{neumayer2022approximation}, and $(\mathrm{iii})$ a \emph{quantitative} lower bound of approximation error is established.

Equipped by the above impossibility results, we proceed to examine two advanced Lipschitz architectures: the GroupSort network \citep{anil2019sorting} and the recently proposed $\ell_\infty$-distance net \citep{zhang2021towards,zhang2022boosting}. We find that besides the linear operation, both networks incorporate other Lipschitz aggregation operations into the neuron design, especially the order statistic functions, thus shedding light on how they work. However, for the MaxMin network \citep{anil2019sorting} --- a computationally efficient version of the GroupSort network implemented in practice, representing Boolean functions and order statistics is possible only when the network is very \emph{deep}. In particular, we prove that representing certain $d$-dimensional Boolean functions requires a depth of $\Omega(d)$, implying that shallow MaxMin networks are not Lipschitz-universal function approximators. In contrast, we show a \emph{two-layer} $\ell_\infty$-distance net suffices to represent any order statistic function on a compact domain or even all Boolean functions. This strongly justifies the empirical success of $\ell_\infty$-distance net over GroupSort (MaxMin) networks.

\textbf{Practical contributions}. Our theoretical insights can also guide in designing better Lipschitz network architectures. Inspired by the importance of order statistics, we propose a general form of Lipschitz network, called SortNet, that extends both GroupSort and $\ell_\infty$ distance networks and incorporates them into a unified framework. Yet, the full-sort operation is computationally expensive and leads to optimization difficulties (as with the GroupSort network). We further propose a specialized SortNet that can be efficiently trained, by assigning each weight vector $\vw$ using \emph{geometric series}, i.e. $w_i$ proportional to $\rho^{i}$ for some $0\le \rho<1$. This leads to a restricted version of SortNet but still covers $\ell_\infty$-distance net as a special case. For this particular SortNet, we skillfully derive a stochastic estimation that gives an unbiased approximation of the neuron output without performing sorting operations explicitly. This eventually yields an efficient training strategy with similar cost as training standard networks, thus making certified robust training free. Extensive experiments demonstrate that the proposed SortNet is scalable, efficient, and consistently achieves better certified robustness than prior Lipschitz networks across multiple datasets and perturbation radii. In particular, our approach even scales on a variant of ImageNet, and surpasses the best-known result \citep{xu2020automatic} with a 22-fold decrease in training time thanks to our ``free'' certified training approach.

The contribution and organization of this paper can be summarized as follows:


\begin{itemize}[topsep=0pt,leftmargin=30pt]
\setlength{\itemsep}{0pt}
    \item We develop a systematic study for the expressive power of Lipschitz neural networks using the tools of Boolean function theory. We prove the impossibility results of standard Lipschitz networks in two settings: a) certified $\ell_\infty$ robustness on discrete datasets (Section \ref{sec:discrete}); b) Lipschitz function approximation (Section \ref{sec:function_approximation}).
    \item We provide insights into how recently proposed networks can bypass the impossibility results. In particular, we show that a \emph{two-layer} $\ell_\infty$-distance net can precisely represent any Boolean functions, while \emph{shallow} GroupSort networks cannot (Section \ref{sec:order_statistic}). 
    \item We propose SortNet, a Lipschitz network that generalizes GroupSort and $\ell_\infty$-distance net. For a special type of SortNet, we derive a stochastic training approach that bypasses the difficulties in calculating sorting operations explicitly and makes certified training free (Section \ref{sec:sortnet}).
    \item Extensive experiments demonstrate that SortNet exhibits better certified robustness on several benchmark datasets over baseline methods with high training efficiency (Section \ref{sec:experiments}).
\end{itemize}

\section{Related Work}
Extensive studies have been devoted to developing neural networks with certified robustness guarantees. Existing approaches can be mainly divided into the following three categories.

\textbf{Certified defenses for standard networks}. A variety of works focus on establishing certified robustness for standard neural networks. However, exactly calculating the certified radius of a standard ReLU network is known to be NP-hard \citep{katz2017reluplex}. Researchers thus developed a class of relaxation-based approaches that provide a tight lower bound estimate of the certified robustness efficiently. These approaches typically use convex relaxation to calculate a bound of the neuron outputs under input perturbations layer by layer  \citep{wong2018provable,wong2018scaling,weng2018towards,singh2018fast,mirman2018differentiable,gehr2018ai2,wang2018efficient,zhang2018efficient}. See also \citep{balunovic2020Adversarial,krishnamurthy2018dual,dvijotham2020efficient,raghunathan2018certified,raghunathan2018semidefinite,xiao2019training,croce2019provable,lee2020lipschitz,wang2021beta,huang2021training,shi2022efficient} for more advanced approaches. However, most of these works suffer from high computational costs and are hard to scale up to large datasets. Currently, the only scalable convex relaxation approach is based on \emph{interval bound propagation} (IBP) \citep{mirman2018differentiable,gowal2018effectiveness,zhang2020towards,xu2020automatic,shi2021fast}, but the produced bound is known to be loose \citep{salman2019convex}, and a recent study showed that IBP cannot achieve enough certified robustness on simple datasets for any standard ReLU network \citep{mirman2021fundamental}.

\textbf{Certified defenses using Lipschitz networks}. On the other hand, Lipschitz networks inherently imply certified robustness, resulting in a \emph{much simpler} certification process based on the output margin (see Proposition \ref{thm:lipschitz}). Yet, most prior works can only handle the $\ell_2$-norm Lipschitz situation by leveraging specific mathematical properties such as the spectral norm \citep{cisse2017parseval,yoshida2017spectral,gouk2018regularisation,tsuzuku2018lipschitz,farnia2019generalizable,qian2019lnonexpansive,anil2019sorting,leino21gloro,meunier2022dynamical} or orthogonality of weight matrices \citep{li2019preventing,trockman2021orthogonalizing,singla2021skew,singla2022improved}. For the $\ell_\infty$-norm, standard Lipschitz networks were shown to give only a vanishingly small certified radius \citep{tsuzuku2018lipschitz}. Huster et al. \citep{huster2018limitations} found that standard Lipschitz ReLU networks cannot represent certain simple functions such as the absolute value, which inspired the first expressive Lipschitz architecture called the GroupSort network \citep{anil2019sorting}. Since then, GroupSort has been extensively investigated \citep{cohen2019universal,tanielian2021approximating}, but its performance is still much worse than the above relaxation-based approaches even on MNIST. Recently, Zhang et al. \citep{zhang2021towards,zhang2022boosting} first proposed a practical 1-Lipschitz architecture w.r.t. $\ell_\infty$-norm based on a special neuron called the $\ell_\infty$-distance neuron, which can scale to TinyImageNet with state-of-the-art certified robustness over relaxation-based approaches. However, despite its practical success, it is rather puzzling how such a simple architecture can work while prior approaches all failed. Answering this question may require an in-depth re-examination of Lipschitz networks (w.r.t. $\ell_\infty$-norm), which is the focus of this paper.

\textbf{Certified defenses via randomized smoothing}. As a rather different and parallel research line, randomized smoothing typically provides \emph{probabilistic} certified $\ell_2$ robustness guarantees. Due to the wide applicability, randomized smoothing has been scaled up to ImageNet and achieves state-of-the-art certified accuracy for $\ell_2$ perturbations \citep{lecuyer2019certified,li2019certified,cohen2019certified,salman2019provably,zhai2019macer,jeong2020consistency}. However, certifying robustness with high probability requires sampling a large number of noisy inputs (e.g., $10^5$) for a single image, leading to a high computational cost at inference. Moreover, theoretical results pointed out that it cannot achieve non-trivial certified $\ell_\infty$ robustness if the perturbation radius is larger than $\Omega(d^{-1/2})$ where $d$ is the input dimension \citep{yang2020randomized,blum2020random,kumar2020curse,wu2021completing}.

\section{The Expressive Power of Lipschitz Neural Networks}
\label{sec:theory}
\subsection{Preliminaries}
\label{sec:preliminaries}
\textbf{Notations}. We use boldface letters to denote vectors (e.g., $\vx$) or vector functions (e.g., $\vf$), and use $x_i$ (or $f_i$) to denote its $i$-th element. For a unary function $\sigma$, $\sigma(\vx)$ applies $\sigma(\cdot)$ element-wise on vector $\vx$. The $\ell_p$-norm ($p\ge 1$) and $\ell_\infty$-norm of a vector $\vx$ are defined as $\|\vx\|_p=(\sum_i |x_i|^p)^{1/p}$ and $\|\vx\|_\infty=\max_i |x_i|$, respectively. The matrix $\infty$-norm is defined as $\|\mathbf W\|_\infty=\max_i \|\mathbf W_{i,:}\|_1$ where $\mathbf W_{i,:}$ is the $i$-th row of the matrix $\mathbf W$. The $k$-th largest element of a vector $\vx$ is denoted as $x_{(k)}$. We use $[n]$ to denote the set $\{1,\cdots,n\}$, and use $\ve_i$ to denote the unit vector with the $i$-th element being one. We adopt the big O notations by using $\mathcal O(\cdot)$, $\Omega(\cdot)$, and $\Theta(\cdot)$ to hide universal constants.

\textbf{Lipschitzness}. A mapping $\vf:\mathbb R^n\to \mathbb R^m$ is said to be $L$-Lipschitz continuous w.r.t. norm $\|\cdot\|$ if for any pair of inputs $\vx_1,\vx_2\in\mathbb R^n$,
\begin{equation}
\label{eq:lipschitz}
    \|\vf(\vx_1)-\vf(\vx_2)\|\le L \|\vx_1-\vx_2\|.
\end{equation}
If the mapping $\vf$ represented by a neural network has a small Lipschitz constant $L$, then (\ref{eq:lipschitz}) implies that the change of network output can be strictly controlled under input perturbations, resulting in \emph{certified} robustness guarantees as shown in the following proposition.
\begin{proposition}
\label{thm:lipschitz}
\emph{(Certified robustness of Lipschitz networks)} For a neural network $\vf$ with Lipschitz constant $L$ under $\ell_p$-norm $\|\cdot\|_p$, define the resulting classifier $g$ as $g(\vx):=\arg\max_k f_k(\vx)$ for an input $\vx$. Then $g$ is provably robust under perturbations $\|\boldsymbol\delta\|_p< \frac c L\operatorname{margin}(\vf(\vx))$, i.e.
\begin{equation}
\setlength{\belowdisplayskip}{4pt}
\setlength{\abovedisplayskip}{6pt}
\label{eq:margin_certification}
    g(\vx+\boldsymbol\delta)=g(\vx) \quad\text{ for all }\bm\delta \text{ with } \|\boldsymbol\delta\|_p<  c/L\cdot \operatorname{margin}(\vf(\vx)).
\end{equation}
Here $c=\sqrt[p] 2/2$ is a constant depending only on the norm $\|\cdot\|_p$, which is $1/2$ for the $\ell_\infty$-norm, and $\operatorname{margin}(\vf(\vx))$ is the margin between the largest and second largest output logits.
\end{proposition}
\vspace{-5pt}

The proof of Proposition \ref{thm:lipschitz} is simple and can be found in Appendix \ref{sec:proof_lipschitz} or \citep[Appendix P]{li2019preventing}. It can be seen that the robust radius is inversely proportional to the Lipschitz constant $L$. 

\textbf{Standard Lipschitz networks}. Throughout this paper, we refer to standard neural networks as neural networks formed by affine layers (e.g., fully-connected or convolutional layers) and element-wise activation functions. Based on the Lipschitz property of composite functions, most prior works enforce the 1-Lipschitzness of a multi-layer neural network by constraining each layer to be a 1-Lipschitz mapping. For the $\ell_\infty$-norm, it is further equivalent to constraining the weight matrices to have bounded $\infty$-norm, plus using Lipschitz activation functions \citep{anil2019sorting}, which can be formalized as
\begin{equation}
\setlength{\belowdisplayskip}{4pt}
\setlength{\abovedisplayskip}{4pt}
\label{eq:standatd_lipschitz}
    \vx^{(l)}=\sigma^{(l)}(\mathbf W^{(l)}\vx^{(l-1)}+\vb^{(l)})\quad \text{s.t. }\|\mathbf W^{(l)}\|_\infty\le 1 \text{ and } \sigma^{(l)} \text{ being 1-Lipschitz},\  l\in [M].
\end{equation}
Here $M$ is the number of layers and usually $\sigma^{(M)}(x)=x$ is the identity function. The network takes $\vx^{(0)}:=\vx$ as the input and outputs $\vx^{(M)}$. For $K$-class classification problems, $\vx^{(M)}\in \mathbb R^K$ and the network predicts the class $g(\vx):=\arg\max_{k\in [K]} x^{(M)}_k$. We refer to the resulting network as a standard Lipschitz network.

\subsection{Certified robustness on discrete Boolean datasets}
\label{sec:discrete}
In this section, we will construct a class of counterexamples for which certified $\ell_\infty$ robustness can be arbitrarily poor using standard Lipschitz networks. We focus on the \emph{Boolean dataset}, a discrete dataset where both inputs and labels are Boolean-valued and the relationship between inputs and their labels $(\vx^{(i)}, y^{(i)})\in \{0,1\}^d\times \{0,1\}$ can be described using a Boolean function $g^\text{B}:\{0,1\}^d\to \{0,1\}$. The key motivation lies in the finding that Boolean vectors correspond to the vertices of a $d$-dimensional hypercube, and thus are geometrically related to the $\ell_\infty$-distance metric. In particular, the $\ell_\infty$-distance between any two different data points in a Boolean dataset is always 1, which means that the dataset is \emph{well-separated}. This yields the following proposition, stating that it is always possible to achieve \emph{optimal} certified $\ell_\infty$ robustness on Boolean datasets by using Lipschitz classifiers.

\begin{proposition}
\label{thm:nearest_neighbor}
For any Boolean dataset $\mathcal D=\{(\vx^{(i)},y^{(i)})\}_{i=1}^n$, there exists a classifier $\hat g:\mathbb R^d\to \{0,1\}$ induced by a 1-Lipschitz mapping $\hat \vf:\mathbb R^d\to \mathbb R^2$, such that $\hat g$ can fit the whole dataset with $\operatorname{margin}(\hat \vf(\vx^{(i)}))=1$ $\forall i\in[n]$, thus achieving a certified $\ell_\infty$ radius of $1/2$ by Proposition \ref{thm:lipschitz}.
\end{proposition}
\vspace{-5pt}

The key observation in the proof (Appendix \ref{sec:proof_nearest_neighbor}) is that one can construct a so-called \emph{nearest neighbor classifier} that achieves a large margin on the whole dataset and is \emph{1-Lipschitz}. Based on Proposition \ref{thm:nearest_neighbor}, it is natural to ask whether standard Lipschitz networks of the form (\ref{eq:standatd_lipschitz}) can perform well on Boolean datasets. Unfortunately, we show it is not the case, even on a class of simple datasets constructed using \emph{symmetric} Boolean functions.

\begin{definition}
\label{def:stmmetric_boolean}
A Boolean function $g^\text{B}:\{0,1\}^d\to \{0,1\}$ is symmetric if it is invariant under input permutation, i.e. $g^\text{B}(x_1,\cdots,x_d)=g^\text{B}(x_{\pi(1)},\cdots,x_{\pi(d)})$ for any $ \vx\in\{0,1\}^d$ and $\pi\in S_d$.
\end{definition}
\begin{example}
Two of the most basic operations in Boolean algebra are the logical AND/OR, both of which belong to the class of symmetric Boolean functions. Other important examples include the exclusive-or (XOR, also called the parity function), NAND, NOR, and the majority function (or generally, the threshold functions). See Appendix \ref{sec:symmetirc_Boolean} for a detailed description of these examples.
\end{example}
\begin{theorem}
\label{thm:discrete}
For any non-constant symmetric Boolean function $g^\text{B}:\{0,1\}^d\to \{0,1\}$, there exists a Boolean dataset with labels $y^{(i)}=g^\text{B}(\vx^{(i)})$, such that no standard Lipschitz network can achieve a certified $\ell_\infty$ robust radius larger than $1/2d$ on the dataset.
\end{theorem}
\vspace{-5pt}

\textbf{Implications}. Theorem \ref{thm:discrete} shows that the certified radius of standard Lipschitz networks must vanish as dimension $d$ grows, which is in stark contrast to the constant radius given by Proposition \ref{thm:nearest_neighbor}. Also, note that the logical AND/OR functions are perhaps the most basic classifiers (which simply make predictions based on the existence of binary input features). It is thus not surprising to see that standard Lipschitz networks perform poorly on real-world datasets (e.g., even the simple MNIST dataset where input pixels are almost Boolean-valued (black/white) \citep{tsuzuku2018lipschitz}).


We give an elegant proof of Theorem \ref{thm:discrete} in Appendix \ref{sec:proof_discrete}, where we also prove that the bound of $1/2d$ is \emph{tight} in Proposition \ref{thm:discrete_tight}. Moreover, we discuss the sample complexity by proving that a dataset of size $\mathcal O(d)$ already suffices to give $\mathcal O(1/d)$ certified radius (Corollary \ref{thm:discrete_efficient}). To the best of our knowledge, Theorem \ref{thm:discrete} is the first impossibility result that targets the certified $\ell_\infty$ robustness of standard Lipschitz networks with an \emph{quantitative} upper bound on the certified radius. In the next section, we will extend our analysis to the function approximation setting and make discussions with literature results \citep{huster2018limitations,anil2019sorting}.

\subsection{Lipschitz function approximation}
\label{sec:function_approximation}
Classic approximation theory has shown that standard neural networks are universal function approximators \citep{cybenko1989approximation,leshno1993multilayer}, in that they can approximate any continuous function on a compact domain arbitrarily well. For 1-Lipschitz neural networks, an analogous question is whether they can approximate all 1-Lipschitz functions accordingly. Unfortunately, the result in Section \ref{sec:discrete} already implies a negative answer. Indeed, by combining Proposition \ref{thm:nearest_neighbor} and Theorem \ref{thm:discrete}, $\hat \vf$ is clearly a 1-Lipschitz function that cannot be approximated by any standard Lipschitz network.

To gain further insights into the structure of unrepresentable 1-Lipschitz functions, let us consider the \emph{continuousization} of discrete Boolean functions. For the symmetric case, one needs to find a class of 1-Lipschitz continuous functions that are also invariant under permutations. It can be found that a simple class of candidates is the \emph{order statistics}, i.e. the $k$-th largest element of a vector. One can check that the $k$-th order statistic $x_{(k)}$ is indeed 1-Lipschitz and is precisely the continuousization of the $k$-threshold Boolean function defined as $g^{\text{B},k}(\vx):=\mathbb I(\sum_i x_i\ge k)$. In particular, $x_{(1)}=\max_i x_i$ and $x_{(d)}=\min_i x_i$ corresponds to the logical OR/AND functions, respectively. Importantly, note that any Boolean function that is both symmetric and \emph{monotonic} is a $k$-threshold function, and vice versa (Appendix \ref{sec:symmetirc_Boolean}). Therefore, $k$-threshold functions can be regarded as the most elementary Boolean functions, suggesting that the ability to express order statistics is necessary for neural networks.

Unfortunately, using a similar analysis as the previous section, we have the following theorem:

\begin{theorem}
\label{thm:function_approximation}
Any standard Lipschitz network $f:\mathbb R^d\to \mathbb R$ cannot approximate the simple 1-Lipschitz function $\vx\to x_{(k)}$ for arbitrary $k\in [d]$ on a bounded domain $\mathcal K= [0,1]^d$ if $d\ge 2$. Moreover, there exists a point $\widehat\vx\in \mathcal K$, such that
\begin{equation}
\setlength{\belowdisplayskip}{1pt}
\setlength{\abovedisplayskip}{2pt}
\label{eq:function_approximation_error}
    |f(\widehat\vx)-\widehat x_{(k)}|\ge \frac 1 2 - \frac 1 {2d}.
\end{equation}
\end{theorem}
\vspace{-2pt}

We give a proof in Appendix \ref{sec:proof_function_approximation}. The above theorem indicates that order statistics cannot be \emph{uniformly} approximated using any standard Lipschitz network regardless of the network size. Moreover, note that the trivial constant function $\tilde f(\vx)=1/2$ already achieves an approximation error of $1/2$ uniformly on $\mathcal K$, implying that standard Lipschitz networks can hardly improve upon trivial solutions. 

\begin{remark}
Theorem \ref{thm:function_approximation} can be easily generalized to weaker forms of non-uniform approximation, e.g., using the $\ell_p$-norm as distance metrics \citep{pinkus1999approximation}, by proving that there exists a hypercube $\mathcal B_\infty(\widehat\vx)$ centered at $\widehat \vx$ with length $\Theta(1)$, such that $|f(\tilde\vx)-\tilde x_{(k)}|\ge \Theta(1)$ holds for all $\tilde x\in \mathcal B_\infty(\widehat \vx)$ when $d\ge 2$. See Corollary \ref{thm:function_approximation_corollary} for details.
\end{remark}
\vspace{-4pt}

\textbf{Discussion with prior works} \citep{anil2019sorting,huster2018limitations,neumayer2022approximation}.
The work of Anil et al. also gave negative results on the expressive power of standard Lipschitz networks\footnote{The result is described w.r.t. $\ell_2$-norm, but with some effort, it can be extended to the $\ell_\infty$-norm case.} \citep[Theorem 1]{anil2019sorting}. They proved a different (weaker) version of Theorem \ref{thm:function_approximation}, showing that if the activation function $\sigma$ is \emph{monotonic}, the network cannot \emph{precisely} represent non-linear 1-Lipschitz functions whose gradient norm is 1 almost everywhere (e.g., the absolute value function proved before by \citep{huster2018limitations}). They did not give a quantitative approximation error. The intuition is that any monotonic non-linear 1-Lipschitz activation (e.g., ReLU) must have regions with slopes less than 1, leading to \emph{gradient attenuation} during backpropagation. The authors thus attributed the reason to the activation function, which is not \emph{gradient-norm preserving} (GNP). However, such an explanation is still not fully satisfactory, as GNP can be simply achieved using a non-monotonic activation (e.g., $\sigma(x)=|x|$). Consequently, one may expect that a standard Lipschitz network built on a suitable (non-monotonic) activation function can have sufficient expressive power. Such an idea is recently explored in \citep{neumayer2022approximation}, where the authors proved that using a general 1-Lipschitz piecewise linear activation with 3 linear regions, the corresponding network achieves the maximum expressive power compared with other Lipschitz activations and can approximate any \emph{one-dimensional} 1-Lipschitz function. They pose the high dimension setting as an open problem.

Unfortunately, Theorem \ref{thm:function_approximation} addressed the open problem with negative answer, stating that such networks are not expressive even for the two-dimensional setting. It also implies that GNP is not \emph{sufficient} to explain the failure of standard Lipschitz networks. Instead, we draw a rather different conclusion, arguing that the lack of expressiveness is due to the inability of \emph{weight-constrained} affine transformations to perform basic Boolean operations (even in two dimensions). A further justification is given in Section \ref{sec:order_statistic}. Finally, compared with Anil et al. \citep{anil2019sorting}, the form of Theorem \ref{thm:function_approximation} is more fundamental, in the sense that it does not make assumptions on the activation function, and it gives a quantitative error bound on the approximation that is arbitrarily close to the plain fit $f(\vx)=1/2$.


\vspace{-1.5pt}

\subsection{Investigating more advanced Lipschitz networks}
\vspace{-1.5pt}

\label{sec:order_statistic}
Seeing the above impossibility results, we then examine two representative works of (non-standard) Lipschitz networks in literature: the GroupSort network \citep{anil2019sorting} and the recently proposed $\ell_\infty$-distance net \citep{zhang2021towards,zhang2022boosting}. Notably, both networks are Lipschitz-universal function approximators and thus fully expressive. The GroupSort network makes minimum changes to standard Lipschitz networks (\ref{eq:standatd_lipschitz}), by replacing element-wise activation $\sigma$ with GroupSort layers. GroupSort partitions the input vector into groups, sorts the sub-vector of each group in descending order, and finally concatenates the resulting sub-vectors. Since sorting is computationally expensive, the authors considered a practical version of GroupSort with a group size of 2, called MaxMin, which simply calculates the maximum and minimum pair by pair \citep{chernodub2016norm}. $\ell_\infty$-distance net, on the other hand, is fundamentally different from standard Lipschitz networks. Each neuron in an $\ell_\infty$-distance net is designed based on the $\ell_\infty$-distance function $y=\|\vx-\vw\|_\infty+b$ (with parameters $\{\vw,b\}$). Despite the somewhat unusual structure, $\ell_\infty$-distance net has been shown to substantially outperform GroupSort (MaxMin) in terms of certified $\ell_\infty$ robustness according to \citep{zhang2021towards}, a puzzling thing to be understood.

We provide a possible explanation for this. We find that both networks incorporate order statistics into the neuron design, either explicitly (GroupSort) or implicitly ($\ell_\infty$-distance net), thus bypassing the impossibility result in Theorem \ref{thm:function_approximation}. Indeed, the sorting operations in GroupSort explicitly calculate order statistics. As for $\ell_\infty$-distance net, we show its basic neuron can implicitly represent the max function on a bounded domain, by assigning the weight $\vw=-c\mathbf 1$ and the bias $b=-c$ with a sufficiently large constant $c$:
\begin{equation}
\setlength{\belowdisplayskip}{5pt}
\setlength{\abovedisplayskip}{5pt}
    \textstyle y=\|\vx-\vw\|_\infty+b=\max_i |x_i-(-c)|-c=\max_i x_i\quad \text{if }c\ge \max_i -x_i,
\end{equation}
and thus can represent the logical OR operation. In general, we have the following theorem:
\begin{theorem}
\label{thm:ell_inf_net}
A two-layer $\ell_\infty$-distance net can exactly represent the following functions: $(\mathrm{i})$ any discrete Boolean function; $(\mathrm{ii})$ any continuous order-statistic function on a compact domain.
\end{theorem}
\vspace{-4pt}

We give a proof in Appendix \ref{sec:proof_ell_inf_net}. Our proof uses the fundamental result in Boolean algebra that any Boolean function can be written in its disjunctive normal form (DNF, see Appendix \ref{sec:dnf}), which can be further reduced to using only the composition of logical OR operations of \emph{literals} and thus be realized by a two-layer $\ell_\infty$-distance net. To represent order statistics, we formulate them as nested max-min functions, which can also be realized by a two-layer $\ell_\infty$-distance net. Therefore, the construction in our proof provides a novel understanding of the mechanism behind the success of $\ell_\infty$-distance nets, since \emph{each $\ell_\infty$-distance neuron can be regarded as a basic ``logical gate'' and the whole network can simulate any Boolean circuit}.

For GroupSort networks with a group size $G\ge d$, a similar result holds. However, it is not the case for practically used MaxMin networks ($G=2$), where we have the following impossibility results:

\begin{theorem}
\label{thm:groupsort}
An $M$-layer MaxMin network $f:\mathbb R^d\to \mathbb R$ cannot approximate any $k$-th order statistic function on a bounded domain $\mathcal K= [0,1]^d$ if $M\le \lceil \log_2 d\rceil$ (no matter how wide the network is). Moreover, there exists a point $\widehat \vx\in \mathcal K$, such that
\begin{equation}
\setlength{\belowdisplayskip}{5pt}
\setlength{\abovedisplayskip}{2pt}
    |f(\widehat\vx)- \widehat x_{(k)}|\ge \frac 1 2 -\frac {2^{M-2}} d\ge \frac 1 4\quad \text{if }M\le \lfloor \log_2 d\rfloor.
\end{equation}
\end{theorem}
\begin{theorem}
\label{thm:groupsort_boolean}
Let $M_d$ be the minimum depth such that an $M_d$-layer MaxMin network can represent any (discrete) $d$-dimensional Boolean function. Then $M_d=\Omega(d)$.
\end{theorem}
\vspace{-4pt}
The above theorems show that MaxMin networks must be \emph{very deep} in order to represent Boolean functions and order statistics, which is in stark contrast to Theorem \ref{thm:ell_inf_net}, as a constant depth is sufficient for $\ell_\infty$-distance nets. Based on Theorem \ref{thm:groupsort_boolean}, we directly have the following corollary:
\begin{corollary}
\label{thm:groupsort_universal}
The function class induced by $M_d$-layer MaxMin networks is not a universal approximator to the $d$-dimensional 1-Lipschitz functions if $M_d=o(d)$.
\end{corollary}

The proofs of Theorems \ref{thm:groupsort} and \ref{thm:groupsort_boolean} are deferred to Appendix \ref{sec:proof_groupsort} and \ref{sec:proof_groupsort_boolean}, respectively. In particular, the proof of Theorem \ref{thm:groupsort_boolean} is non-trivial and makes elegant use of Boolean circuit theory, so we present a proof sketch here. The key insight is that for any Boolean function, if it can be expressed by some MaxMin network $f$, then it can be expressed by a special MaxMin network with the same topology as $f$ such that all the weight vectors $\vw$ are sparse with at most one non-zero element, either $1$ or $-1$ (Corollary \ref{thm:groupsort_weight_sparse}). This implies that \emph{weight vectors have no use} in representing Boolean functions and thus MaxMin networks reduce to \emph{2-ary Boolean circuits}, i.e. directed acyclic graphs whose internal nodes are logical gates including NOT and the 2-ary AND/OR. Note that for a 2-ary Boolean circuit that has $M$ layers and outputs a scalar, the number of nodes will not exceed $2^{M+1}-1$ (achieved by a complete binary tree). However, the classic result in Boolean circuit theory (Shannon 1942) showed that for most Boolean functions of $d$ variables, a lower bound on the minimum size of 2-ary Boolean circuits is $\Omega(2^d/d)$ , which thus yields $M=\Omega(d)$ and concludes the proof.

In Appendix, we also discuss the tightness of the above theorems. We prove that a depth of $\mathcal O(\log_2 d)$ is sufficient to represent any order statistic function using Boolean circuit theory (Theorem \ref{thm:groupsort_deep}), and a straightforward construction using DNF shows that a depth of $\mathcal O(d)$ is sufficient to represent any Boolean functions (Proposition \ref{thm:groupsort_upper_bound_boolean}). Thus both theorems are tight.

Unfortunately, training deep MaxMin networks is known to be challenging due to optimization difficulties \citep{cohen2019universal}. Consequently, prior works only use a shallow MaxMin network with no more than 4 layers \citep{anil2019sorting,cohen2019universal}, which severely lacks expressive power. One possible solution is to increase the group size, and several works explored this aspect using toy examples and observed significant benefits empirically \citep{cohen2019universal,tanielian2021approximating}. However, a large group size involves computationally expensive sorting operations and makes the network hard to train \citep{anil2019sorting}, limiting its value in practice.

\section{A Unified Framework of Lipschitz Neural Networks}
\label{sec:sortnet}

The above theoretical results have justified order statistics as a crucial component in representing a class of Boolean functions, shedding light on how GroupSort and $\ell_\infty$-distance net work. Based on these insights, in this section, we will propose a unified framework of Lipschitz networks that take the respective advantage of prior Lipschitz architectures, and then give a practical (specialized) version that enables efficient training.

Consider a general Lipschitz network constructed using the following three basic types of 1-Lipschitz operations: $(\mathrm{i})$ norm-bounded affine transformations, e.g. $y=\vw^{\mathrm{T}}\vx$ ($\|\vw\|_1\le 1$) and $y=x+b$; $(\mathrm{ii})$ 1-Lipschitz unary activation functions $\sigma$; $(\mathrm{iii})$ order statistics. The first two types are extensively used in standard Lipschitz networks, while the last type is motivated by Section \ref{sec:theory} and is of crucial importance. We propose the following network which naturally combines the above components:
\begin{definition}
\label{def:sortnet}
(SortNet) Define an $M$-layer fully-connected SortNet $\vf$ as follows. The network takes $\vx=\vx^{(0)}$ as input, and the $k$-th unit in the $l$-th hidden layer $x^{(l)}_k$ is computed by 
\begin{equation}
\setlength{\belowdisplayskip}{3pt}
\setlength{\abovedisplayskip}{3pt}
\label{eq:sortnet_neuron}
  \begin{aligned}
   x^{(l)}_k=(\vw^{(l,k)})^{\mathrm{T}}\operatorname{sort}(\sigma(\vx^{(l-1)}+\vb^{(l,k)})),\quad
   \text{s.t.}\  \|\vw^{(l,k)}\|_1\le 1,\quad  l\in [M], k \in [d_l]
\end{aligned}  
\end{equation}
where $d_l$ is the size of the $l$-th layer, and $\operatorname{sort}(\vx):=(x_{(1)},\cdots,x_{(d)})^{\mathrm{T}}$ calculates all order statistics of $\vx\in\mathbb R^d$. The network outputs $\vf(\vx)=\vx^{(M)}+\vb^{\text{out}}$. Here $\{\vw^{(l,k)}\}$, $\{\vb^{(l,k)}\}$ and $\vb^{\text{out}}$ are parameters.
\end{definition}
\vspace{-2pt}

It is easy to see that SortNet is 1-Lipschitz w.r.t. $\ell_\infty$-norm. We now show that SortNet is a general architecture that extends both GroupSort and $\ell_\infty$-distance networks into a unified framework.
\begin{proposition}
Any GroupSort network with an arbitrary group size on a compact input domain can be represented by a SortNet with the same topological structure using activation $\sigma(x)=x$.
\end{proposition}
\begin{proposition}
Any $\ell_\infty$-distance net can be represented by a SortNet with the same topological structure by fixing the weights $\vw^{(l,k)}=\ve_1$ and using the absolute-value activation $\sigma(x)=|x|$.
\end{proposition}
\vspace{-3pt}

See Appendix \ref{sec:special_sortnet} for a proof. As a result, SortNet can exploit the respective advantage of these Lipschitz networks. Compared with GroupSort, SortNet can freely use activation functions such as the absolute value, thus easily addressing the problem claimed in \citep{huster2018limitations,anil2019sorting}. Moreover, unlike GroupSort, the bias vector $\vb^{(l,k)}$ in SortNet (\ref{eq:sortnet_neuron}) can be assigned diversely for different neurons in the same layer. In this way, one can control the sorting behavior of each neuron individually by varying the bias value without disturbing the output of other neurons, which is very beneficial (see Appendix \ref{sec:special_sortnet_groupsort} for details). Compared with $\ell_\infty$-distance net, SortNet adds linear transformation and incorporates all order statistics (rather than only the maximum), thus can represent certain functions more effectively.

\textbf{A practical version of SortNet}. As with GroupSort networks, we also design a practical (specialized) version of SortNet which enjoys efficient training. But different from the MaxMin network that reduces the group size, we keep the full-dimensional order statistics as they are crucial for the expressive power (Section \ref{sec:order_statistic}). The key observation is that in (\ref{eq:sortnet_neuron}), the only required computation is the linear combination of order statistics (i.e. $\vw^{\mathrm{T}} \operatorname{sort}(\cdot)$), rather than the entire sorting results (i.e. $\operatorname{sort}(\cdot)$). We find that for certain carefully designed choices of the weight vector $\vw$, there exist efficient approximation algorithms that can give a good estimation of $\vw^{\mathrm{T}} \operatorname{sort}(\cdot)$. In particular, we propose an assignment of the weight vector that follows geometric series, i.e. $w_i$ proportional to $\rho^{i}$, in which case we have the following result:
\begin{proposition}
\label{thm:dropout}
Let $\vw\in \mathbb R^d$ be a vector satisfying $w_i=(1-\rho)\rho^{i-1}, i\in[d]$ for some $0\le \rho<1$. Then for any vector $\vx\in\mathbb R_+^d$ with non-negative elements,
\begin{equation}
\setlength{\belowdisplayskip}{2pt}
\setlength{\abovedisplayskip}{2pt}
\label{eq:expectation_dropout}
    \vw^{\mathrm{T}}\operatorname{sort}(\vx)=\mathbb E_{\vs\sim\operatorname{Ber(1-\rho)}}[\max_i s_i x_i].
\end{equation}
Here $\vs$ is a random vector following independent Bernoulli distribution with probability $1-\rho$ being 1 and $\rho$ being 0.
\end{proposition}
\vspace{-8pt}
\begin{proof}
Without loss of generality, assume $x_1,\cdots,x_d$ are different from each other. Denote $j_1,\cdots,j_d$ as the sorting indices such that $\operatorname{sort}(\vx)=(x_{j_1},\cdots,x_{j_d})$. Then
\begin{align*}
    \textstyle\mathbb E_{\vs\sim\operatorname{Ber(\rho)}}[\max_i s_i x_i]
    =&\textstyle\sum_{k\in [d]}  \operatorname{Pr}_{\vs\sim\operatorname{Ber(\rho)}}\left[\max_{i} s_i x_i=x_{j_k}\right] x_{j_k}\\
    =&\textstyle\sum_{k\in [d]}  \operatorname{Pr}_{\vs\sim\operatorname{Ber(\rho)}}\left[s_{j_k}=1 \text{ and }s_{j_i}=0\  \forall 1\le i<k\right] x_{j_k}\\
    =&\textstyle\sum_{k\in [d]}  (1-\rho)\rho^{k-1} \cdot x_{(k)}=\vw^{\mathrm{T}}\operatorname{sort}(\vx).
\end{align*}

\vspace{-25pt}
\end{proof}
\vspace{-5pt}

It is easy to check that the weight $\vw$ in the above proposition satisfies $\|\vw\|_1\le 1$, which guarantees the Lipschitzness. The non-negative condition on $\vx$ in Proposition \ref{thm:dropout} holds when using a suitable activation function in neuron (\ref{eq:sortnet_neuron}), such as the absolute value function.

Proposition \ref{thm:dropout} suggests that one can use $\max_i s_i x_i$ to give an \emph{unbiased} estimation of $\vw^{\mathrm{T}} \operatorname{sort}(\vx)$. In this way, the expensive sorting operation is avoided and replaced by a max operation, thus significantly reducing the computational cost in training. We give an efficiently GPU implementation for training SortNet in Appendix \ref{sec:gpu_implementaion}. Note that $\vs$ is a random Bernoulli vector, so the above calculation is similar to applying a mask on the input of each neuron, like dropout \citep{srivastava2014dropout}. It means that the introduced stochasticity may further prevent overfitting and benefit generalization performance.

\textbf{Regarding the value of $\rho$}. When $\rho=0$, only the maximum value is taken into the computation and the resulting network can recover the $\ell_\infty$-distance net by choosing the activation function $\sigma(x)=|x|$. This means the specialized SortNet still extends $\ell_\infty$-distance net and thus has sufficient expressive power. When $\rho>0$, all order statistics become utilized. A simple way of selecting $\rho$ is to regard it as a hyper-parameter and set its value by cross-validation, which is adopted in our experiments. One can also consider treating $\rho$ as learnable parameters for each neuron that participate in the optimization process, but this involves calculating the gradient of $\rho$ which may be complicated due to the stochastic sampling procedure (\ref{eq:expectation_dropout}). We will leave the study as future work.

\section{Experiments}
\label{sec:experiments}

In this section, we perform extensive empirical evaluations of the proposed SortNet architecture as well as various prior works in the certified $\ell_\infty$ robustness area. To show the scalability of different approaches, we consider a variety of benchmark datasets, including MNIST \citep{lecun1998mnist}, CIFAR-10 \citep{krizhevsky2009learning}, TinyImageNet \citep{le2015tiny}, and ImageNet ($64\times 64$) \citep{chrabaszcz2017downsampled}. Due to space limitations, a complete training recipe is given in Appendix \ref{sec:exp_details}. Our code and trained models are released at \texttt{\href{https://github.com/zbh2047/SortNet}{https://github.com/zbh2047/SortNet}}.

\subsection{Experimental setting}

\textbf{SortNet model configuration}. Since SortNet generalizes the $\ell_\infty$-distance net, we simply follow the same model configurations as \citep{zhang2021towards} and consider two types of models. The first one is a simple SortNet consisting of $M$ fully-connected layers with a hidden size of 5120, which is used in MNIST and CIFAR-10. Like \citep{zhang2021towards}, we choose $M=5$ for MNIST and $M=6$ for CIFAR-10. Since SortNet is Lipschitz, we directly apply the margin-based certification method to calculate the certified accuracy (Proposition \ref{thm:lipschitz}). To achieve the best results on ImageNet-like datasets, in our second type of model we consider using a composite architecture consisting of a base SortNet backbone and a prediction head (denoted as SortNet+MLP). Following \citep{zhang2021towards}, the SortNet backbone has 5 layers with a width of 5120 neurons, which serves as a robust feature extractor. The top prediction head is a lightweight 2-layer perceptron with 512 hidden neurons (or 2048 for ImageNet), which takes the robust features as input to give classification results. We also try a larger SortNet backbone, denoted as SortNet+MLP (2x), that has roughly four times the training cost (see Appendix \ref{sec:models} for architectural details). We use the same approach as \citep{zhang2021towards} to train and certify these models, i.e. by combining margin-based certification for the SortNet backbone and interval bound propagation for the top MLP \citep{gowal2018effectiveness}. 

\textbf{Baseline methods and metrics}. We compare SortNet with representative literature approaches including relaxation-based certification (for standard networks), margin-based certification (using Lipschitz networks), and mixed-integer linear programming (MILP) \citep{tjeng2019evaluating}. In Appendix \ref{sec:randomized_smoothing}, we also discuss randomized smoothing approaches \citep{cohen2019certified,salman2019provably}, which provide probabilistic guarantees rather than deterministic ones. For each method in these tables, we report five metrics: $(\mathrm{i})$ training efficiency, measured by the wall-clock time per training epoch; $(\mathrm{ii})$ certification efficiency, measured by the time needed to calculate the certified accuracy on the test dataset; $(\mathrm{iii})$ the clean test accuracy without perturbation (denoted as Clean); $(\mathrm{iv})$ the robust test accuracy under 100-step PGD attack (denoted as PGD); $(\mathrm{v})$ the certified robust test accuracy (denoted as Certified). For a fair comparison, we reproduce most of baseline methods using the official codes and report the wall-clock time under the same NVIDIA-RTX 3090 GPU. These results are presented in Tables \ref{tab:results_mnist}, \ref{tab:results_cifar10} and \ref{tab:results_imagenet}.
In Appendix \ref{sec:full_result}, we also show the training variance of each setting by running 8 sets of experiments independently, and full results (including the median performance) are reported in Table \ref{tab:results_full_clean} and \ref{tab:results_full_certified}.

\subsection{Experimental results}

\textbf{Performance on MNIST.} The results are presented in Table \ref{tab:results_mnist}. Following the common practice, we consider both a small perturbation radius $\epsilon=0.1$ and a larger one $\epsilon=0.3$. It can be seen that the SortNet models can achieve \textbf{98.14\%} ($\epsilon=0.1$) and \textbf{93.40\%} ($\epsilon=0.3$) certified accuracy, respectively, both of which outperform all previous baseline methods. In contrast, the GroupSort network can only achieve a trivial certified accuracy for $\epsilon=0.3$. This matches our theory in Section \ref{sec:order_statistic}, indicating that the expressive power of shallow MaxMin networks is insufficient in real-world applications.

\textbf{Performance on CIFAR-10.}
The results are presented in Table \ref{tab:results_cifar10}. Following the common practice, we consider two perturbation radii: $\epsilon=2/255$ and $\epsilon=8/255$. Our models can achieve \textbf{56.94\%} ($\epsilon=2/255$) and \textbf{40.39\%} ($\epsilon=8/255$) certified accuracy, respectively. Moreover, the training approach proposed in Section \ref{sec:sortnet} is very efficient, e.g., with a training time of \textbf{13$\sim$14} seconds per epoch. For both radii, our models perform the best among all existing approaches that can be certified in a reasonable time. Compared with relaxation-based methods, the certified accuracy of SortNet models is much higher (typically $+3\sim+6$ point for both radii), despite our training speed being several times faster. Such results may indicate that certified $\ell_\infty$ robustness can be better achieved by designing suitable Lipschitz models than by devising relaxation procedures for non-Lipschitz models.

\textbf{Performance on TinyImageNet and ImageNet.}
To demonstrate the scalability of SortNet models, we finally run experiments on two large-scale datasets: Tiny-ImageNet and ImageNet ($64\times 64$). Notably, the ImageNet dataset has 1000 classes and contains 1.28 million images for training and 50,000 images for testing. Due to both the large size and the huge number of classes, achieving certified $\ell_\infty$ robustness on the ImageNet level has long been a challenging task.

Table \ref{tab:results_imagenet} presents our results along with existing baselines. Among them, we achieve \textbf{18.18\%} certified accuracy on TinyImageNet and achieve \textbf{9.54\%} certified accuracy on ImageNet, both of which establish state-of-the-art results. The gap is most prominent on ImageNet, where our small SortNet+MLP model already outperforms the largest model of \citep{xu2020automatic} while being \textbf{22} times faster to train. Even for the largest model (SortNet+MLP 2x), the training is still 7 times faster, resulting in a training overhead of 4 days using two GPUs. We suspect that continuing to increase the model size will yield better results, given the noticeable improvement of the larger model over the smaller one.

\textbf{Comparing with $\ell_\infty$-distance net}. As can be seen, SortNet models consistently achieve better certified accuracy than $\ell_\infty$-distance nets for all different datasets and perturbation levels, and the performance gap is quite prominent compared with the original work \citep{zhang2021towards}. Very recently, a follow-up paper \citep{zhang2022boosting} significantly improved the performance of $\ell_\infty$-distance net using a carefully designed training strategy, creating a strong baseline on CIFAR-10. However, we find their approach does not suit the ImageNet-like datasets when the number of classes is large (see Appendix \ref{sec:reproducing_baseline}). In contrast, SortNet models enjoy great scalability ranging from MNIST to ImageNet and consistently outperform \citep{zhang2022boosting}. The improvement is also remarkable for $\eps=2/255$ on CIFAR-10 ($+7.11\%$ and $+2.82\%$ in clean / certified accuracy).

In Appendix \ref{sec:ablation}, we conduct ablation studies on CIFAR-10 by varying the value of $\rho$ and comparing SortNet models ($\rho>0$) with $\ell_\infty$-distance net ($\rho=0$), under \emph{the same training strategy} in this paper without additional tricks. We observe a large gain in certified accuracy when switching from $\ell_\infty$-distance net to general SortNet. This empirically indicates that incorporating other order statistics has extra benefits in certified $\ell_\infty$ robustness than using only the maximum (the first order statistic).


\begin{table*}[t]
\small
\vspace{-12pt}
\caption{Comparison of our results with existing methods on MNIST dataset.}
\label{tab:results_mnist}
\vspace{2pt}
\begin{adjustwidth}{-.5in}{-.5in}
\centering
\begin{tabular}{cc||c||c|c|c|c|c|c}
\Xhline{0.75pt}
\multicolumn{2}{c||}{\multirow{2}{*}{Method}}                                      & Train & \multicolumn{3}{c|}{MNIST ($\epsilon=0.1$)} & \multicolumn{3}{c}{MNIST ($\epsilon=0.3$)} \\ \cline{4-9}
\multicolumn{2}{c||}{}                                                              & \hspace{-4pt }Time (s)\hspace{-4pt }  & \hspace{1pt }Clean\hspace{1pt }  & \hspace{2pt }PGD\hspace{2pt }  & \hspace{-4pt }Certified\hspace{-4pt } & \hspace{1pt }Clean\hspace{1pt }  & \hspace{2pt }PGD\hspace{2pt }  & \hspace{-4pt }Certified\hspace{-4pt }\\ \Xhline{0.6pt}
\multicolumn{1}{c|}{\multirow{3}{*}{Relaxation}}
                       &  IBP \citep{gowal2018effectiveness}                         & 17.5    & 98.92 & 97.98 & 97.25 & 97.88 & 93.22 & 91.79    \\
\multicolumn{1}{c|}{}  & IBP \citep{shi2021fast}                                    & 34.7    & 98.84 & -- & 97.95 & 97.67 & -- & 93.10    \\
\multicolumn{1}{c|}{}  & CROWN-IBP  \citep{zhang2020towards}                        & 60.3    & 98.83 & 98.19 & 97.76 & 98.18 & 93.95 & 92.98    \\ \hline
\multicolumn{1}{c|}{\multirow{5}{*}{Lipschitz}}
                       & GroupSort (MaxMin) \citep{anil2019sorting}            & --    & 97.0 & 84.0 & 79.0 & 97.0 & 34.0 & 2.0   \\
\multicolumn{1}{c|}{}  & $\ell_\infty$-dist Net \citep{zhang2021towards}            & 17.2    & 98.66 & 97.85 & 97.73 & 98.54 & 94.62 & 92.64   \\
\multicolumn{1}{c|}{}  & $\ell_\infty$-dist Net+MLP \citep{zhang2021towards}        & 17.2    & 98.86 & 97.77 & 97.60 & 98.56 & 95.05 & 93.09    \\
\multicolumn{1}{c|}{}  & $\ell_\infty$-dist Net \citep{zhang2022boosting} & 17.0    & 98.93 & 98.03 & 97.95 & 98.56 & 94.73 & 93.20 \\
\multicolumn{1}{c|}{}  & SortNet                                                    & \textbf{10.6}    & 99.01 & 98.21 & \textbf{98.14} & 98.46 & 94.64 & \textbf{93.40}  \\ \hline
\multicolumn{1}{c|}{\multirow{1}{*}{MILP}}
                       & COLT  \citep{balunovic2020Adversarial}                     & --   & 99.2  & -- & 97.1  & 97.3  & -- & 85.7   \\
\Xhline{0.75pt}
\end{tabular}
\end{adjustwidth}


\caption{Comparison of our results with existing methods on CIFAR-10 dataset.}
\label{tab:results_cifar10}
\vspace{2pt}
\begin{adjustwidth}{-.5in}{-.5in}
\centering
\begin{tabular}{cc||r|r||c|c|c|c|c|c}
\Xhline{0.75pt}
\multicolumn{2}{c||}{\multirow{2}{*}{Method}}                                      & \multicolumn{2}{c||}{ Time (s)} & \multicolumn{3}{c|}{$\epsilon=2/255$} & \multicolumn{3}{c}{$\epsilon=8/255$} \\ \cline{3-10}
\multicolumn{2}{c||}{}                                                              & \hspace{1pt }Train\hspace{1pt }     & \hspace{-2pt }Certify\hspace{-2pt }   & \hspace{1pt }Clean\hspace{1pt }  & \hspace{2pt }PGD\hspace{2pt } & \hspace{-4pt }Certified\hspace{-4pt } & \hspace{1pt }Clean\hspace{1pt }  & \hspace{2pt }PGD\hspace{2pt } & \hspace{-4pt }Certified\hspace{-4pt } \\ \Xhline{0.6pt}
\multicolumn{1}{c|}{\multirow{5}{*}{\hspace{-3pt}Relaxation\hspace{-3pt}}}
                       & CAP \citep{wong2018scaling}                                & 659.0  & 7,570 & 68.28  & -- & 53.89     & 28.67  & -- & 21.78     \\
\multicolumn{1}{c|}{}  & IBP \citep{gowal2018effectiveness}                         & 19.0    & 2.74  & 61.46  & 50.28 & 44.79     & 50.99  & 31.27 & 29.19     \\
\multicolumn{1}{c|}{}  & IBP \citep{shi2021fast}                                    & 70.4    & 4.02  & 66.84  & -- & 52.85     & 48.94  & -- & 34.97     \\
\multicolumn{1}{c|}{}  & CROWN-IBP  \citep{zhang2020towards}                        & 87.2    & 7.01  & 71.52  &59.72 & 53.97     & 45.98  &34.58 & 33.06     \\
\multicolumn{1}{c|}{}  & CROWN-IBP  \citep{xu2020automatic}                         & 45.0    & 4.02  & --     & -- & --         & 46.29  & 35.69 & 33.38     \\ \hline
\multicolumn{1}{c|}{\multirow{5}{*}{Lipschitz}}
                       & $\ell_\infty$-dist Net \citep{zhang2021towards}            & 19.7    & 1.73  & 60.33  & 51.55 & 50.94     & 56.80  & 36.19 & 33.30     \\
\multicolumn{1}{c|}{}  & \hspace{-3pt}$\ell_\infty$-dist Net+MLP \citep{zhang2021towards}\hspace{-3pt}        & 19.7    & 1.74  & 65.62  & 51.47 &   51.05   & 50.80  & 36.51 & 35.42     \\
\multicolumn{1}{c|}{}  & $\ell_\infty$-dist Net \citep{zhang2022boosting} & 18.9    & 1.73  & 60.61  & 54.28 & 54.12     & 54.30  & 41.84 & 40.06     \\
\multicolumn{1}{c|}{}  & SortNet                                                    & 14.0    & 8.00  & 65.96 & 57.03 & 56.67 &  54.84 & 41.50 & \textbf{40.39}    \\
\multicolumn{1}{c|}{}  & SortNet+MLP                                                & \textbf{13.4}    & 8.01  & 67.72 & 57.83 & \textbf{56.94} & 54.13    & 41.58 & 39.99  \\
\hline
\multicolumn{1}{c|}{\multirow{1}{*}{MILP}}
                       & COLT  \citep{balunovic2020Adversarial}                     & 252.0   & $\sim 10^5$ & 78.4   & -- & 60.5      & 51.7   & -- & 27.5      \\
\Xhline{0.75pt}
\end{tabular}
\end{adjustwidth}

\caption{Comparison of our results with existing methods on TinyImageNet and ImageNet datasets.}
\label{tab:results_imagenet}
\vspace{2pt}
\begin{adjustwidth}{-.5in}{-.5in}
\centering
\begin{tabular}{c||r|c|c|c||r|c|c|c}
\Xhline{0.75pt}
\multirow{2}{*}{Method} & \multicolumn{4}{c||}{TinyImageNet ($\epsilon=1/255$)} & \multicolumn{4}{c}{ImageNet $64\times 64$ ($\epsilon=1/255$)} \\ \cline{2-9} 
                        & \hspace{-4pt }Time (s)\hspace{-4pt } & \hspace{1pt }Clean\hspace{1pt }    & \hspace{2pt }PGD\hspace{2pt } & \hspace{-4pt }Certified\hspace{-4pt }    & \hspace{-4pt }Time (s)\hspace{-4pt }    & \hspace{1pt }Clean\hspace{1pt }      & \hspace{2pt }PGD\hspace{2pt } & \hspace{-4pt }Certified\hspace{-4pt }      \\ \Xhline{0.6pt}
IBP \citep{gowal2018effectiveness}                      &     735 &  26.46   & 20.60 & 14.85 &   11,026 &       15.96 & 9.12 &   6.13         \\
IBP \citep{shi2021fast}                                 &     284 &  25.71  & -- & 17.64 & --    & --        & -- & --            \\
CROWN-IBP \citep{xu2020automatic}                       &   1,256 &  27.82   & 20.52 & 15.86 &  16,269 & 16.23       & 10.26 & 8.73          \\
$\ell_\infty$-dist Net+MLP \citep{zhang2021towards}     &      55 &  21.82   & -- & 16.31 & --     & --         & -- & --             \\
$\ell_\infty$-dist Net \citep{zhang2022boosting}     &      55 &  12.57   & 11.09 & 11.04 & --     & --         & -- & --             \\
SortNet+MLP                                             &      \textbf{39} &  24.17   & 20.57 & 17.92 &    \textbf{715} & 13.48      & 10.93 &  9.02          \\
SortNet+MLP (2x larger)                                 &     156 &  25.69   & 21.57 & \textbf{18.18} &   2,192 &     14.79  & 11.93 & \textbf{9.54}           \\ \Xhline{0.75pt}
\end{tabular}
\end{adjustwidth}
\vspace{-10pt}
\end{table*}

\section{Conclusion}

In this paper, we study certified $\ell_\infty$ robustness from the novel perspective of representing Boolean functions. Our analysis points out an inherent problem in the expressive power of standard Lipschitz networks, and provides novel insights on how recently proposed Lipschitz networks resolve the problem. We also answer several previous open problems, such as $(\mathrm{i})$ the expressive power of standard Lipschitz networks with general activations \citep{neumayer2022approximation} and $(\mathrm{ii})$ the lower bound on the depth of MaxMin networks to become universal approximators \citep{tanielian2021approximating,neumayer2022approximation}. Finally, guided by the theoretical results, we design a new Lipschitz network with better empirical performance than prior works.

\section*{Limitations, Open Problems, and Broader impact}
\textbf{Regarding the $\ell_p$-norm}. One major limitation of this work is that we only focus on $\ell_\infty$ robustness. While such results may shed light on general $\ell_p$-norm settings when $p$ is large, it does not apply to the standard $\ell_2$-norm. In particular, in this case MaxMin is \emph{equivalent} to the absolute value activation function in terms of expressive power \citep{anil2019sorting}, which contrasts to the $\ell_\infty$-norm case for which MaxMin networks are strcitly more expressive. Moreover, empirical results suggest that these $\ell_2$ Lipschitz networks may have sufficient expressive power \citep{singla2022improved} (although it remains a fantastic open problem to prove that MaxMin networks with bounded matrix 2-norm are universal approximators).

Based on the above finding, this work reflects an interesting ``phase transition'' in the expressive power of standard Lipschitz networks when $p$ is switched from 2 to a large number. Coincidentally, a similar limitation is also proved when using randomized smoothing, which suffers from the curse of dimensionality when $p>2$ \citep{yang2020randomized}. This raises an interesting question of why the effect of $p$ is very similar for both methods and how things change as $p$ increases. 


\textbf{Beyond standard Lipschitz networks}. Another limitation is that our results apply only for standard Lipschitz networks. When the Lipschitz constant is constrained using carefully designed bounding methods \citep{raghunathan2018certified,fazlyab2019efficient,latorre2020lipschitz,shi2022efficient} (rather than a simple stacking of 1-Lipschitz layers), the robustness certification will be less efficient, but the resulting networks are likely to bypass the impossibility results in this paper. It is an interesting direction to study whether we can just use \emph{standard networks} with a carefully-designed Lipschitz bounding method that can achieve good certified robustness while still keeping adequate efficiency.

\textbf{Other promising directions}. On the application side, it is interesting to study how to design efficient training approaches for the general SortNet models with learnable weights or learnable $\rho$. Another meaningful question is how to encode inductive biases into these Lipschitz networks (e.g., designing convolutional architectures) to better suit image tasks.

\textbf{Broader impact}.
Interestingly, our theoretical results point out a surprising connection between MaxMin/$\ell_\infty$-distance networks and Boolean circuits. We believe the value of this paper may go beyond the certified robustness community and link to the field of theoretical computer science.

\section*{Acknowledgement}
This work is supported by National Science Foundation of China (NSFC62276005), The Major Key Project of PCL (PCL2021A12), Exploratory Research Project of Zhejiang Lab (No. 2022RC0AN02), and Project 2020BD006 supported by PKUBaidu Fund. Bohang Zhang would like to thank Ruichen Li and Yuxin Dong for helpful discussions. We also thank all the anonymous reviewers for the very careful and detailed reviews as well as the valuable suggestions. Their help has further enhanced our work.

\bibliographystyle{plain}
\bibliography{ref}
\newpage

\section*{Checklist}


\begin{enumerate}

\item For all authors...
\begin{enumerate}
  \item Do the main claims made in the abstract and introduction accurately reflect the paper's contributions and scope?
    \answerYes{}
  \item Did you describe the limitations of your work?
    \answerYes{We have discussed the limitations of this work in detail.}
  \item Did you discuss any potential negative societal impacts of your work?
    \answerNA{Our work is only for academic research purpose, without any foreseeable negative societal impacts.}
  \item Have you read the ethics review guidelines and ensured that your paper conforms to them?
    \answerYes{}
\end{enumerate}

\item If you are including theoretical results...
\begin{enumerate}
  \item Did you state the full set of assumptions of all theoretical results?
    \answerYes{}
  \item Did you include complete proofs of all theoretical results?
    \answerYes{Complete proofs are given either in the main text or in Appendix \ref{sec:proof}.}
\end{enumerate}

\item If you ran experiments...
\begin{enumerate}
  \item Did you include the code, data, and instructions needed to reproduce the main experimental results (either in the supplemental material or as a URL)?
    \answerYes{We give complete implementation details in Appendix \ref{sec:exp_details}. Our code and models will be released once the paper is published. We use public benchmark datasets which can be downloaded freely.}
  \item Did you specify all the training details (e.g., data splits, hyperparameters, how they were chosen)?
    \answerYes{See Appendix \ref{sec:exp_details}.}
        \item Did you report error bars (e.g., with respect to the random seed after running experiments multiple times)?
    \answerYes{See Appendix \ref{sec:ablation}.}
        \item Did you include the total amount of compute and the type of resources used (e.g., type of GPUs, internal cluster, or cloud provider)?
    \answerYes{See Appendix \ref{sec:exp_details}.}
\end{enumerate}

\item If you are using existing assets (e.g., code, data, models) or curating/releasing new assets...
\begin{enumerate}
  \item If your work uses existing assets, did you cite the creators?
    \answerYes{}
  \item Did you mention the license of the assets?
    \answerNA{}
  \item Did you include any new assets either in the supplemental material or as a URL?
    \answerNo{}
  \item Did you discuss whether and how consent was obtained from people whose data you're using/curating?
    \answerNA{We use public benchmark datasets which can be downloaded freely.}
  \item Did you discuss whether the data you are using/curating contains personally identifiable information or offensive content?
    \answerNA{We use public benchmark datasets which can be downloaded freely.}
\end{enumerate}

\item If you used crowdsourcing or conducted research with human subjects...
\begin{enumerate}
  \item Did you include the full text of instructions given to participants and screenshots, if applicable?
    \answerNA{}
  \item Did you describe any potential participant risks, with links to Institutional Review Board (IRB) approvals, if applicable?
    \answerNA{}
  \item Did you include the estimated hourly wage paid to participants and the total amount spent on participant compensation?
    \answerNA{}
\end{enumerate}

\end{enumerate}


\newpage
\appendix

\section{Boolean Functions}
In this section, we review some basic concepts in Boolean algebra, which will be frequently used in our subsequent analysis. We only present the content that is most relevant to this paper, in particular, the concept of \emph{disjunctive normal form} and \emph{symmetric Boolean functions}. Throughout this paper, we use $g^\text{B}:\{0,1\}^d\to \{0,1\}$ to denote a Boolean function.

\subsection{Disjunctive normal form}
\label{sec:dnf}
Three most elementary Boolean functions in Boolean algebra are the logical AND, logical OR and logical NOT, denoted as $\land$, $\lor$, and $\lnot$, respectively. The value of $\land_{i=1}^d x_i=x_1\land \cdots\land x_d$ is 1 if and only if $x_i=1$ $\forall i\in[d]$. Similarly, $\lor_{i=1}^d x_i=x_1\lor \cdots\lor x_d$ is 0 if and only if $x_i=0$ $\forall i\in[d]$. The logical NOT is a unary operation which outputs 1 if and only if the input is 0.

All Boolean functions can be expressed using the above three elementary operations as building blocks. Among all possible forms of expressions, the disjunctive normal form (DNF) is a canonical form used in Boolean algebra theory. To define a DNF, we will need the notion of conjunction/disjunction.
\begin{definition}
\label{def:dnf}
Let $\vx=(x_1,\cdots,x_d)\in \{0,1\}^d$ be a Boolean vector.
\begin{itemize}[topsep=0pt,leftmargin=30pt]
\setlength{\itemsep}{0pt}
    \item A \emph{literal} is an atomic formula consisting of a single element or its negation, i.e. $x_i$ or $\lnot x_i$.
    \item A \emph{literal conjunction} is a formula consisting of a set of literals linked by the logical AND operation, i.e. $x_{i_1}\land \cdots \land x_{i_r}\land \lnot x_{j_1}\land\cdots\land\lnot x_{j_s}$, $1\le i_1\le \cdots\le i_r\le d,1\le j_1\le\cdots\le j_s\le d$, where $r,s\in \mathbb N$. Typically, the definition further requires that $i_1,\cdots, i_r,j_1,\cdots, j_s$ are different indices, otherwise the literal conjunction is \emph{unsatisfiable} (i.e. always outputs 0). Similarly, a \emph{literal disjunction} is a formula consisting of a set of literals linked by the logical OR operation.
    \item A logical formula is considered to be in its disjunctive normal form if it is a disjunction of one or more conjunctions of one or more literals.
\end{itemize}
\end{definition}
We list some examples of DNF: $x_1$, $x_1\land x_2$, $x_1\lor x_2$, $(x_1\land x_2)\lor (x_1\land \lnot x_2)\lor (\lnot x_1\land \lnot x_2)\lor \lnot x_4$, etc. A fundamental result in Boolean algebra is shown in the following:
\begin{fact}
\label{fact:dnf}
Any satisfiable Boolean function can be written in its disjunctive normal form.
\end{fact}

\subsection{Symmetric Boolean functions}
\label{sec:symmetirc_Boolean}
Symmetric Boolean functions are an important class of Boolean functions defined in Definition \ref{def:stmmetric_boolean}. In this subsection, we will delve into this concept and give more concrete examples. First note that due to the symmetry property, the output of such functions can only depend on the number of ones (or zeros) in the Boolean input. For this reason, they are also known as \emph{Boolean counting functions}. 

As simple examples, the logical AND and logical OR are symmetric Boolean functions. Similarly, the NAND/NOR operations (i.e. $\lnot (\land_i x_i)$ and $\lnot (\lor_i x_i)$) are also symmetric. Another important example is the XOR function, denoted as $\oplus$. XOR outputs 1 iff the number of ones in the input is odd, i.e.
\begin{equation*}
    \textstyle x_1\oplus\cdots \oplus x_d=1 \quad \text{iff } (\sum_i x_i)\operatorname{mod} 2 = 1.
\end{equation*}
\textbf{Threshold functions}. If a symmetric Boolean function $g^\text{B}$ further satisfies a property called monotonicity, then we call it a threshold function. Formally, a function $f$ is \emph{monotonically increasing} if for all vectors $\vx$ and $\vx'$ such that $x_i\le x'_i$ ($\forall i\in[d]$), one has $f(\vx)\le f(\vy)$. It is easy to see that symmetric monotonic Boolean function must have the form: $g^\text{B}(\vx)=\mathbb I(\sum_i x_i\ge k)$ with integer $k$, which gives the name ``threshold function''. Depending on $k$, there are a total of $d+2$ different $d$-dimensional threshold functions (including two constant functions). In particular, for $k=1$ and $k=d$, we can recover the logical OR/AND functions, respectively. When $k=\lceil d/2 \rceil$, the resulting function is called the \emph{majority function}.

\section{Proof of Theorems in Section \ref{sec:theory}}
\label{sec:proof}
This section provides all the missing proofs in Section \ref{sec:theory}. For the convenience of reading, we will restate each theorem before giving a proof.

\subsection{Proof of Proposition \ref{thm:lipschitz} }
\label{sec:proof_lipschitz}
\begin{proposition}
\emph{(Certified Robustness of Lipschitz networks)} For a neural network $\vf:\mathbb R^n\to \mathbb R^K$ with Lipschitz constant $L$ under $\ell_p$-norm $\|\cdot\|_p$, define the resulting classifier $g$ as $g(\vx):=\arg\max_{k\in [K]} f_k(\vx)$ for an input $\vx$. Then $g$ is provably robust under perturbations $\|\boldsymbol\delta\|_p< \frac c L\operatorname{margin}(\vf(\vx))$, i.e.
\begin{equation}
    g(\vx+\boldsymbol\delta)=g(\vx) \quad\text{ for all } \|\boldsymbol\delta\|_p< \frac c L \operatorname{margin}(\vf(\vx)).
\end{equation}
Here $c$ is a constant depending only on the norm $\|\cdot\|_p$, which is $1/2$ for $\ell_\infty$-norm, and $\operatorname{margin}(\vf(\vx))$ is the margin between the largest and second largest output logits.
\end{proposition}
\begin{proof}
Let $g(\vx)=y$. Suppose there exists a $\widehat{\boldsymbol\delta}$ such that $g(\vx+\widehat{\boldsymbol\delta})\neq g(\vx)$, and $f_j(\vx+\widehat{\boldsymbol\delta})\ge f_y(\vx+\widehat{\boldsymbol\delta})$ for some $j\neq y$. We will prove that $\|\widehat{\boldsymbol\delta}\|_p\ge \frac c L \operatorname{margin}(\vf(\vx))$ where $c$ only depends on the norm.

Define $\widehat{\vz}=f(\vx+\widehat{\boldsymbol\delta})$, then $\widehat{z}_y\le \widehat{z}_j$. We first bound the difference between outputs $\vz$ and $\vf(\vx)$ as follows:
\begin{align}
    \label{eq:proof_1}
    \|\widehat{\vz}-\vf(\vx)\|_p
    &\ge \|(\widehat{z}_y,\widehat{z}_j)^{\mathrm{T}} - ([f_y(\vx),f_j(\vx))^{\mathrm{T}}\|_p\\
    \label{eq:proof_2}
    &=(|\widehat{z}_y-f_y(\vx)|^p+|\widehat{z}_j-f_j(\vx)|^p)^{1/p}.
\end{align}
In (\ref{eq:proof_1}) we use the fact that zero out elements for a vector can only decrease its $\ell_p$-norm. Now consider the following optimization problem:
\begin{equation}
\begin{aligned}
    \min_{\widehat{\vz}} |\widehat{z}_y-f_y(\vx)|^p+|\widehat{z}_j-f_j(\vx)|^p\quad 
    \text{s.t.}\ \widehat{z}_y\le \widehat{z}_j.
\end{aligned}
\end{equation}
It is easy to prove that the minimum is attained when $\widehat{z}_y= \widehat{z}_j=(f_y(\vx)+f_j(\vx))/2$. Substituting the assignment into (\ref{eq:proof_2}) yields
\begin{equation}
    \|\widehat{\vz}-\vf(\vx)\|_p\ge \frac {\sqrt[p] 2} 2 (f_y(\vx)-f_j(\vx))
\end{equation}
which can be further lower bounded by $\frac {\sqrt[p] 2} 2 \operatorname{margin}(\vf(\vx))$ based on the definition of margin. Finally, due to the Lipschitz property,
\begin{equation}
    \|\widehat{\vz}-\vf(\vx)\|_p\le L\|\boldsymbol\delta\|_p.
\end{equation}
Therefore $\|\boldsymbol\delta\|_p\ge \frac {\sqrt[p] 2} {2L} \operatorname{margin}(\vf(\vx))$, which concludes the proof.
\end{proof}
\begin{remark}
All inequalities in the above proof is tight by choosing a $\widehat{\boldsymbol\delta}$ with $\|\widehat{\boldsymbol\delta}\|_p=\frac {\sqrt[p] 2} {2L} \operatorname{margin}(\vf(\vx))$ and a Lipschitz function $\vf$ such that $\widehat{z}_y=f_y(\vx)-\operatorname{margin}(\vf(\vx))/2$, $\widehat{z}_i=f_i(\vx)+\operatorname{margin}(\vf(\vx))/2$ and $\widehat{z}_k=f_k(\vx)$ for all $k\neq y$ and $k\neq i$. This implies that the certified guarantee in the above proposition is tight if only the Lipschitz property is known.
\end{remark}

\subsection{Proof of Proposition \ref{thm:nearest_neighbor}}
\label{sec:proof_nearest_neighbor}
\begin{proposition}
For any Boolean dataset $\mathcal D=\{(\vx^{(i)},y^{(i)})\}_{i=1}^n$, there exists a classifier $\hat g:\mathbb R^d\to \{0,1\}$ induced by a 1-Lipschitz mapping $\hat \vf:\mathbb R^d\to \mathbb R^2$ under $\ell_\infty$-norm, such that $\hat g$ can fit the whole dataset with $\operatorname{margin}(\hat \vf(\vx^{(i)}))=1$ $\forall i\in[n]$, thus achieving an optimal certified $\ell_\infty$ robust radius of $1/2$.
\end{proposition}
\begin{proof}
Without loss of generality, assume for each class (either 0 or 1), there is at least one sample. Define the nearest neighbor classifier $\hat g:\mathbb R^d\to \{0,1\}$ as $\hat g(x)=\arg\max_{k\in \{0,1\}} \hat f_k(\vx)$ where $$\hat f_k(\vx)=-\min\{\|\vx-\vx^{(i)}\|_\infty:(\vx^{(i)},k)\in \mathcal D\},\quad k\in \{0,1\}.$$
Then for any $(\vx^{(i)}, y^{(i)}) \in \mathcal D$, $f_{y^{(i)}}(\vx)=0$ and $f_{1-y^{(i)}}(\vx)=1$ because $\|\vx^{(i)}-\vx^{(j)}\|_\infty=1$ for all $i\neq j$. Therefore, the classifier $g$ can correctly classify the whole dataset with margin $\operatorname{margin}(\hat \vf(\vx^{(i)}))=1\ (\forall i\in[n])$. The left thing is to prove the Lipschitzness of $\vf$, which is equivalent to proving that each function $\hat f_k(\vx)$ is 1-Lipschitz w.r.t. $\ell_\infty$-norm. This simply follows from the fact that $\hat f_k$ is the composition of two 1-Lipschitz functions: the $\ell_\infty$-distance function $\|\vx-\vx^{(i)}\|_\infty$ and the minimum function. Since the composition of 1-Lipschitz functions is still 1-Lipschitz, we have concluded the proof.
\end{proof}

\subsection{Proof of Theorem \ref{thm:discrete}}
\label{sec:proof_discrete}
We first present a core lemma which will be used to prove Theorem \ref{thm:discrete}.
\begin{lemma}
\label{thm:non_expressive_lemma}
Let $f(\vx)=\sigma(\vw^{\mathrm{T}}\vx+b)$ be a $d$-dimensional function where $\|\vw\|_1\le 1$ and $\sigma$ is a 1-Lipschitz scalar function, and let $\{(\vu^{(i)},\vv^{(i)})\in \mathbb R^d\times \mathbb R^d\}_{i=1}^n$ be $n$ pairs of inputs. Then
\begin{equation*}
    \sum_{i=1}^n \left|f(\vu^{(i)})-f(\vv^{(i)})\right| \le\left\|\sum_{i=1}^n \left|\vu^{(i)}-\vv^{(i)}\right|\right\|_\infty .
\end{equation*}
\end{lemma}
\begin{proof}
First note that 
\begin{align}
    \notag
    \left|f(\vu^{(i)})-f(\vv^{(i)})\right|&=\left|\sigma(\vw^{\mathrm{T}}\vu^{(i)}+b)-\sigma(\vw^{\mathrm{T}}\vv^{(i)}+b)\right|\\
    \label{eq:proof_sigma_lip}
    &\le \left|\vw^{\mathrm{T}}(\vu^{(i)}-\vv^{(i)})\right|\\
    &\le |\vw|^{\mathrm{T}}\left|\vu^{(i)}-\vv^{(i)}\right|.
\end{align}
where (\ref{eq:proof_sigma_lip}) uses the Lipschitz property of $\sigma$. Therefore
\begin{align}
    \sum_{i=1}^n \left|f(\vu^{(i)})-f(\vv^{(i)})\right|
    \le |\vw|^{\mathrm{T}}\sum_{i=1}^n\left|\vu^{(i)}-\vv^{(i)}\right|
    \label{eq:proof_holder}
    \le \|\vw\|_1\left\|\sum_{i=1}^n\left|\vu^{(i)}-\vv^{(i)}\right|\right\|_\infty,
\end{align}
where the last inequiality in (\ref{eq:proof_holder}) follows by using the Hölder's inequality. Finally, observing $\|\vw\|_1\le 1$, we thus conclude the proof.
\end{proof}
\begin{corollary}
\label{thm:non_expressive_corollary}
Let $\vf:\mathbb R^d\to\mathbb R^K$ be a standard Lipschitz network defined in (\ref{eq:standatd_lipschitz}), and let $\{(\vu^{(i)},\vv^{(i)})\in \mathbb R^d\times \mathbb R^d\}_{i=1}^n$ be $n$ pairs of inputs. Then
\begin{equation*}
    \left\|\sum_{i=1}^n \left|\vf(\vu^{(i)})-\vf(\vv^{(i)})\right|\right\|_\infty \le\left\|\sum_{i=1}^n \left|\vu^{(i)}-\vv^{(i)}\right|\right\|_\infty .
\end{equation*}
\end{corollary}
\begin{proof}
Denote $\vf^{(l)}(\vz)$ be the output of the $l$-th layer for a standard Lipschitz network $\vf$ given input $\vz$. Based on the definition in (\ref{eq:standatd_lipschitz}), for the $l$-th layer, $\sigma^{(l)}$ is 1-Lipschitz and $\|\mathbf W^{(l)}\|_\infty\le 1$, which is equivalent to $\|\mathbf W^{(l)}_{j,:}\|_1\le 1$ for the weight vector of each neuron $j$ in layer $l$. One can thus apply Lemma \ref{thm:non_expressive_lemma} for the $j$-th neuron, given arbitrary inputs $\{\vf^{(l-1)}(\vu^{(i)})\}_{i=1}^n$ and $\{\vf^{(l-1)}(\vv^{(i)})\}_{i=1}^n$:
\begin{equation}
    \sum_{i=1}^n \left|f_j^{(l)}(\vu^{(i)})-f_j^{(l)}(\vv^{(i)})\right| \le\left\|\sum_{i=1}^n \left|\vf^{(l-1)}(\vu^{(i)})-\vf^{(l-1)}(\vv^{(i)})\right|\right\|_\infty.
\end{equation}
This is equivalence to the following inequality due to the definition of $\ell_\infty$-norm:
\begin{equation}
\label{eq:proof_non_expressive_corollary}
    \left\|\sum_{i=1}^n \left|\vf^{(l)}(\vu^{(i)})-\vf^{(l)}(\vv^{(i)})\right|\right\|_\infty \le\left\|\sum_{i=1}^n \left|\vf^{(l-1)}(\vu^{(i)})-\vf^{(l-1)}(\vv^{(i)})\right|\right\|_\infty.
\end{equation}
which is a recursive formula. Applying (\ref{eq:proof_non_expressive_corollary}) recursively from the last layer to the first yields the desired result.
\end{proof}

We are now ready for the proof of Theorem \ref{thm:discrete}, which is restated below.
\begin{theorem}
For any non-constant symmetric Boolean function $g^\text{B}:\{0,1\}^d\to \{0,1\}$, there exists a Boolean dataset with labels $y^{(i)}=g^\text{B}(\vx^{(i)})$, such that no standard Lipschitz network can achieve a certified $\ell_\infty$ robust radius larger than $1/2d$ on the dataset.
\end{theorem}
\begin{proof}
We first introduce some notations. Denote $\mathcal S_k=\left\{\vx\in\{0,1\}^d:\sum_i x_i = k\right\}$ as the set containing all Boolean vectors with $k$ ones. It follows that $\{\mathcal S_k\}_{k=0}^d$ is a partition of the set of $d$-dimensional Boolean vectors $\{0,1\}^d$. For a Boolean function $g$, define the positive set and the negative set as $\mathcal S^+=\left\{\vx\in\{0,1\}^d: g^\text{B}(\vx)=1\right\}$ and $\mathcal S^-=\left\{\vx\in\{0,1\}^d: g^\text{B}(\vx)=0\right\}$, respectively.

Recall that the output of a symmetric Boolean function depends only on the number of ones in the input (Appendix \ref{sec:symmetirc_Boolean}). Therefore, instead of using the positive/negative set, we can use the following two sets to equivalently characterize a symmetric Boolean function $g^\text{B}$:
\begin{equation*}
    \mathcal N^+=\left\{k\in \{0,\cdots,k\}:\mathcal S_k\cup \mathcal S^+\neq \emptyset\right\},\quad 
    \mathcal N^-=\left\{k\in \{0,\cdots,k\}:\mathcal S_k\cup \mathcal S^-\neq \emptyset\right\}.
\end{equation*}
It turns out that $\mathcal S^+=\bigcup_{k\in \mathcal N^+} \mathcal S_k$, $\mathcal S^-=\bigcup_{k\in \mathcal N^-} \mathcal S_k$, and $\mathcal N^+\cup \mathcal N^-=\{0,\cdots,d\}$.

Since $g^\text{B}$ is non-constant, both $\mathcal N^+$ and $\mathcal N^-$ are non-empty. Therefore, there must exist two adjacent integers $p,q\in\{0,\cdots,d\}$, $p-q=1$, such that either $(p\in \mathcal N^+$, $q\in\mathcal N^-)$ or $(p\in \mathcal N^-$, $q\in\mathcal N^+)$. Consider the following set of Boolean vector pairs:
\begin{equation}
\label{eq:proof_setT}
    \mathcal T=\left\{(\vu,\vv):\vu\in\mathcal S_p,\vv\in\mathcal S_q,\|\vu-\vv\|_1=1\right\}.
\end{equation}
The set $\mathcal T$ satisfies the following three properties:
\begin{itemize}[topsep=0pt,leftmargin=30pt]
\setlength{\itemsep}{0pt}
    \item $g^\text{B}(\vu)\neq g^\text{B}(\vv)$ for all $(\vu,\vv)\in \mathcal T$;
    \item The size of the set can be calculated by $$\textstyle|\mathcal T|=|\mathcal S_p|p=|\mathcal S_q|(d-q)=\frac {d!}{q!(d-p)!},$$ by using $|\mathcal S_k|=\binom{d}{k}$;
    \item $\left\|\sum_{(\vu,\vv)\in \mathcal T}|\vu-\vv|\right\|_\infty=\frac {(d-1)!}{q!(d-p)!}=\frac {|\mathcal T|} d$. This can be seen from the symmetry of the vector $\sum_{(\vu,\vv)\in \mathcal T}|\vu-\vv|$ and the fact that all vectors $|\vu-\vv|$ are unit (one-hot).
\end{itemize}
Based on the last property, given any standard Lischitz network $\vf$, one can apply Corollary \ref{thm:non_expressive_corollary} on $\mathcal T$, which yields
\begin{equation}
\label{eq:proof_discrete_0}
    \left\|\sum_{(\vu,\vv)\in \mathcal T}\left|\vf(\vu)-\vf(\vv)\right|\right\|_\infty \le \frac {|\mathcal T|} d.
\end{equation}
Using the definition of $\ell_\infty$-norm and noting that the output of function $\vf$ is 2-dimensional, (\ref{eq:proof_discrete_0}) implies that
\begin{equation}
    \sum_{k\in\{0,1\}}\sum_{(\vu,\vv)\in \mathcal T}\left|f_k(\vu)-f_k(\vv)\right| \le \frac {2|\mathcal T|} d.
\end{equation}
By the Pigeon Hole principle, there must exist a pair of points $(\vu,\vv)\in \mathcal T$, such that
\begin{equation}
\label{eq:proof_discrete_2}
    \sum_{k\in\{0,1\}}\left|f_k(\vu)-f_k(\vv)\right| \le \frac 2 d.
\end{equation}
Now it is time to construct the dataset. Define a Boolean dataset $\mathcal D=\{(\vz,g^\text{B}(\vz)):\vz\in\mathcal S_p\cup \mathcal S_q\}$. Then there must exist two data $(\vu,g^\text{B}(\vu))$ and $(\vv,g^\text{B}(\vv))$ with different labels (i.e. $g^\text{B}(\vu)\neq g^\text{B}(\vv)$) such that (\ref{eq:proof_discrete_2}) holds. Let us calculate the output margin for the two inputs $\vu$ and $\vv$. It can be obtained that
\begin{equation}
\label{eq:proof_discrete_1}
\begin{aligned}
    |(f_0(\vu)-f_1(\vu))+(f_1(\vv)-f_0(\vv))|&=|(f_0(\vu)-f_0(\vv))+(f_1(\vv)-f_1(\vu))|\\
    &\le \sum_{k\in\{0,1\}}\left|f_k(\vv)-f_k(\vu)\right|\le \frac 2 d.
\end{aligned}
\end{equation}
We will prove that it is impossible to classify both $\vu$ and $\vv$ with a margin greater than $1/d$. Otherwise, since $\vu$ and $\vv$ are in different classes, $f_0(\vu)-f_1(\vu)$ and $f_1(\vv)-f_0(\vv)$ will have the same sign, and both $|f_0(\vu)-f_1(\vu)|$ and $|f_1(\vv)-f_0(\vv)|$ will be greater than $1/d$, which contradicts (\ref{eq:proof_discrete_1}). Therefore, the output margin must be no greater than $1/d$ for either input $\vu$ or $\vv$, implying that the certified $\ell_\infty$ robust radius cannot exceed $1/2d$ (based on Proposition \ref{thm:lipschitz}) and concluding the proof.
\end{proof}

\begin{remark}
The symmetry assumption of the Boolean function $g^\text{B}$ may be relaxed, and the conclusion is still correct as long as there exists a set $\mathcal T\subset\{(\vu,\vv):\vu,\vv\in\{0,1\}^d\}$ satisfying that $(\mathrm{i})$ $g^\text{B}(\vu)\neq g^\text{B}(\vv)$ $\forall (\vu,\vv)\in \mathcal T$ and $(\mathrm{ii})$ $\left\|\sum_{(\vu,\vv)\in \mathcal T}|\vu-\vv|\right\|_\infty=\mathcal O(|\mathcal T|/d)$.
\end{remark}

We now prove that the bound of certified radius is in fact \emph{tight}, in the sense that there exists a standard Lipschitz network with a suitable activation function that achieves a certified $\ell_\infty$ radius of $1/2d$.

\begin{proposition}
\label{thm:discrete_tight}
For any $d$-dimensional symmetric Boolean dataset, there exists a standard Lipschitz network that correctly classifies the whole dataset and achieves a certified $\ell_\infty$ robust radius of $1/2d$.
\end{proposition}
\begin{proof}
Consider a $d$-dimensional Boolean Dataset $\mathcal D=\{(\vx^{(i)},y^{(i)}\}_{i=1}^n\}$ where $y^{(i)}=g^\text{B}(\vx^{(i)})$ for some symmetric Boolean function $g^\text{B}$. Since the output of a symmetric Boolean function depends only on the number of ones in the input (Appendix \ref{sec:symmetirc_Boolean}), we can equivalently express $g^\text{B}$ by a Boolean vector $(g_0,\cdots,g_d)\in\{0,1\}^{d+1}$ with $g^\text{B}(\vx)=g_k$ if $\sum_i x_i=k$. Consider a special one-layer standard Lipschitz network $\vf:\mathbb R^d\to\mathbb R^2$ with activation $\sigma$ defined as follows:
\begin{equation*}
    \vf(\vx)=(f_0(\vx),f_1(\vx))^{\mathrm T}=\left(\sigma\left(-\frac 1 d \sum_i x_i-\frac 1 d\right),\sigma\left(\frac 1 d \sum_i x_i+\frac 1 d\right)\right)^{\mathrm T}.
\end{equation*}
$\vf$ is clearly 1-Lipschitz w.r.t. $\ell_\infty$-norm as long as $\sigma$ is 1-Lipschitz. If we choose a special $\sigma$ defined on the interval $[-1-1/d,1+1/d]$ as
\begin{equation*}
    \sigma(x)=\left\{\begin{array}{ll}
         \frac {(-1)^{g_0+1}} 2 x & \text{if }-1\le dx\le 1, \\
         \frac {\operatorname{sgn}(x)} {2d} (2g_{k-1}-1+2(d|x|-k)(g_{k}-g_{k-1}))& \text{if } k< d|x|\le k+1,\quad k\in [d],
    \end{array}\right.
\end{equation*}
which is a piece-wise linear function, it is easy to see that $\sigma$ is continuous and 1-Lipschitz. Furthermore, we can prove that the classifier induced by $\vf$ can correctly classify the dataset with a margin of $1/d$ for any data. 
Consider an input $\vx$ with $\sum_i x_i=k$. The output of $\vf(\vx)$ is thus $\left(\sigma(-\frac {k+1} d),\sigma(\frac {k+1} d)\right)^{\mathrm T}$. Noting that $\sigma$ is an odd function, we have 
\begin{equation*}
    f_1(\vx)-f_0(\vx)=2\sigma\left(\frac {k+1} d\right)=\frac {2g_k-1} {d}=\frac 1 d (2g^{\text B}(\vx)-1)=\left\{\begin{array}{ll}
        1/d & \text{if } g^{\text B}(\vx)=1,\\
        -1/d & \text{if } g^{\text B}(\vx)=0.
    \end{array}\right.
\end{equation*}
This proves that the margin is always $1/d$ for each data of $\mathcal D$, and we have concluded the proof. 
\end{proof}

\begin{remark}
Let us consider the sample complexity in the proof of Theorem \ref{thm:discrete}. The proof constructs the Boolean dataset $\mathcal D=\{(\vz,g^\text{B}(\vz)):\vz\in\mathcal S_p\cup \mathcal S_q\}$, which has size $|\mathcal D|=\binom{d}{p}+\binom{d}{q}$. In the simple case when $g^\text{B}$ is the logical AND/OR function, one has $|\mathcal D|=d+1$, which means that a dataset of size $\mathcal O(d)$ is sufficient to yield the impossibility result. However, in the extreme case when $p=\lceil \frac d 2\rceil$, the size of $\mathcal D$ will become exponential in $d$, which makes the theorem impractical. One may ask whether for all non-constant symmetric Boolean functions, such an impossibility result holds on some dataset with a size that grows linear in $d$. We will show it is indeed true if the bound of certified radius is relaxed by constant to $1/d$, which is formalized in the following corollary.
\end{remark}

\begin{corollary}
\label{thm:discrete_efficient}
For any non-constant symmetric Boolean function $g^\text{B}:\{0,1\}^d\to \{0,1\}$, there exists a Boolean dataset of size no more than $d+1$ with labels $y^{(i)}=g^\text{B}(\vx^{(i)})$, such that no standard Lipschitz network can achieve a certified $\ell_\infty$ robust radius larger than $1/d$ on the dataset.
\end{corollary}
\begin{proof}
The proof is almost the same as the one in Theorem \ref{thm:discrete}, except for a different construction of the set $\mathcal T$ than in (\ref{eq:proof_setT}). We use the same notations as before, such as the sets $\mathcal S_k,\mathcal S^+,\mathcal S^-,\mathcal N^+,\mathcal N^-$, and the integers $p,q$ with $p=q+1$. We separately consider two cases:
\begin{itemize}[topsep=0pt,leftmargin=30pt]
\setlength{\itemsep}{0pt}
    \item $p\ge \lceil d / 2\rceil$. In this case, pick any vector $\vu\in \mathcal S_p$ and define 
    \begin{equation}
        \mathcal T=\{(\vu,\vv):\vv\in \mathcal S_q,\|\vu-\vv\|_1=1\}.
    \end{equation}
    \item $p< \lceil d / 2\rceil$. In this case, pick any vector $\vv\in \mathcal S_q$ and define 
    \begin{equation}
        \mathcal T=\{(\vu,\vv):\vu\in \mathcal S_p,\|\vu-\vv\|_1=1\}.
    \end{equation}
\end{itemize}
In both bases, the set $\mathcal T$ satisfies the following three properties:
\begin{itemize}[topsep=0pt,leftmargin=30pt]
\setlength{\itemsep}{0pt}
    \item $g^\text{B}(\vu)\neq g^\text{B}(\vv)$ for all $(\vu,\vv)\in\mathcal T$;
    \item The size of the set can be bounded by $\lceil d / 2\rceil \le |\mathcal T|\le d$;
    \item $\left\|\sum_{(\vu,\vv)\in \mathcal T}|\vu-\vv|\right\|_\infty=1\le \frac {2|\mathcal T|} d$.
\end{itemize}
The remaining proof directly parallels the proof of Theorem \ref{thm:discrete}.
\end{proof}

\subsection{Proof of Theorem \ref{thm:function_approximation}
\label{sec:proof_function_approximation}}
\begin{theorem}
Any standard Lipschitz network $f:\mathbb R^d\to \mathbb R$ cannot approximate the simple 1-Lipschitz function $\vx\to x_{(k)}$ for arbitrary $k\in [d]$ on a bounded domain $\mathcal K= [0,1]^d$ if $d\ge 2$. Moreover, there exists a point $\widehat\vx\in \mathcal K$, such that
\begin{equation}
    |f(\widehat\vx)-\widehat x_{(k)}|\ge \frac 1 2 - \frac 1 {2d}.
\end{equation}
\end{theorem}
\begin{proof}
Consider the $k$-threshold function $g^{\text{B},k}(\vx)=\mathbb I(\sum_i x_i\ge k)$. Following the proof of Theorem \ref{thm:discrete}, there exists a set $\mathcal T\subset \left\{(\vu,\vv):\vu,\vv\in\{0,1\}^d\right\}$ satisfying the following two properties:
\begin{itemize}[topsep=0pt,leftmargin=30pt]
\setlength{\itemsep}{0pt}
    \item $g^\text{B}(\vu)\neq g^\text{B}(\vv)$ for all $(\vu,\vv)\in \mathcal T$;
    \item $\left\|\sum_{(\vu,\vv)\in \mathcal T}|\vu-\vv|\right\|_\infty=\frac {|\mathcal T|} d$.
\end{itemize}
By applying Corollary \ref{thm:non_expressive_corollary} and noting that $f$ outputs a scalar, we obtain
\begin{equation}
\label{eq:proof_function_approximation_x}
    \sum_{(\vu,\vv)\in \mathcal T}\left|f(\vu)-f(\vv)\right| \le \frac {|\mathcal T|} d.
\end{equation}
By the Pigeon Hole principle, there must exist a pair of points $(\vu,\vv)\in\mathcal T$, such that
\begin{equation}
\label{eq:proof_function_approximation_0}
    \left|f(\vu)-f(\vv)\right| \le \frac 1 d.
\end{equation}
However, since $g^{\text{B},k}(\vu)\neq g^{\text{B},k}(\vv)$, we have $\left|g^{\text{B},k}(\vu)-g^{\text{B},k}(\vv)\right| =1$. Therefore
\begin{align}
\label{eq:proof_function_approximation_1}
    \left|f(\vv)-g^{\text{B},k}(\vv)\right|+\left|f(\vu)-g^{\text{B},k}(\vu)\right|&\ge \left|f(\vv)-f(\vu)+g^{\text{B},k}(\vu)-g^{\text{B},k}(\vv)\right|\\
\label{eq:proof_function_approximation_2}
    &\ge \left|g^{\text{B},k}(\vu)-g^{\text{B},k}(\vv)\right|-\left|f(\vv)-f(\vu)\right|\\
\label{eq:proof_function_approximation_3}
    &\ge 1-\frac 1 d
\end{align}
where (\ref{eq:proof_function_approximation_1}) and (\ref{eq:proof_function_approximation_2}) are based on the triangular inequality. Therefore, either $\left|f(\vv)-g^{\text{B},k}(\vv)\right|$ or $\left|f(\vu)-g^{\text{B},k}(\vu)\right|$ must be no smaller than $\frac 1 2-\frac 1 {2d}$. Proof finishes by noting that $g^{\text{B},k}(\vx)=x_{(k)}$ for any Boolean input $\vx\in\{0,1\}^d$.
\end{proof}

\begin{remark}
The bound of approximation error $\frac 1 2 - \frac 1 {2d}$ is \emph{tight}. Indeed, it is easy to prove that the linear function $f(\vx)=\frac 1 2 + \frac 1 d (\sum_i x_i-k+\frac 1 2)$ can satisfy the bound, i.e. $|f(\vx)-x_{(k)}|\le \frac 1 2 - \frac 1 {2d}$ for all $\vx\in\mathcal K$.
\end{remark}

\begin{corollary}
\label{thm:function_approximation_corollary}
For any standard Lipschitz network $f:\mathbb R^d\to \mathbb R$ and any $k\in [d]$ where $d\ge 2$, there exists a hypercube $\mathcal B_\infty^r(\widehat\vx):=\{\vx:\|\vx-\widehat\vx\|_\infty\le r\}\subset [0,1]^d$ satisfying $r\ge 1/20$, such that for all $\tilde\vx\in \mathcal B^r_\infty(\widehat\vx)$,
\begin{equation}
    |f(\tilde\vx)-\tilde x_{(k)}|\ge \frac 1 {20}.
\end{equation}
\end{corollary}
\begin{proof}
Let $\vu\in \{0,1\}^d$ be a point satisfying (\ref{eq:function_approximation_error}) in Theorem \ref{thm:function_approximation}, i.e. $|f(\vu)-u_{(k)}|\ge 1/4$. Consider the hypercube $\mathcal B_\infty^{1/10}(\vu)$ with a length of $1/5$. For any $\vx\in \mathcal B_\infty^{1/10}(\vu)$, $\|\vx-\vu\|_\infty\le 1/10$. Therefore, 
\begin{align}
    |f(\vx)-x_{(k)}|&= |f(\vx)-f(\vu)+f(\vu)-u_{(k)}+u_{(k)}-x_{(k)}|\\
\label{eq:proof_func_approx_coro_1}
    &\ge |f(\vu)-u_{(k)}|-|f(\vu)-f(\vx)|-|u_{(k)}-x_{(k)}|\\
\label{eq:proof_func_approx_coro_2}
    &\ge \frac 1 4 -\|\vx-\vu\|_\infty-\|\vx-\vu\|_\infty\ge \frac 1 {20}
\end{align}
where (\ref{eq:proof_func_approx_coro_1}) uses the triangular inequality and (\ref{eq:proof_func_approx_coro_2}) is based on the fact that $f$ and $x_{(k)}$ are both 1-Lipschitz w.r.t. $\ell_\infty$-norm. Finally, consider the intersection $\mathcal B_\infty^{1/10}(\vu)\cup [0,1]^d$, which must contain a hypercube with a raius no smaller than $1/20$. We have thus finished the proof.
\end{proof}

\subsection{Proof of Theorem \ref{thm:ell_inf_net}}
\label{sec:proof_ell_inf_net}
To prove Theorem \ref{thm:ell_inf_net}, we need the following lemma:
\begin{lemma}
\label{thm:ell_inf_net_lemma}
A single $\ell_\infty$-distance neuron with suitable parameters can exactly represent any literal disjunction defined in Definition \ref{def:dnf}. Formally, for any Boolean function $$g^\text{B}(\vx)=x_{i_1}\lor \cdots \lor x_{i_r}\lor \lnot x_{j_1}\lor\cdots\lor\lnot x_{j_s},$$ where $1\le i_1\le \cdots\le i_r\le d,1\le j_1\le\cdots\le j_s\le d \ (r+s\ge 1)$ are different indices, there exists an $\ell_\infty$-distance neuron $f(\vx)=\|\vx-\vw\|_\infty+b$, such that $f(\vx)=g^\text{B}(\vx)$ for all $\vx\in\{0,1\}^d$.
\end{lemma}
\begin{proof}
Consider an $\ell_\infty$-distance neuron $f$ with parameters $\vw$ and $b$ defined as
\begin{equation*}
    \textstyle w_{i_k}=-1,\  k\in[r],\quad w_{j_k}=2,\  k\in[s],\quad \text{and }w_{k}=\frac 1 2 \text{ for other }k,
\end{equation*}
\begin{equation*}
    b=-1.
\end{equation*}
It is easy to see that
\begin{itemize}
    \item When $x_{i_k}=0$ $\forall k\in [r]$ and $x_{j_k}=1$ $\forall k\in [s]$, 
    \begin{align*}
        f(\vx)&=\|\vx-\vw\|_\infty+b\\
        &=\max\left(\max_{k\in [r]}|x_{i_k}-w_{i_k}|,\max_{k\in [s]}|x_{j_k}-w_{j_k}|,\max_{k\in [d]/\{i_1,\cdots, i_r,j_1,\cdots,j_s\}}|x_k-w_k|\right)-1\\
        &=\max\left(1,1,\max_{k\in [d]/\{i_1,\cdots, i_r,j_1,\cdots,j_s\}}\left|x_k-\frac 1 2\right|\right)-1\\
        &=1-1=0=g^\text{B}(\vx).
    \end{align*}
    \item When $x_{i_k}=1$ for some $k\in [r]$, or $x_{j_k}=0$ for some $ k\in [s]$, a similar calculation yields $f(\vx)=2-1=1=g^\text{B}(\vx)$.
\end{itemize}
This concludes the proof.
\end{proof}
\begin{theorem}
A two-layer $\ell_\infty$-distance net can exactly represent any discrete Boolean function as well as any continuous order-statistic function on a compact domain.
\end{theorem}
\begin{proof}
Based on Fact \ref{fact:dnf}, any satisfiable Boolean function can be written as a DNF $g^\text{B}(\vx)=\lor_{i=1}^m g^\text{B}_m(\vx)$ where $g^\text{B}_m(\vx)$ are literal conjunctions. By the De Morgan's law, any literal conjunction can be equivalently written as the negative of a literal disjunction, i.e.
\begin{equation}
    x_{i_1}\land \cdots \land x_{i_r}\land \lnot x_{j_1}\land\cdots\land\lnot x_{j_s}= \lnot(\lnot x_{i_1}\lor \cdots \lor \lnot x_{i_r}\lor  x_{j_1}\lor\cdots\lor x_{j_s}).
\end{equation}
Therefore, we can write $g^\text{B}(\vx)=\lor_{i=1}^m \lnot \hat g^\text{B}_i(\vx)$ where $\hat g^\text{B}_i(\vx):=\lnot g^\text{B}_i(\vx)$ are all disjunctive literals. Due to Lemma \ref{thm:ell_inf_net_lemma}, each $\hat g^\text{B}_i(\vx)$ can be represented by an $\ell_\infty$-distance neuron that takes $\vx$ as input, and $g^\text{B}(\vx)$ can be represented by an $\ell_\infty$-distance neuron that takes the vector $(\hat g^\text{B}_1(\vx),\cdots,\hat g^\text{B}_m(\vx))^\mathrm{T}$ as input. This proves that a two-layer $\ell_\infty$-distance net suffices to represent any satisfiable discrete Boolean function. Finally, for the unsatisfiable case where the Boolean function simply outputs zero, we can construct the following two-layer $\ell_\infty$-distance net:
\begin{equation}
    f(\vx)=\left\|(\|\vx\|_\infty,\|\vx-\mathbf 1\|_\infty)^\mathrm{T}\right\|_\infty-1
\end{equation}
where $\mathbf 1$ is the all-one vector. It follows that $(\|\vx\|,\|\vx-\mathbf 1\|)^\mathrm{T}$ must be one of the following three cases: $(0,1)^\mathrm{T}$ (when $\vx=\mathbf 0$), $(1,0)^\mathrm{T}$ (when $\vx=\mathbf 1$), or $(1,1)^\mathrm{T}$ (when $\vx \neq\mathbf 0$ and $\vx \neq\mathbf 1$). Therefore $f(\vx)=0$ always hold $\forall \vx\in\{0,1\}^d$. This concludes the first part of the proof.

We next turn to the case of representing order statistics. We use the key observation that any order statistic function can be written as a nested max-min function:
\begin{equation}
\label{eq:proof_ell_inf_maxmin}
    x_{(k)}=\max_{S\subset [d],|S|=k}\min_{i\in S}x_i.
\end{equation}
This can be further equivalently written as a nested max-max function:
\begin{equation}
\label{eq:proof_ell_inf_maxmax}
    x_{(k)}=\max_{S\subset [d],|S|=k}-f^S(\vx)
\end{equation}
where $f^S(\vx)=\max_{i\in S}-x_i$. This basically concludes the proof, as the max function over a set of elements $\{x_{i_1},\cdots,x_{i_r},-x_{j_1},\cdots,-x_{j_s}\}$ can be exactly represented by an $\ell_\infty$-distance neuron following the construction in the proof of Lemma \ref{thm:ell_inf_net_lemma}. Therefore, each $f^S(\vx)$ can be represented by an $\ell_\infty$-distance neuron with input $\vx$, and $x_{(k)}$ can be represented by an $\ell_\infty$-distance neuron that takes all $f^S(\vx)$ as input. As a result, a two-layer $\ell_\infty$-distance net suffices to represent the order statistics.
\end{proof}

\begin{remark}
(On the representation efficiency of $\ell_\infty$-distance net) Based on the above proof, the width of the two-layer $\ell_\infty$-distance net in representing a Boolean function depends on the minimum number of conjunctions in its DNF, which may be exponentially large. Indeed, it is known that no DNF with a polynomial number of conjunctions can approximate the majority function defined in Section \ref{sec:symmetirc_Boolean} \citep{o2007approximation}. Similarly, in the continuous setting, the construction in (\ref{eq:proof_ell_inf_maxmin}, \ref{eq:proof_ell_inf_maxmax}) needs a network of size $\binom{d}{\lfloor d/2\rfloor}$ in order to represent the median function $k=\lfloor d/2\rfloor$, which is exponential in $d$. However, when using a \emph{multi-layer} $\ell_\infty$-distance net, the representation efficiency can be improved dramatically. For example, using the idea of sorting, it is easy to construct an $\ell_\infty$-distance net with polynomial size that represents the median function.
\end{remark}

\subsection{Proof of Theorem \ref{thm:groupsort}}
\label{sec:proof_groupsort}
The proof is almost the same as the one in Theorem \ref{thm:function_approximation}. The major difference is to replace Lemma \ref{thm:non_expressive_lemma} by the following lemma:
\begin{lemma}
\label{thm:non_expressive_groupsort}
Let $f(\vx)=\sigma(\vw_1^{\mathrm{T}}\vx+b_1, \vw_2^{\mathrm{T}}\vx+b_2)$ be a $d$-dimensional function where $\|\vw_i\|_1\le 1$ $\forall i\in[2]$ and $\sigma$ is a two-dimensional 1-Lipschitz function w.r.t. $\ell_\infty$-norm. Let $\{(\vu^{(i)},\vv^{(i)})\in \mathbb R^d\times \mathbb R^d\}_{i=1}^n$ be $n$ pairs of inputs. Then
\begin{equation*}
    \sum_{i=1}^n \left|f(\vu^{(i)})-f(\vv^{(i)})\right| \le 2\left\|\sum_{i=1}^n \left|\vu^{(i)}-\vv^{(i)}\right|\right\|_\infty .
\end{equation*}
\end{lemma}
\begin{proof}
Using the Lipschitz property of $\sigma$, we have
\begin{align}
    \notag
    \left|f(\vu^{(i)})-f(\vv^{(i)})\right|&=\left|\sigma(\vw_1^{\mathrm{T}}\vu^{(i)}+b_1,\vw_2^{\mathrm{T}}\vu^{(i)}+b_2)-\sigma(\vw_1^{\mathrm{T}}\vv^{(i)}+b_1,\vw_2^{\mathrm{T}}\vv^{(i)}+b_2)\right|\\
    &\le \max\left(\left|\vw_1^{\mathrm{T}}(\vu^{(i)}-\vv^{(i)})\right|,\left|\vw_2^{\mathrm{T}}(\vu^{(i)}-\vv^{(i)})\right|\right)\\
    &\le \max\left( |\vw_1|^{\mathrm{T}}\left|\vu^{(i)}-\vv^{(i)}\right|,|\vw_2|^{\mathrm{T}}\left|\vu^{(i)}-\vv^{(i)}\right|\right)\\
    &\le (|\vw_1|+|\vw_2|)^{\mathrm{T}}\left|\vu^{(i)}-\vv^{(i)}\right|
\end{align}
Therefore
\begin{align}
\notag
    \sum_{i=1}^n \left|f(\vu^{(i)})-f(\vv^{(i)})\right|
    &\le (|\vw_1|+|\vw_2|)^{\mathrm{T}}\sum_{i=1}^n\left|\vu^{(i)}-\vv^{(i)}\right|\\
\label{eq:proof_groupsort_holder}
    &\le \||\vw_1|+|\vw_2|\|_1\left\|\sum_{i=1}^n\left|\vu^{(i)}-\vv^{(i)}\right|\right\|_\infty,
\end{align}
where the last inequiality in (\ref{eq:proof_groupsort_holder}) follows by using the Hölder's inequality. Finally, observing that $\||\vw_1|+|\vw_2|\|_1\le \|\vw_1\|_1+\|\vw_2\|_1\le 2$, we have arrived at the conclusion.
\end{proof}

\begin{theorem}
Any $M$-layer MaxMin network $f:\mathbb R^d\to \mathbb R$ cannot approximate any $k$-th order statistic function on a bounded domain $\mathcal K= [0,1]^d$ if $M\le \lceil \log_2 d\rceil$. Moreover, there exists a point $\widehat \vx\in \mathcal K$, such that
\begin{equation}
    |f(\widehat\vx)- \widehat x_{(k)}|\ge \frac 1 2 -\frac {2^{M-2}} d\ge \frac 1 4\quad \text{if }M\le \lfloor \log_2 d\rfloor.
\end{equation}
\end{theorem}
\begin{proof}
For an $M$-layer MaxMin network, there are $M-1$ layers with the MaxMin activation. Using the same proof as in Corollary \ref{thm:non_expressive_corollary}, we have that for any set $\{(\vu^{(i)},\vv^{(i)})\in \mathbb R^d\times \mathbb R^d\}_{i=1}^n$,
\begin{equation*}
    \left\|\sum_{i=1}^n \left|f(\vu^{(i)})-f(\vv^{(i)})\right|\right\|_\infty \le 2^{M-1}\left\|\sum_{i=1}^n \left|\vu^{(i)}-\vv^{(i)}\right|\right\|_\infty .
\end{equation*}
The remaining proof is the same as the proof of Theorem \ref{thm:function_approximation}, with (\ref{eq:proof_function_approximation_x}) replaced by 
\begin{equation}
    \left\|\sum_{(\vu,\vv)\in \mathcal T}\left|f(\vu)-f(\vv)\right|\right\|_\infty \le \frac {2^{M-1}|\mathcal T|} d.
\end{equation}
\end{proof}

We can also give an upper bound of the minimum depth required. The upper bound in the following theorem shows that the lower bound of the depth in Theorem \ref{thm:groupsort} is tight up to a constant factor.

\begin{theorem}
\label{thm:groupsort_deep}
An $M$-layer MaxMin network $f:\mathbb R^d\to\mathbb R$ with dpeth $M=\mathcal O( \log_2 d)$ can exactly represent any $d$-dimensional order statistic function on a bounded domain $\mathcal K$.
\end{theorem}
\begin{proof}
Our proof leverages the work of Ajtai et al. \citep{ajtai19830}, who constructed a so-called \emph{sorting network} which has depth $\mathcal O(\log_2 d)$. The basic component of a sorting network is called the \emph{comparator}. Each comparator connects two inputs and swaps the values if and only if the value of the first input is greater than the second. The comparators have a hierarchical structure (arranged layer by layer), and the execution is parallel for all comparators in the same layer. To avoid conflicts, no input can be simultaneously connected by two comparators in the same layer. We will show that given a sorting network (denoted as $\operatorname{sort}$), one can construct an equivalent MaxMin network $\vf:\mathbb R^d\to\mathbb R^d$ that has the same depth, such that for some fixed permutation $\pi\in S_d$, $\operatorname{sort}(\vx)=(f_{\pi_1}(\vx),\cdots,f_{\pi_d}(\vx))^\mathrm{T}$ holds for all inputs $\vx\in \mathcal K$.

Intuitively, each pair of MaxMin neurons can perform the same operation as a comparator. This is simply because
\begin{equation*}
    (\max(x_i,x_j),\min(x_i,x_j))=\operatorname{MaxMin}(\ve_i^\mathrm{T}\vx,\ve_j^\mathrm{T}\vx).
\end{equation*}
For those inputs that are not connected to any comparator, one can construct an identity function to propagate the input without change, by using a pair of MaxMin neurons with properly chosen bias values so that one neuron is always larger than the other:
\begin{equation*}
    (x_i,b)=\operatorname{MaxMin}(\ve_i^\mathrm{T}\vx,b)\quad \text{where }b\to -\infty.
\end{equation*}
Note that we waste the neuron that always outputs $b$. Eventually, we have constructed a MaxMin network that can output the sorting result (differing by up to a permutation) within $\mathcal O( \log_2 d)$ depth. Therefore, by adding one more layer that extracts the $k$-th order statistic, the network then exactly represents the $k$-th order statistic function.
\end{proof}

\subsection{Proof of Theorem \ref{thm:groupsort_boolean}}
\label{sec:proof_groupsort_boolean}
\begin{proposition}
\label{thm:groupsort_upper_bound_boolean}
An $M$-layer MaxMin network can exactly represent any (discrete) $d$-dimensional Boolean function if $M> d + \lceil \log_2 d\rceil$. 
\end{proposition}
\begin{proof}
First, it is easy to see that a pair of MaxMin neurons can represent any two-dimensional literal conjunctions/disjunctions, i.e. $r_i\land r_j$ and $r_i\lor x_j$, where $r_i$ is either $x_i$ or $\lnot x_i$. For example, $\max(x_i,x_j)=x_i\lor x_j$, $\min(x_i,x_j)=x_i\land x_j$, $\min(1-x_i,1-x_j)=\lnot x_i\land \lnot x_j$, etc. Next, using reduction, any $d$-dimensional literal conjunction/disjunction can be represented by a $(\lceil \log_2 d\rceil+1)$-layer MaxMin network (here $+1$ means that the last layer does not have MaxMin activation). Finally, recall that any Boolean function can be written in its DNF (Appendix \ref{sec:dnf}) and the number of disjunctions cannot exceed $2^d$ for any $d$-dimensional Boolean function, we have proved that a MaxMin network with depth $M= d + \lceil \log_2 d\rceil+1$ can represent any Boolean function.
\end{proof}

In the following, we will prove that the bound of $M=\mathcal O(d)$ is \emph{tight} in order to interpolate $d$-dimensional Boolean functions. To do this, we first define several notations.

\begin{definition}
Let $\mathcal D=\{(\vx^{(i)},y^{(i)})\in \mathbb R^d\times \{0,1\}\}_{i=1}^n$ be a dataset with real vectors as inputs and Boolean labels, and the inputs have a bounded diameter of 1 w.r.t. $\ell_\infty$-norm, i.e. $\|\vx^{(j)}-\vx^{(k)}\|_\infty\le 1$ $\forall j,k\in[n]$. 
\begin{itemize}[topsep=0pt,leftmargin=30pt]
\setlength{\itemsep}{0pt}
    \item Denote $l_i(\mathcal D)=\min_{j\in[n]} x_i^{(j)}$ and $u_i(\mathcal D)=\max_{j\in[n]} x_i^{(j)}$. We say the $i$-th input element is \emph{useful} if $u_i(\mathcal D)-l_i(\mathcal D)=1$; Correspondingly, the $i$-th input element is \emph{useless} if $u_i(\mathcal D)-l_i(\mathcal D)<1$.
    \item Given an index set $\mathcal I\subset [d]$, define the \emph{$\mathcal I$-masked} dataset $\mathcal D_{\mathcal I}^{\text{M}}$ to be the dataset obtained by zeroing out the $i$-th input element of each sample for each $i\in \mathcal I$. Specifically, $\mathcal D_{\mathcal I}^{\text{M}}:=\{(\tilde{\vx}^{(j)},y^{(j)})\}_{j=1}^n$ where
    \begin{equation*}
        \tilde{x}^{(j)}_i=\left\{\begin{array}{cc}
             x^{(j)}_i & i\notin \mathcal I, \\
             0 & i\in \mathcal I.
        \end{array}\right.
    \end{equation*}
    \item Given an index set $\mathcal I\subset [d]$, define the \emph{$\mathcal I$-thresholded} dataset $\mathcal D_{\mathcal I}^{\text{T}}$ to be any dataset satisfying that, for each $i\in \mathcal I$, change the $i$-th input element of all samples to the extreme values, either the minimum $l_i(\mathcal D)$ or the maximum $u_i(\mathcal D)$, but keep the original minimum and maximum elements unchanged. Specifically, $\mathcal D_{\mathcal I}^{\text{T}}:=\{(\tilde{\vx}^{(j)},y^{(j)})\}_{j=1}^n$ where
    \begin{equation*}
    \left\{\begin{array}{ll}
        \tilde{x}^{(j)}_i= x^{(j)}_i  & \text{if}\quad i\notin \mathcal I,\\
        \tilde{x}^{(j)}_i= x^{(j)}_i  & \text{if}\quad i\in \mathcal I,\ {x}^{(j)}_i= l_i(\mathcal D)\text{ or } x_i^{(i)}=u_i(\mathcal D),\\
        \tilde{x}^{(j)}_i\in \{l_i(\mathcal D),u_i(\mathcal D)\} & \text{if}\quad i\in \mathcal I,\  l_i(\mathcal D)<{x}^{(j)}_i< u_i(\mathcal D).
    \end{array}\right.
    \end{equation*}
    Clearly, the $\mathcal I$-thresholded dataset is not unique.
    \item We say a new dataset $\mathcal D^{\text{F}}$ is \emph{formatted} from dataset $\mathcal D$, if $\mathcal D^{\text{F}}$ can be obtained by \emph{masking} all the \emph{useless} elements in $\mathcal D$ and \emph{thresholding} all \emph{useful} elements in $\mathcal D$.
\end{itemize}

\end{definition}

\begin{lemma}
\label{thm:maxmin_sparse}
Let $\mathcal D=\{(\vx^{(i)},y^{(i)})\in \mathbb R^d\times \{0,1\}\}_{i=1}^n$ be a finite dataset with a bounded diameter of 1 w.r.t. $\ell_\infty$-norm. If $\mathcal D$ can be interpolated by a MaxMin network, then any dataset $\mathcal D^{\text{F}}$ formatted from $\mathcal D$ can be interpolated by a MaxMin network with the same structure (i.e. depth and width).
\end{lemma}
\begin{proof}
Our proof is based on induction over the network depth. For the base case, consider the network that performs the identify mapping: $f(x)=x$ where $x\in\mathbb R$ is a scalar. If $f$ interpolates the dataset $\mathcal D$, then $d=1$ and  $\vx^{(i)}\in\{0,1\}$ $\forall i\in [n]$. Therefore, $\mathcal D^{\text{F}}=\mathcal D$ and $f$ interpolates $\mathcal D^{\text{F}}$.

Now assume the conclusion holds for a MaxMin network $f:\mathbb R^{d_f}\to \mathbb R$. We would like to prove that it also holds for the MaxMin network $f\circ \vh$ where $\vh:\mathbb R^{d_h}\to \mathbb R^{d_f}$ is either the affine layer or the MaxMin activation layer. This will yield the conclusion, as any MaxMin network can be constructed by starting with the identity function and stacking affine layers and MaxMin activations (from the last layer to the first).

\textbf{Case 1}: $\vh:\mathbb R^{d_h}\to \mathbb R^{d_f}$ is an affine layer. We denote $h_i(\vx)=\vw_i^{\mathrm{T}}\vx+b_i$ to be the $i$-th element of $\vh$ with weight $\vw_i$ and bias $b_i$. Assume $f\circ \vh$ can interpolate dataset $\mathcal D$. We will construct a new MaxMin network $\tilde f\circ \tilde\vh$ such that $\tilde f$ has the same topology as $f$, $\tilde\vh:\mathbb R^{d_h}\to \mathbb R^{d_f}$ is an affine layer with weights $\tilde \vw_i$ and biases $\tilde b_i$ ($i\in[d_f])$, and $\tilde f\circ \tilde\vh$ interpolates $\mathcal D^{\text{F}}$. Note that $\|\vw_i\|_1\le 1$ and $\|\tilde \vw_i\|_1\le 1$ must hold ($\forall i\in [d_f]$).

Given dataset $\mathcal D$, denote $\vz^{(j)}=\vh(\vx^{(j)})$ $\forall j\in [n]$. Applying $\vh$ to $\mathcal D$ yields another dataset denoted as $\vh(\mathcal D):=\{(\vz^{(j)},y^{(j)})\}_{j=1}^n$, which also has a bounded diameter of 1 due to the 1-Lipschitz property of $\vh$. By the assumption, $f$ interpolates $\vh(\mathcal D)$. Therefore, one can use induction to prove that there exists $\tilde f$ that interpolates $(\vh(\mathcal D))^{\text{F}}$ where $(\vh(\mathcal D))^{\text{F}}$ is any dataset formatted from $\vh(\mathcal D)$. 

We now construct the weights and biases of $\tilde \vh$ as follows:
\begin{itemize}[topsep=0pt,leftmargin=30pt]
    \item If the $i$-th element of $\vh(\mathcal D)$ is useless (i.e. $u_i(\vh(\mathcal D))-l_i(\vh(\mathcal D))<1$), then set $\tilde \vw_i=\mathbf 0$ and $\tilde b_i=0$;
    \item If the $i$-th element of $\vh(\mathcal D)$ is useful, then there exists two samples $\vx^{(j)}$ and $\vx^{(k)}$ satisfying $|z_i^{(j)}-z_i^{(k)}|=1$. Therefore, $|\vw_i^{\mathrm{T}}(\vx^{(j)}-\vx^{(k)})|=1$. Note that $\|\vw_i\|_1\le 1$ and $\|\vx^{(j)}-\vx^{(k)}\|_\infty\le 1$. Thus by Hölder's inequality, it must be $\|\vx^{(j)}-\vx^{(k)}\|_\infty= 1$ and $[\vw_i]_s=0$ for all $s$ satisfying $|x_s^{(j)}-x_s^{(k)}|<1$. Pick any $s$ with $[\vw_i]_s\neq 0$ (thus $|x_s^{(j)}-x_s^{(k)}|=1$). Choose $\tilde\vw_i=\operatorname{sgn}([\vw_i]_s)\ve_s$ (the unit vector with the $s$-th element being 1 or -1 depending on the sign of $[\vw_i]_s$) and $\tilde b_i=(\vw_i-\tilde\vw_i)^{\mathrm{T}}\vx^{(j)}+b_i$.
\end{itemize}

Similarly, denote $\tilde \vz^{(j)}=\tilde \vh(\vx^{(j)})$ $\forall j\in [n]$ and the the corresponding dataset $\tilde\vh(\mathcal D):=\{(\tilde\vz^{(j)},y^{(j)})\}_{j=1}^n$ (which also has a bounded diameter of 1 due to the 1-Lipschitz property of $\tilde \vh$). First observe that if the $i$-th element is useless in $\vh(\mathcal D)$, then it is also useless in $\tilde\vh(\mathcal D)$. We now prove that the converse also holds: if the $i$-th element is useful in $\vh(\mathcal D)$, then it is also useful in $\tilde\vh(\mathcal D)$. Otherwise, there exists $\vx^{(j)}$ and $\vx^{(k)}$ such that $|\vw_i^{\mathrm{T}}(\vx^{(j)}-\vx^{(k)})|=1$ but $|\tilde\vw_i^{\mathrm{T}}(\vx^{(j)}-\vx^{(k)})|<1$. However, in this case, by construction $\tilde\vw_i=\pm\ve_s$ is the unit vector for some $s$, implying that $|x_s^{(j)}-x_s^{(k)}|<1$. Since $[\vw_i]_s\neq 0$, Hölder's inequality implies that $|\vw_i^{\mathrm{T}}(\vx^{(j)}-\vx^{(k)})|<1$, a contradiction.

Therefore, the $i$-th element is useful in $\vh(\mathcal D)$ if and only if it is useful in $\tilde\vh(\mathcal D)$. Now assume that the $i$-th element is both useful in $\vh(\mathcal D)$ and $\tilde\vh(\mathcal D)$. We next prove that $l_i(\vh(\mathcal D))=l_i(\tilde\vh(\mathcal D))$ and $u_i(\vh(\mathcal D))=u_i(\tilde\vh(\mathcal D))$. By the assignment of $\tilde b_i$, there exists an $\vx^{(j)}$ such that $\vw_i^{\mathrm{T}}\vx^{(j)}+b_i=\tilde\vw_i^{\mathrm{T}}\vx^{(j)}+\tilde b_i$, namely $z_i^{(j)}=\tilde z_i^{(j)}$. Since $z_i^{(j)}$ must at the extreme value, without loss of generality, assume $z_i^{(j)}=l_i(\vh(\mathcal D))$. Based on the above paragraph, for any $\vx^{(k)}$ such that $|\vw_i^{\mathrm{T}}(\vx^{(j)}-\vx^{(k)})|=1$, $|\tilde\vw_i^{\mathrm{T}}(\vx^{(j)}-\vx^{(k)})|=1$. Thus for any $z_i^{(k)}=u_i(\vh(\mathcal D))$, $|z_i^{(j)}-z_i^{(k)}|=|\tilde z_i^{(j)}-\tilde z_i^{(k)}|$. Also, $| z_i^{(k)}-\tilde z_i^{(k)}|<2$ because by construction $\|\tilde\vw-\vw\|_1<2$. This proves that $\tilde z_i^{(j)}=l_i(\tilde\vh(\mathcal D))$ and $\tilde z_i^{(k)}=u_i(\tilde\vh(\mathcal D))$, thus $l_i(\vh(\mathcal D))=l_i(\tilde\vh(\mathcal D))$ and $u_i(\vh(\mathcal D))=u_i(\tilde\vh(\mathcal D))$. Note that we have actually proved a stronger result: given samples $\vz^{(k)}$ and $\tilde \vz^{(k)}$, $z_i^{(k)}=l_i(\vh(\mathcal D))\iff \tilde z_i^{(k)}=l_i(\vh(\mathcal D))$, and $z_i^{(k)}=u_i(\vh(\mathcal D))\iff \tilde z_i^{(k)}=u_i(\vh(\mathcal D))$.

Combining the above results, one can conclude that for a dataset $(\vh(\mathcal D))^{\text{F}}$ formatted from $\vh(\mathcal D)$, it is also formatted from $\tilde\vh(\mathcal D)$. \hfill (*)

Now let $\tilde\vh(\mathcal D^{\text{F}})$ be the dataset obtained by applying $\tilde\vh$ to the formatted dataset $\mathcal D^{\text{F}}$. Note that the weights in $\tilde \vh$ are special in that all $\tilde \vw_i$ are either unit or zero vectors by construction. We will use this property to prove that $\tilde\vh(\mathcal D^{\text{F}})$ is just the dataset formatted from $\tilde\vh(\mathcal D)$. The key observation is that
\begin{itemize}[topsep=0pt,leftmargin=30pt]
\setlength{\itemsep}{0pt}
    \item If $[\tilde \vw_i]_s= 0$, then masking/thresholding the $s$-th element in $\mathcal D$ does not change the network output after applying $\tilde h_i$.
    \item If $[\tilde \vw_i]_s\neq 0$, then the $s$-th element must be useful in $\mathcal D$ (bacause the $i$-th element is useful in $\vh(\mathcal D)$ and $[\vw_i]_s\neq 0$). Since $\tilde \vw_i$ is unit ($[\tilde \vw_i]_s=\pm1$), it is easy to see that thresholding the $s$-th element in $\mathcal D$ does not change the minimum and maximum of the output $\tilde h_i$.
\end{itemize}
Using these results as well as the fact that each input element across $\tilde\vh(\mathcal D^{\text{F}})$ can have at most two different values (because the weight of $\tilde\vh$ is either unit or zero and $\mathcal D^{\text{F}}$ is formatted), we have that $\tilde\vh(\mathcal D^{\text{F}})$ is just the dataset formatted from $\tilde\vh(\mathcal D)$! Finally, by induction, $(\vh(\mathcal D))^{\text{F}}$ can be interpolated by $\tilde f$, which yields that $(\tilde \vh(\mathcal D))^{\text{F}}$ can be interpolated by $\tilde f$ based on (*). Therefore, $\tilde \vh(\mathcal D^{\text{F}})$ can be interpolated by $\tilde f$, namely $\tilde f\circ \tilde\vh$ interpolates $\mathcal D^{\text{F}}$, which concludes the proof of Case 1.

\textbf{Case 2}: $\vh:\mathbb R^{d_h}\to \mathbb R^{d_h}$ is the MaxMin layer. Here we consider the MaxMin layer with learnable bias, i.e. $\vh(\vx)=\operatorname{MaxMin}(\vx+\vb)$ with parameter $\vb\in \mathbb R^{d_h}$. Note that the bias will be absorbed into the previous linear layer so it does not change the resulting architecture, but it facilities the proof here. Assume $f\circ \vh$ can interpolate dataset $\mathcal D$. We will construct a new MaxMin network $\tilde f\circ \tilde \vh$ such that $\tilde f$ has the same topology as $f$, $\tilde \vh$ is the MaxMin layer with bias $\tilde \vb$, and $\tilde f\circ\tilde \vh$ interpolates $\mathcal D^{\text{F}}$. As in Case 1, we use the same notations of $\vh(\mathcal D)$, $\tilde\vh(\mathcal D)$, $(\vh(\mathcal D))^{\text{F}}$, $\vh(\mathcal D^{\text{F}})$, etc.

Given dataset $\mathcal D$, denote $\vz^{(j)}=\vh(\vx^{(j)})$ $\forall j\in [n]$. It follows that $z_{2i-1}^{(j)}=\max(x_{2i-1}^{(j)}+b_{2i-1},x_{2i}^{(j)}+b_{2i})$ and $z_{2i}^{(j)}=\min(x_{2i-1}^{(j)}+b_{2i-1},x_{2i}^{(j)}+b_{2i})$. We can separately focus on each pair. Consider the following three sub-cases:
\begin{itemize}[topsep=0pt,leftmargin=30pt]
\setlength{\itemsep}{0pt}
    \item Both the $(2i-1)$-th and the $2i$-th elements in $\vh(\mathcal D)$ are useful. In this sub-case, we can prove that both the $(2i-1)$-th and the $2i$-th elements in $\mathcal D$ are useful. This is because
    \begin{align*}
        2&=u_{2i-1}(\vh(\mathcal D))-l_{2i-1}(\vh(\mathcal D))+u_{2i}(\vh(\mathcal D))-l_{2i}(\vh(\mathcal D))\\
        &=\left(\max_{j\in[n]}z_{2i-1}^{(j)}+\max_{j\in[n]}z_{2i}^{(j)}\right)-\left(\min_{j\in[n]}z_{2i-1}^{(j)}+\min_{j\in[n]}z_{2i}^{(j)}\right)\\
        &\le \left(\max_{j\in[n]}x_{2i-1}^{(j)}+\max_{j\in[n]}x_{2i}^{(j)}+b_{2i-1}++b_{2i}\right)-\left(\min_{j\in[n]}x_{2i-1}^{(j)}+\min_{j\in[n]}x_{2i}^{(j)}+b_{2i-1}++b_{2i}\right)\\
        &=u_{2i-1}(\mathcal D)-l_{2i-1}(\mathcal D)+u_{2i}(\mathcal D)-l_{2i}(\mathcal D)\le 2.
    \end{align*}
    The equality holds when $u_{2i-1}(\mathcal D)-l_{2i-1}(\mathcal D)=1$ and $u_{2i}(\mathcal D)-l_{2i}(\mathcal D)=1$. Also, it is easy to see that thresholding the $(2i-1)$-th and the $2i$-th elements in $\mathcal D$ is equivalent to thresholding them in $\vh(\mathcal D)$. Formally, $l_{2i-1}(\vh(\mathcal D^{\text{F}}))=l_{2i-1}(\vh(\mathcal D))$ and $u_{2i-1}(\vh(\mathcal D^{\text{F}}))=u_{2i-1}(\vh(\mathcal D))$ (similar for $l_{2i}$ and $u_{2i}$). Therefore, we just set $\tilde b_{2i-1}=b_{2i-1}$ and $\tilde b_{2i}=b_{2i}$.
    \item Both the $(2i-1)$-th and the $2i$-th elements in $\vh(\mathcal D)$ are useless. This sub-case is trivial: these useless elements can be handled in $\tilde f$, because the first layer in $\tilde f$ is linear and the corresponding weights are zero for the useless input elements (see Case 1).
    \item One of the $(2i-1)$-th and the $2i$-th elements in $\vh(\mathcal D)$ is useful and the other is useless. Without loss of generality, assume the $(2i-1)$-th element in $\vh(\mathcal D)$ is useful and the $2i$-th element in $\vh(\mathcal D)$ is useless. Then we can prove that either ($l_{2i-1}(\vh(\mathcal D))=l_{2i-1}(\mathcal D)$ and $u_{2i-1}(\vh(\mathcal D))=u_{2i-1}(\mathcal D)$) or ($l_{2i-1}(\vh(\mathcal D))=l_{2i}(\mathcal D)$ and $u_{2i-1}(\vh(\mathcal D))=u_{2i}(\mathcal D)$). Otherwise, it is easy to see that both the $(2i-1)$-th and the $2i$-th elements in $\mathcal D$ are useless and thus the $(2i-1)$-th element in $\vh(\mathcal D)$ cannot be useful.
    
    Again without loss of generality, assume the $l_{2i-1}(\vh(\mathcal D))=l_{2i-1}(\mathcal D)$ and $u_{2i-1}(\vh(\mathcal D))=u_{2i-1}(\mathcal D)$. Then by assigning a sufficiently small bias $\tilde b_{2i}\to -\infty$ and setting $\tilde b_{2i-1}= b_{2i-1}$, one can ensure that even after thresholding/masking $\mathcal D$, the following relation still holds: $l_{2i-1}(\vh(\mathcal D^{\text{F}}))=l_{2i-1}(\mathcal D^{\text{F}})$ and $u_{2i-1}(\vh(\mathcal D^{\text{F}}))=u_{2i-1}(\mathcal D^{\text{F}})$. On the other hand, the $2i$-th element in $\vh(\mathcal D)$ is useless so it can be handled in $\tilde f$ by the corresponding zero weights in the first layer of $\tilde f$ (see Case 1).
\end{itemize}
Combining the three sub-cases, we have concluded the proof of Case 2.
\end{proof}

Note that for a Boolean dataset $\mathcal D$, the formatted version is itself, i.e. $\mathcal D^{\text{F}}=\mathcal D$. Thus the proof of Lemma \ref{thm:maxmin_sparse} directly leads to the following corollary:

\begin{corollary}
\label{thm:groupsort_weight_sparse}
Let $\mathcal D$ be a Boolean dataset that can be interpolated by a MaxMin network $f$. Then there exists a network $\tilde f$ with the same architecture (i.e. depth and width) as $f$, such that:
\begin{itemize}[topsep=0pt,leftmargin=30pt]
\setlength{\itemsep}{0pt}
    \item Denote $\mathbf W^{(l)}\in \mathbb R^{d_l\times d_{l-1}}$ as the weight matrix of the $l$-th layer of $\tilde f$. Then the following holds for all $i\in [d_l]$:
    \begin{equation}
        [\mathbf W^{(l)}]_{i,:}\in\{s\cdot\ve_r:s\in\{1,-1\},r\in[d_{l-1}]\}\cup\{\mathbf 0\}.
    \end{equation}
    \vspace{-15pt}
    \item Denote $f_i^{(l)}(\vx)$ as the $i$-th neuron output of the $l$-th layer given input $\vx$, and denote $\mathcal S_i^{(l)}=\{f_i^{(l)}(\vx):(\vx,y)\in \mathcal D\}$. Then either $|\mathcal S_i^{(l)}|=1$, or $|\mathcal S_i^{(l)}|=2$ and $\max \mathcal S_i^{(l)}  - \min \mathcal S_i^{(l)}=1$.
\end{itemize}
\end{corollary}

We are now ready for the proof of Theorem \ref{thm:groupsort_boolean}, which is restated below.

\begin{theorem}
Let $M_d$ be the minimum depth such that an $M_d$-layer MaxMin network can exactly represent any (discrete) $d$-dimensional Boolean function. Then $M_d=\Omega(d)$.
\end{theorem}
\begin{proof}
Let $\tilde f$ be an $M_d$-layer MaxMin network that interpolates the $d$-dimensional Boolean function, satisfying the condition in Corollary \ref{thm:groupsort_weight_sparse}. Using the notations in Corollary \ref{thm:groupsort_weight_sparse}, one has
\begin{equation}
    f_{2i-1}^{(l)}(\vx)=\max\left( [\mathbf W^{(l)}]_{2i-1,:}^{\mathrm{T}}\vf^{(l-1)}(\vx)+b^{(l)}_{2i-1}, [\mathbf W^{(l)}]_{2i,:}^{\mathrm{T}}\vf^{(l-1)}(\vx)+b^{(l)}_{2i}\right)
\end{equation}
which must have one of the following forms due to the choice of weight $[\mathbf W^{(l)}]$:
\begin{itemize}[topsep=0pt,leftmargin=30pt]
\setlength{\itemsep}{0pt}
    \item $f_{2i-1}^{(l)}(\vx)=c$, where $c\in \mathbb R$ is a constant;
    \item $f_{2i-1}^{(l)}(\vx)=\pm f_r^{(l-1)}(\vx)+c$, where $r\in[d_{l-1}]$ and $c\in \mathbb R$ are constants;
    \item $f_{2i-1}^{(l)}(\vx)=\max\left(s_1\cdot f_{r_1}^{(l-1)}(\vx),s_2\cdot f_{r_2}^{(l-1)}(\vx)\right)+c$, where $s_1,s_2\in\{1,-1\}$, $r_1,r_2\in[d_{l-1}]$ and $c\in \mathbb R$ are constants.
\end{itemize}
In the last case, we further have $\max_{(\vx,y)\in\mathcal D}s_1\cdot f_{r_1}^{(l-1)}(\vx)=\max_{(\vx,y)\in\mathcal D}s_2\cdot f_{r_2}^{(l-1)}(\vx)$ and $\min_{(\vx,y)\in\mathcal D}s_1\cdot f_{r_1}^{(l-1)}(\vx)=\min_{(\vx,y)\in\mathcal D}s_2\cdot f_{r_2}^{(l-1)}(\vx)$ due to the second property in Corollary \ref{thm:groupsort_weight_sparse}. Therefore, $f_{2i-1}^{(l)}(\vx)$ can calculate a constant, the identity function, or the 2-ary logical OR of \emph{literals} (up to an additional bias). Similarly, $f_{2i}^{(l)}(\vx)$ can calculate a constant, the identity function, or the logical AND of \emph{literals}.

Therefore, when considering the whole MaxMin network, it is no powerful than a 2-ary Boolean circuit. Here a 2-ary Boolean circuit refers to a directional acyclic computation graph where each internal node connects to at most two nodes by incoming edges and can only calculate the 2-ary logical AND, logical OR, and (unary) logical NOT. Define the depth of a 2-ary Boolean circuit to be the length of the longest path from the input to the output. It then follows that the depth correponds to the depth of the MaxMin netwotk, and it suffices to prove that the depth of a 2-ary Boolean circuit must be $\Omega(d)$ to represent certain Boolean functions of $d$ variables.

Note that for a 2-ary Boolean circuit that has $M$ layers and outputs a scalar, the number of nodes will not exceed $2^{M+1}-1$, where the maximum size is achieved by a complete binary tree. However, the classic result in Boolean circuit theory (Shannon 1942) showed that for most Boolean functions of $d$ variables, a lower bound on the minimum size of 2-ary Boolean circuits is $\Omega(2^d/d)$ , which thus yields $M=\Omega(d)$ and concludes the proof.
\end{proof}



\section{Special SortNet}
\label{sec:special_sortnet}
\subsection{GroupSort Network}
\label{sec:special_sortnet_groupsort}
In this section, we formally prove that any GroupSort network can be exactly represented by a SortNet with the same topological structure. An $M$-layer GroupSort network with group size $G$ and hidden size $Gd$ can be defined as:
\begin{equation}
\label{eq:group_sort}
\begin{aligned}
    &\left\{
    \begin{array}{ll}
         \widetilde \vz^{(l)}=\widetilde{\mathbf W}^{(l)}\widetilde\vx^{(l-1)}+\widetilde\vb^{(l)} & l\in[M], \\
         \widetilde x^{(l)}_{iG+j}= \left(\{\widetilde z^{(l)}_{iG+t}\}_{t=1}^G\right)_{(j)} & l\in[M-1], i\in\{0,\cdots,d-1\}, j\in[G],
    \end{array}
    \right.\\
    &s.t. \ \|\widetilde{\mathbf W}^{(l)}\|_\infty\le 1 \text{ for } l\in [M].
\end{aligned}
\end{equation}
In (\ref{eq:group_sort}), $(\cdot)_{(j)}$ denotes the $j$-th largest element in a sequence. The GroupSort network takes $\widetilde\vx^{(0)}$ as input and outputs $\widetilde\vz^{(M)}$. We will construct a SortNet such that the output is equal to $\widetilde\vz^{(M)}$.

The key observation is that for any vector $\vz$ in a bounded domain $\mathbb K$, the GroupSort activation with arbitrary group size $G$ can be represented by a sort operation plus a bias term. Concretely, we prove
\begin{equation}
\label{eq:group_sort_bias}
    \operatorname{GroupSort}_G(\vz)=\operatorname{sort}(\vz+\vb)+\vc
\end{equation}
by assigning
\begin{align}
    b_{iG+j}=-iC,\quad c_{iG+j}=iC\quad \text{for}\ j\in [G],i\in \{0,\cdots,d-1\}
\end{align}
where $C$ is a sufficiently large positive constant. This is because a large $C$ dominates the relative order, so the sorting result is divided into groups of size $G$. Then $\vb+\vc=0$ ensures that the introduced bias is finally eliminated.

Equipped with (\ref{eq:group_sort_bias}), we are now ready to construct the SortNet. We use the notations in Definition \ref{def:sortnet}. Denote vectors $\va$ and $\vb$ as
\begin{align*}
\begin{array}{ll}
     a_{j}=-jC & \text{for}\ j\in [Gd], \\
     b_{iG+j}=-iC & \text{for}\ j\in [G],i\in [d].
\end{array}
\end{align*}
The parameters of SortNet are assigned below:
\begin{equation*}
\begin{array}{ll}
    \vw^{(l,k)}=[\widetilde{\mathbf W}^{(l)}]_{k,:}\  \text{(the }k\text{th row)} & k\in [Gd]\\
    \vb^{(1,k)}=\va & k\in [Gd]\\
    \vb^{(2,k)}=\vb+\widetilde \vb^{(1)}-\widetilde{\mathbf W}^{(1)}\va     &k\in [Gd]\\
    \vb^{(l,k)}=\vb+\widetilde \vb^{(l-1)}-\widetilde{\mathbf W}^{(l-1)}\vb     &k\in [Gd], 3\le l\le M\\
    \vb^{\text{out}}=\widetilde \vb^{(M)}-\widetilde{\mathbf W}^{(M)}\vb        & 
\end{array}
\end{equation*}
The activation is simply chosen to be the identity function $\sigma(x)=x$. This leads to the following calculation which can be proved straightforwardly:
\begin{align*}
    &\vx^{(1)}=\widetilde\vz^{(1)}-\widetilde\vb^{(1)}+\widetilde{\mathbf W}^{(1)}\va\\
    &\vx^{(l)}=\widetilde\vz^{(l)}-\widetilde\vb^{(l)}+\widetilde{\mathbf W}^{(l)}\vb \quad \text{for } 2\le l\le M
\end{align*}
Finally, the SortNet outputs $\vf(\vx)=\vx^{(M)}+\vb^{\text{out}}=\widetilde \vz^{(M)}$, which matches the output of the GroupSort network. Proof finishes.

\begin{remark}
In the above construction, the same value is assigned for all biases $\vb^{(l,k)}$ of each SortNet layer. It implies that SortNet is more powerful than GroupSort, since SortNet allows diverse biases for different neurons in the same layer. This is very useful because the value of bias plays an important role in controlling the sorting behavior. For example, in a SortNet neuron (\ref{eq:sortnet_neuron}), setting $b_i^{(l,k)}$ to a sufficiently large value will dominant $x_i^{(l-1)}$ to rank the first, which can extract the $i$-th coordinate of the input vector if $\vw^{(l,k)}=\ve_1$. Similarly, setting $b_i^{(l,k)}$ to a large negative value will cause $x_i^{(l-1)}$ to rank the last, and when combined with $w_i^{(l,k)}=0$, will preclude $x_i^{(l-1)}$ in the computation of the neuron output. However, for a GroupSort layer, performing these operations will affect the output of other neurons as all neurons in the group share the same bias vector and sorting operation. The interference between neurons may lead to some neurons being wasted and not being able to express the required operations, which is undesirable. In contrast, SortNet does not face the above problems.
\end{remark}

\subsection{$\ell_\infty$-distance Net}
\label{special_sortnet_ell_inf_net}
In this section we formally prove that any $\ell_\infty$-distance net can be exactly represented by a special SortNet with the same topological architecture. An $M$-layer $\ell_\infty$-distance net with a hidden size of $d$ can be formally defined as:
\begin{equation}
\label{eq:ell_inf_net}
    \widetilde x^{(l)}_k=\|\widetilde\vx^{(l-1)}-\widetilde \vw^{(l,k)}\|_\infty+\widetilde b^{(l)}_k,\quad  l\in[M], k\in[d].
\end{equation}
The network takes $\widetilde\vx^{(0)}$ as input and outputs $\widetilde\vx^{(M)}$. We will construct a SortNet such that the output exactly matches $\widetilde\vx^{(M)}$. 

The key observation is that for any input vector $\vz$ and parameter $\vw$,
\begin{equation}
    \|\vz- \vw\|_\infty=\max_i |z_i- w_i|=\ve_1^{\mathrm{T}}\operatorname{sort}(|\vz- \vw|)
\end{equation}
where $\ve_1$ is the unit vector with the first element being one. Based on the above equation, we now construct a SortNet as follows. We use the notations in Definition \ref{def:sortnet} and set all weights $\vw^{(l,k)}=\ve_1$, set biases
\begin{align*}
    &\vb^{(1,k)}=-\widetilde \vw^{(1,k)},\\
    &\vb^{(l,k)}=\widetilde \vb^{(l-1)}-\widetilde \vw^{(l,k)}\quad \text{for} \ 2\le l\le M,\\
    &\vb^{\text{out}}=\widetilde \vb^{(M)},
\end{align*}
and set the activation $\sigma(x)=|x|$. Then by a simple induction over layer $l$ we can prove
\begin{equation}
    \vx^{(l)}=\widetilde \vx^{(l)}-\widetilde \vb^{(l)}.\quad \forall\ l\in [M]
\end{equation}
Finally, the SortNet outputs $\vf(\vx)=\vx^{(M)}+\vb^{\text{out}}=\widetilde \vx^{(M)}$, which matches the output of the $\ell_\infty$-distance net. Proof finishes.

\section{An Efficient GPU Implementation of SortNet}
\label{sec:gpu_implementaion}
In this section we describe an efficient GPU implementation of Section \ref{sec:sortnet} for training and inference, which is used in this paper. Some basic tensor operations are described in the context of Pytorch framework.

\subsection{Training}
\label{sec:dropout_training}
Consider a fully-connected SortNet layer $\vf$ with parameters $\mathbf W,\mathbf B\in \mathbb R^{d_{\text{out}}\times d_{\text{in}}}$, and $W_{ij}=(1-\rho)\rho^{j-1}$. For an input batch $\mathbf X\in\mathbb R^{n\times d_{\text{in}}}$ consisting of $n$ samples, denote the output in a training iteration as $\mathbf Z=\vf(\mathbf X)\in\mathbb R^{n\times d_{\text{out}}}$ which is stochastic and depends on the mask. The generated stochastic mask is also a matrix denoted as $\mathbf S\in \{0,1\}^{n\times d_{\text{in}}}$. According to the formula (\ref{eq:expectation_dropout}), 
\begin{equation}
\label{eq:gpu_naive}
    Z_{ik}=\max_{j\in [d_{\text{in}}]} S_{ij}\sigma(X_{ij}+B_{kj})
\end{equation}
Note that $S_{ij}$ can only be 0 or 1, therefore the multiplication between $S_{ij}$ and $\sigma(X_{ij}-B_{kj})$ is unnecessary. We can enumerate the indices such that $S_{ij}=1$, resulting in the following calculation:
\begin{equation}
\label{eq:gpu_naive2}
    Z_{ik}=\max_{j,\ \text{s.t.}\ S_{ij}=1} \sigma(X_{ij}+B_{kj})
\end{equation}
With hyper-parameter $\rho$, (\ref{eq:gpu_naive2}) will give a computational complexity of $\Theta((1-\rho)nd_{\text{in}}d_{\text{out}})$.

In fact, the above computation can be further accelerated if two adjacent SortNet layers are cascaded. In this case, the output of the first layer becomes the input of the second layer, which will also be dropped out, so the unused neurons does not need to be computed. In other words, it suffices to only calculate the output neurons which are involved in the next layer of computation. This results in a computational complexity of $\Theta((1-\rho)^2 nd_{\text{in}}d_{\text{out}})$.

Things become more complicated when it turns to a GPU implementation. Since GPU is highly parallel, it is best to share the same instructions for different samples in a batch for ease of parallelization. However, the number of ones in the mask may vary for different samples. In other words, in (\ref{eq:gpu_naive2}) enumerating the indices $j$ subject to $S_{ij}=1$ cannot be efficiently executed in a parallel way.

To address the problem, we introduce a notion called sub-batch. In a sub-batch, all samples have the same mask for ease of parallelization, and different sub-batches have different masks. Typically, a sub-batch contains 32 samples (due to the GPU underlying architecture). A batch size of 512 then contains 16 sub-batches. For each sub-batch, we first generate two index sequences $\vs^{\text{in}}\in [d_{\text{in}}]^{d_1}$ and $\vs^{\text{out}}\in [d_{\text{out}}]^{d_2}$ (the input mask and output mask respectively). We then fetch the input neurons using mask $\vs^{\text{in}}$, which corresponds to a gather operation (e.g. \texttt{torch.gather} in Pytorch). Denote the gathered input as $\widetilde{\mathbf X}\in \mathbb R^{32\times d_1}$. Then the output $\widetilde{\mathbf Z}\in \mathbb R^{32\times d_2}$ of this sub-batch can be calculated by
\begin{equation}
    \widetilde{ Z}_{ik}=\max_{j\in [d_1]} \sigma(\widetilde{X}_{ij}+B_{s^{\text{out}}_k, s^{\text{in}}_j})\quad i\in[32],k\in[d_2].
\end{equation}

We finally point out that the introduced sub-batch is only used to accelerate GPU implementation. If one does not care about the speed, simply using (\ref{eq:gpu_naive2}) for training leads to almost the same test performance.

\subsection{Inference}
For inference, the sorting operations must be calculated exactly. To speed up the computation, note that the weight $\vw$ in Proposition \ref{thm:dropout} is exponentially decayed, implying that a lot of elements are close to zero. This leads to the following results:
\begin{proposition}
\label{thm:inference}
Let $\vw\in \mathbb R^d$ be a vector with $w_i=(1-\rho)\rho^{i-1}$. Then for any vector $\vx\in\mathbb R_+^d$ with non-negative elements and any integer $k>0$,
\begin{equation}
    0\le \vw^{\mathrm{T}}\operatorname{sort}(\vx)-\sum_{i=1}^k w_i x_{(i)}\le \rho^k\|\vx\|_\infty.
\end{equation}
\end{proposition}
The proof of Proposition \ref{thm:inference} is trivial. Based on the above proposition, picking a $k$ much smaller than $d$ already suffices to control the error below the numerical precision. In this way, only a top-$k$ operation is needed. We choose $k=10$ for all experiments. Also note that zeroing out some elements of a weight vector does not contradict the Lipschitz property of the sort neuron, and thus margin-based certification still gives a correct certified accuracy.

\section{Training Details}
\label{sec:exp_details}
In this section we provide complete train details for readers to reproduce our results. Our experiments are implemented using the Pytorch framework.

\subsection{Datasets}
We consider four benchmark datasets: MNIST, CIFAR-10, TinyImageNet and ImageNet ($64\times 64$). The statistics of these datasets are listed below, including the number of classes, the dataset size, and the image size.
\begin{table}[ht]
    \centering
    \small
    \caption{Datasets used in this paper.}
    \vspace{2pt}
    \label{tab:dataset}
    \setlength\tabcolsep{4pt}
    \begin{tabular}{c|cccc}
    \hline
         Dataset & \# Classes & \begin{tabular}[c]{@{}c@{}}\# Images\\ (train)\end{tabular} & \begin{tabular}[c]{@{}c@{}}\# Images\\ (test)\end{tabular} & \begin{tabular}[c]{@{}c@{}}Image\\ size\end{tabular} \\ \hline
         MNIST & 10 & 60K & 10K & $28\times 28$\\
         CIFAR-10 & 10 & 50K & 10K & $32\times 32$\\
         TinyImageNet & 200 & 100K & 10K & $64\times 64$\\
         ImageNet & 1000 & 1.28M & 50K & $64\times 64$\\
    \hline
    \end{tabular}
    \vspace{-10pt}
\end{table}

For all datasets, we pre-process each training image using the same pipeline as previous works. Concretely, we use random horizontal flip (except for MNIST) and random crop as the data augmentation. Each image is padded to a larger size with $\texttt{padding\_mode=edge}$ and then randomly crop to the original size. The number of pixels padded is 1,3,5,4 for MNIST, CIFAR-10, TinyImageNet, and ImageNet, respectively. Finally, the images are normalized to have zero-mean and unit variance across the dataset. For testing, no data augmentation is performed.

\subsection{Models}
\label{sec:models}
This paper considers 3 types of models in total. The first is a simple fully-connected SortNet. The second is the combination of a Lipschitz SortNet backbone and a non-Lipschitz two-layer perceptron. In this way SortNet serves as a robust feature extractor and the top perceptron is used for classification. The last is a larger model similar to the second model, except we try to use a partially convolutional SortNet backbone for better performance. Details of these models are presented in Table \ref{tab:architecture}. For the ImageNet dataset, we increase the number of hidden neurons of the top perceptron to 2048, since the number of classes is 1000 (larger than the original hidden size which is only 512).

\begin{table}[ht]
    \centering
    \small
    \caption{Network architectures used in this paper. In the table, SortFC($n$) denotes a fully-connected SortNet layer with $n$ output neurons, and SortConv($n$, ker=$k$) denotes a convolutional SortNet layer with $n$ output channels and a kernel size of $k$. Padding is not used. We use s=$s$ to denote strided convolutions with step $s$, and by default s=1. FC($n$) is the fully-connected linear layer with $n$ output neurons. $K$ denotes the number of classes.}
    \vspace{2pt}
    \label{tab:architecture}
    \setlength\tabcolsep{2pt}
    \begin{tabular}{c|c|c}
    \Xhline{0.75pt}
    SortNet                                                             & SortNet+MLP                                                         & SortNet+MLP (2x)                                                    \\ \Xhline{0.6pt}
     \begin{tabular}[c]{@{}c@{}}SortFC(5120)\\ SortFC(5120)\\ SortFC(5120)\\ SortFC(5120)\\ SortFC(5120)\\ SortFC($K$)\end{tabular} & \begin{tabular}[c]{@{}c@{}}SortFC(5120)\\ SortFC(5120)\\ SortFC(5120)\\ SortFC(5120)\\ SortFC(5120)\\ FC(512, bias=True)\\ Tanh\\ FC($K$,  bias=True)\end{tabular} & \begin{tabular}[c]{@{}c@{}}SortConv(5120,ker=60,s=4)\\ SortConv(5120, ker=1)\\ SortConv(5120, ker=1)\\ SortConv(5120, ker=1)\\ SortConv(5120, ker=1)\\ FC(512, bias=True)\\ Tanh\\ FC($K$,  bias=True)\end{tabular} \\ \Xhline{0.75pt}
    \end{tabular}
\end{table}

We also list information of the model size in Table \ref{tab:model_size}. Roughly speaking, SortNet and SortNet+MLP have similar computational costs, and SortNet+MLP (2x) is about 4 times larger than SortNet+MLP. We also point out that the models used in this paper have exactly the same topology and number of parameters as $\ell_\infty$-distance Net (or $\ell_\infty$-distance Net+MLP) in \citep{zhang2021towards}, which enables a fair comparison between our approach and theirs.

\begin{table}[ht]
    \centering
    \small
    \caption{Statistics of model size on different datasets. In the table, ``FLOPs'' measures the number of floating-point operations in forward propagation, and ``\# Neurons'' denotes the number of hidden neurons. }
    \vspace{2pt}
    \label{tab:model_size}
    \setlength\tabcolsep{2pt}
    \begin{tabular}{c|c|ccc}
    \Xhline{0.75pt}
                                   & Model            & \begin{tabular}[c]{@{}c@{}}CIFAR\\ -10\end{tabular} & \begin{tabular}[c]{@{}c@{}}Tiny\\ ImageNet\end{tabular} & \begin{tabular}[c]{@{}c@{}}ImageNet\\ ($64\times 64$)\end{tabular} \\ \Xhline{0.6pt}
    \multirow{3}{*}{FLOPs}         & SortNet          & 121M    & --                                                      & --                                                                \\
                                   & SortNet+MLP      & 123M    & 171M                                                    & 180M                                                              \\
                                   & SortNet+MLP (2x) & --    & 651M                                                    & 684M                                                               \\ \hline
    \multirow{3}{*}{\# Parameters} & SortNet          & 121M     & --                                                      &  --                                                                \\
                                   & SortNet+MLP      & 123M     & 170M                                                    & 180M                                                               \\
                                   & SortNet+MLP (2x) & --     & 171M                                                    & 204M                                                              \\ \hline
    \multirow{3}{*}{\# Neurons}    & SortNet          & 25.6K    & --                                                     & --                                                                 \\
                                   & SortNet+MLP      & 26.1K    & 26.1K                                                    & 27.6K                                                               \\
                                   & SortNet+MLP (2x) & --    & 103K                                                   & 104K                                                               \\ \Xhline{0.75pt}
    \end{tabular}
\end{table}

\subsection{Loss function}
Let $(\vx,y)$ be a training sample and $\vf$ be the mapping represented by the network. For a SortNet model, since it is strictly 1-Lipschitz, we can directly apply hinge loss to increase the margin of output logits (see Proposition \ref{thm:lipschitz}). To make the training more effective, we adopt the following loss function:
\begin{equation}
\label{eq:loss}
\begin{aligned}
    \ell(\vf,\vx,y)=\lambda \cdot \ell_{\text{CE}}(s\cdot \vf(\vx),y) +\mathbb I\left(\mathop{\arg\max}_i [\vf(\vx)]_i=y\right) \ell_{\text{hinge}}(\vf(\vx)/\theta,y)
\end{aligned}
\end{equation}
where $\theta$ is the hinge threshold hyper-parameter, $s$ is a learnable scalar, $\lambda$ is the mixing hyper-parameter, $\mathbb I(\cdot)$ is the indicator function, and $\ell_{\text{CE}}$ and $\ell_{\text{hinge}}$ are the cross-entropy loss and standard hinge loss, respectively:
\begin{align}
    &\ell_{\text{CE}}(\vz,y)=\log\left(\sum_i\exp(z_i)\right)-z_y,\\ &\ell_{\text{hinge}}(\vz,y)=\max\{\max_{i\neq y} z_i-z_y+1, 0\}.
\end{align}
Intuitively speaking, this loss first uses cross entropy to improve the clean accuracy, then boost the certified robustness for samples that have been correctly classified. Such a loss design is quite common and is also used in prior works \citep{singla2022improved,zhang2022boosting,ding2020mma}. The coefficient $\lambda$ balances the two loss terms and controls the trade-off between accuracy and robustness. We find good results can be achieved by starting with an initial value $\lambda_0$ and slowly reducing it throughout training until it approaches zero (similar to \citep{zhang2020towards}).

Now consider the composite SortNet+MLP model. The whole model is non-Lipschitz, so we have to use relaxation-based methods to calculate a margin vector for the top MLP. We adopt the simplest method called \emph{interval bound propagation} (IBP) \citep{gowal2018effectiveness}, similar to \citep{zhang2021towards}. After obtaining the margin vector, we use the following loss function
\begin{equation}
\label{eq:loss_mlp}
\begin{aligned}
    \ell(\vf,\vx,y)=\lambda \cdot \ell_{\text{CE}}(s\cdot \vf(\vx),y) + \mathbb I\left(\mathop{\arg\max}_i [\vf(\vx)]_i=y\right) \ell_{\text{hinge}}({\text{IBP}}(\vf,\vx,y,\epsilon),y)
\end{aligned}
\end{equation}
which is almost the same as (\ref{eq:loss}). Here ${\text{IBP}}(\vf,\vx,y,\epsilon)$ denotes the margin calculated by interval bound propagation for sample $(\vx,y)$ and network $\vf$ under perturbation radius $\epsilon$. Similarly, the coefficient $\lambda$ decays from an initial value $\lambda_0$ to zero throughout training, and the value $\epsilon$ follows a warmup schedule in the training process. Such a strategy is broadly applied in previous literature \citep{gowal2018effectiveness,zhang2020towards,xu2020automatic}. See Appendix \ref{sec:hyper-paramters} for details of the $\lambda$ and $\epsilon$ schedule.

\subsection{Training strategy}
\textbf{Parameter initialization}. All parameters (in both the SortNet and the top MLP) are initialized using the standard Gaussian distribution. The learnable scalar in the loss function is initialized to be 1. To keep the output scale of the linear layer the same as the input scale, the output of each linear neuron is further multiplied by $1/\sqrt d$, where $d$ is the input dimension. Such an initialization strategy is often called the NTK initialization \citep{jacot2018neural}.

\textbf{Training pipeline}. At each iteration, we first generate masks in each \emph{hidden} layer and remove neurons with zero masks, resulting in a sub-network. This is similar to applying dropout in each hidden layer with a dropout rate of $\rho$ (we do not dropout the input). We then perform forward and backward propagation on the sub-network (corresponding to an $\ell_\infty$-distance net) following Section \ref{sec:dropout_training} and calculate the gradients of parameters. We use two additional tricks that is adopted in \citep{zhang2021towards}: batch normalization and $\ell_p$-relaxation.

\textbf{Batch normalization}. Since the output of the sort neuron (\ref{eq:sortnet_neuron}) is always non-negative due to the absolute-value activation function, a mean-shift version of batch normalization \citep{ioffe2015batch} is applied after each SortNet layers. Concretely, it normalizes each hidden neuron to have zero mean for a batch of samples in each iteration. We do not use scaling since it destroys the Lipschitz property. Unlike the original batch normalization, we do not calculate the running mean statistics in our training procedure since the model is also randomly sampled (with hidden neurons dropped out). The running mean can be calculated later after the training finishes, by conducting one pass of forward propagation on the whole training dataset for the inference model.

\textbf{$\ell_p$-relaxation}. The calculation of maximum in (\ref{eq:expectation_dropout}) leads to sparse gradient for input $\vx$, which makes optimization difficult. To alleviate the problem, we adopt $\ell_p$-relaxation proposed in \citep{zhang2021towards} for training, which gives a smooth approximation of the maximum operation. Concretely, for a vector $\vx$ with non-negative elements, the following holds:
\begin{equation}
    \max_i x_i = \lim_{p\to +\infty}\textstyle \left(\sum_i x_i^p\right)^{1/p}.
\end{equation}
At the initial phase, $p$ is set to a small value so that the gradients are non-sparse. As training progresses, $p$ gradually increases until reaching a large number. For the last few epochs, $p$ is set to infinity. In this paper, we exponentially increase $p$ from 8 to 1000 without tuning for all experiments, following \citep{zhang2021towards}.

\subsection{Details and hyper-parameters}
\label{sec:hyper-paramters}
All the experiments regarding SortNet models are run on NVIDIA Tesla V100S-PCIe GPUs. As for speed comparison to prior works, we further test time metrics for all prior works as well as our approach on the latest NVIDIA RTX-3090 GPUs (24GB memory). The CUDA version is 11.1.

\textbf{Hyper-parameters of the optimizer}. Following \citep{zhang2021towards,zhang2022boosting}, we adopt the Adam optimizer with hyper-parameters $\beta_1=0.9$, $\beta_2=0.99$ and $\epsilon=10^{-10}$, and use a batch size of 512 for all datasets and experiments. The learning rate is initialized to be 0.02 and is decayed using a simple cosine annealing throughout the whole training process. The only exception is for ImageNet, where we use 2 GPUs in parallel due to the large dataset size. This results in a total batch size of 1024, and we change the initial learning rate to be 0.01. In all experiments, the weight decay is set to 0.02 for all but SortNet parameters (e.g., for the top MLP), and we do not use weight decay for SortNet parameters. 

\textbf{$\epsilon$ Schedule}. For the SortNet+MLP model, a warmup over $\epsilon$ is required due to the use of interval bound propagation, and we adopt the same warmup schedule as in \citep{xu2020automatic} (with the same hyper-parameters), which increases $\epsilon_{\text{train}}$ first-exponentially-then-linearly from 0 to $1.1\epsilon_{\text{test}}$. 

\textbf{Epoch-related hyperparameters}. Such hyper-parameters are listed below, which depends on the dataset.
\begin{itemize}[topsep=0pt]
\setlength{\itemsep}{0pt}
    \item \textbf{MNIST}. We train all models for 1500 epochs. The $\ell_p$ schedule starts at the 100th epoch and ends at the 1450th epoch.
    \item \textbf{CIFAR-10}. We train all models for 3000 epochs. The $\ell_p$ schedule starts at the 200th epoch and ends at 2950th epoch. The warmup of $\epsilon$ is in the first 500 epochs.
    \item \textbf{TinyImageNet}. We train all models for 1000 epochs. The $\ell_p$ schedule starts at the 100th epoch and ends at the 980th epoch. The warmup of $\epsilon$ is in the first 500 epochs.
    \item \textbf{ImageNet ($64\times 64$)}. We train all models for 300 epochs. The $\ell_p$ schedule starts at the 50th epoch and ends at the 290th epoch. The warmup of $\epsilon$ is in the first 200 epochs.
\end{itemize}

\textbf{Hyper-parameters $\theta$, $\rho$ and $\lambda$}. For the SortNet model, there is a hinge threshold hyper-parameter $\theta$ that has to be tuned, and we choose $\theta$ using a course grid search. Concertely, we pick $\theta=0.6$ for $\epsilon=0.1$ and $\theta=0.9$ for $\epsilon=0.3$ on MNIST, and pick $\theta=16/255$ for $\epsilon=2/255$ and $\theta=48/255$ for $\epsilon=8/255$ on CIFAR-10. As for hyper-parameter $\rho$, we find $\rho=0.3$ typically works well on most settings, while a slight adjustment may achieve the best performance. On MNIST, we pick $\rho=0.3$ for $\epsilon=0.1$ and $\rho=0.25$ for $\epsilon=0.3$. On CIFAR-10, we pick $\rho=0.25$ for $\epsilon=2/255$ and $\rho=0.15$ for $\epsilon=8/255$ in SortNet, and pick $\rho=0.3$ for $\epsilon=2/255$ and $\rho=0.4$ for $\epsilon=8/255$ in SortNet+MLP. On Tiny-ImageNet and ImageNet $64\times 64$, we do not tune the hyper-parameter $\rho$ and simply pick the value 0.3. Sensitivity analysis of hyper-parameter $\rho$ is shown in Section \ref{sec:ablation}.

Finally, as for hyper-parameter $\lambda$ in loss (\ref{eq:loss}), we find choosing its value proportional to $1/\epsilon$ works well in all experiments. Such a choice is quite reasonable in balancing the cross-entropy term and the hinge term in the loss (\ref{eq:loss}) because the magnitude of the gradient of the hinge term $\ell_{\text{hinge}}(\vf(\vx)/\theta,y)$ is $\mathcal O(1/\theta)$ which is (roughly) inversely proportional to $\epsilon$. To make the value more regular, we just pick $\lambda_0$ from the set \{0.01, 0.02, 0.05, 0.1, 0.2, 0.5, 1.0\}, while maintaining the property of being inversely proportional to $\epsilon$. Throughout training, $\lambda$ is decayed from $\lambda_0$ to a vanishing small value $0.01\lambda_0$ exponentially, and is performed simultaneously with the $\ell_p$ relaxation process. The concrete values in different settings are listed below.
\begin{table}[ht]
    \centering
    \small
    \vspace{-5pt}
    \caption{The value of hyper-parameter $\lambda$ in loss (\ref{eq:loss}) used in this paper.}
    \label{tab:lambda}
    \vspace{2pt}
    \begin{tabular}{ccc}
    \hline
        Setting & $\lambda_0$ & $\lambda_{\text{end}}$\\ \hline
        MNIST ($\epsilon=0.1$) & 0.1 & 0.001 \\
        MNIST ($\epsilon=0.3$) & 0.02 & 0.0002\\
        CIFAR-10 ($\epsilon=2/255$) & 1.0 & 0.01 \\
        CIFAR-10 ($\epsilon=8/255$) & 0.2 & 0.002\\
    \hline
    \end{tabular}
    \vspace{-10pt}
\end{table}

As for $\lambda$ in loss (\ref{eq:loss_mlp}), we find the value do not depend on $\epsilon$, and in all settings we simply pick $\lambda_0=1.0$ and $\lambda_{\text{end}}=0.01$ without tuning.

\subsection{Reproducing baseline methods}
\label{sec:reproducing_baseline}
In this paper we also reproduce baseline methods by measuring the wall-clock time in training and certification. For all methods in Tables \ref{tab:results_mnist}, \ref{tab:results_cifar10} and \ref{tab:results_imagenet}, we test the computational cost on a single NVIDIA RTX-3090 GPU. We use the official Github codes and commands for each method whenever possible. In some situations, the batch size used in the original command may be too large to fit the 24GB GPU memory, so we reduce the batch size to match the memory capacity.
\begin{itemize}[topsep=0pt]
\setlength{\itemsep}{0pt}
    \item \textbf{IBP} \citep{gowal2018effectiveness}. We use the official code of \texttt{CROWN-IBP} on CIFAR-10 dataset since the original code may not be reproducible according to \citep{zhang2020towards}. Note that the corresponding numbers in Table \ref{tab:results_cifar10} is also obtained from the CROWN-IBP paper. We use the official code of \texttt{auto\_LiRPA} on TinyImageNet and ImageNet ($64\times 64$) datasets. The corresponding numbers of TinyImageNet dataset in Table \ref{tab:results_imagenet} are obtained from \citep{xu2020automatic}. Due to limited GPU memory, for TinyImageNet and ImageNet ($64\times 64$), we reduce the batch size to 64.
    \item \textbf{CROWN-IBP} \citep{zhang2020towards}. We use the official code of \texttt{CROWN-IBP} on CIFAR-10 dataset. Due to limited GPU memory, we reduce the batch size to 256. The training speed is measured in the $\epsilon$-warmup phase, and the certification speed is measured at $\epsilon=2/255$ where both IBP and linear relaxation is used in certification.
    \item \textbf{CROWN-IBP} \citep{xu2020automatic}. We use the official code of \texttt{auto\_LiRPA} on CIFAR-10, TinyImageNet, and ImageNet ($64\times 64$) datasets. Due to limited GPU memory, we reduce the batch size to 256 on CIFAR-10 and 32 on TinyImageNet and ImageNet. The training speed is measured in the $\epsilon$-warmup phase, and the certification is the simple IBP for these settings.
    \item \textbf{IBP} \citep{shi2021fast}. We use the official code of \texttt{Fast-Certified-Robust-Training}. Due to limited GPU memory, for TinyImageNet dataset we reduce the batch size to 64.
    \item \textbf{CAP} \citep{wong2018scaling}. We use the official code of \texttt{convex\_adversarial}. 
    \item \textbf{COLT} \citep{balunovic2020Adversarial}. We use the official code of \texttt{COLT}. The training has 4 stages, and the per-epoch training time (seconds) for each stage is 51, 240, 332, 385, respectively. Therefore the average time listed in Table \ref{tab:results_cifar10} is 252.0 seconds.
    \item \textbf{$\ell_\infty$-distance Net} \citep{zhang2021towards,zhang2022boosting}. We use the official code of \texttt{L\_inf-dist-net-v2}. The training speed is measured in the $\ell_p$-relaxation phase which is the slowest.
\end{itemize}
We also reproduce 100-PGD attack for the $\ell_\infty$-distance net in the original paper \citep{zhang2021towards}, because their paper only reported the robust accuracy under the (\emph{weaker}) 20-step PGD attack. Furthermore, we report two additional settings: $\epsilon=0.1$ on MNIST and $\epsilon=2/255$ on CIFAR-10 based on their Github repo \texttt{L\_inf-dist-net}, which are not presented in \citep{zhang2021towards}.

We also try to reproduce the result of $\ell_\infty$-distance net on TinyImageNet using the training approach proposed by \citep{zhang2022boosting}. We try our best to grid search over several hyper-parameters, including $\lambda_0$, $\lambda_{\text{end}}$ and $\theta$, while keeping other hyper-parameters the same as in \citep{zhang2022boosting}. The values are chosen from $\lambda_0\in \{0.05,0.1,0.2,0.5,1.0\}$, $\lambda_{\text{end}}\in \{0.0005,0.001,0.002,0.005,0.01\}$, and $\theta\in\{12/255,16/255,20/255,24/255\}$. We use 1000 training epochs, and the $\ell_p$-relaxation starts at the 20th epoch. Results are present in Table \ref{tab:results_imagenet}, where we can only achieve 11.04\% certified accuracy, which is much lower than other methods. We hypothesis that their approach may not suit for large-scale datasets with a huge number of classes.





\section{Ablation Studies}
\label{sec:ablation}
We finally make a further investigation of the performance of SortNet models by varying the value of $\rho$ and comparing it with $\ell_\infty$-distance net (corresponding to $\rho=0$), using the training approach proposed in Appendix \ref{sec:exp_details}.
Figure \ref{fig:result_for_rho} presents the performance of trained models, where we plot both clean accuracy and certified accuracy w.r.t. different choices of hyper-parameters  $\rho$ ranging from 0 to 0.5. It can be seen that for all four cases in Figure \ref{fig:result_for_rho}, a mid-range value consistently gives the best results. In particular, choosing $\rho=0.3$ typically attains the peak certified accuracy, and the improvement is most significant when comparing to the the extreme case of $\rho=0$ (often more than 5 points). Considering that the training in each figure is under the same configuration and hyper-parameters, the improvements justify that SortNet is a better Lipschitz model than $\ell_\infty$-distance net for certified $\ell_\infty$ robustness, due to the introduced full order statistics.

\begin{figure*}[ht]
    \centering
    \includegraphics[width=\textwidth]{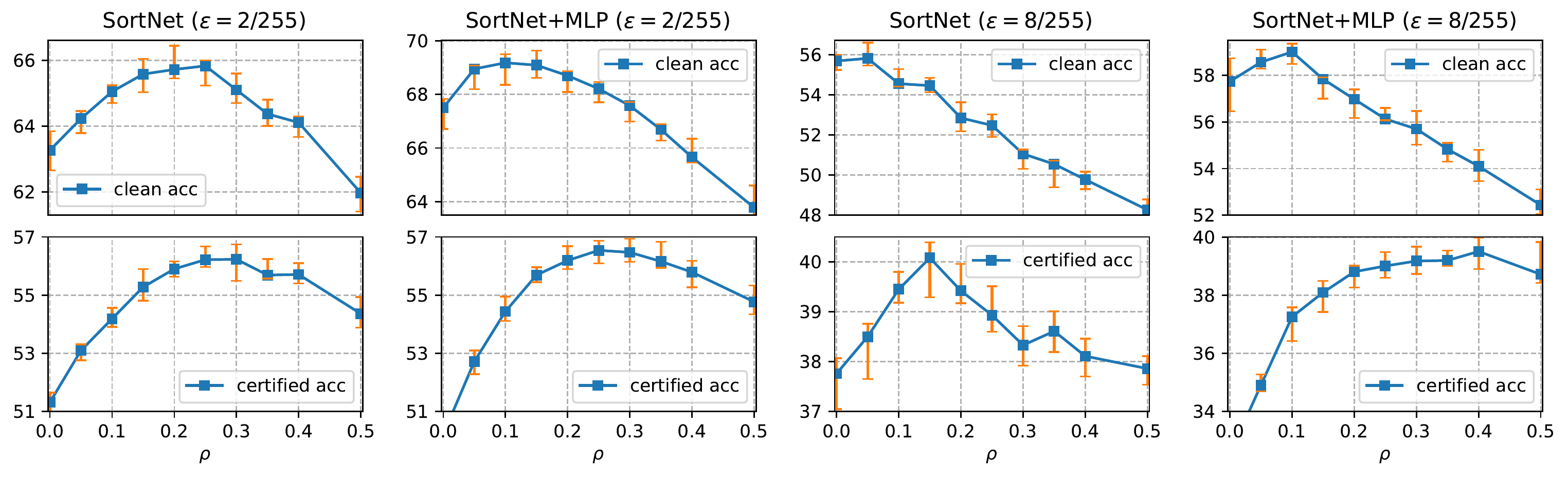}
    \caption{Performance of SortNet/SortNet+MLP trained with different hyper-paramters $\rho$ on CIFAR-10 dataset. For each $\rho$ we independently run 8 experiments and show the median value (blue square) as well as the best/worst performance (orange bar).}
    \label{fig:result_for_rho}
\end{figure*}

\section{Randomized Smoothing}
\label{sec:randomized_smoothing}
Different from other approaches in Table \ref{tab:results_cifar10}, randomized smoothing is a \emph{probabilistic} method that provides certified guarantees mainly for $\ell_2$ perturbations. To apply these methods in the $\ell_\infty$ perturbation case, a conversion of perturbation radius is performed by using norm inequalities. Specifically, to certify the robustness under $\ell_\infty$ perturbations with radius $\epsilon$, one can certify the robustness under $\ell_2$ with radius $\epsilon \sqrt d$ where $d$ is the input dimension. This clearly produces a lower bound estimate of the certified accuracy.

In the following table, we present the best-known results on MNIST and CIFAR-10. See \citep{salman2019provably,blum2020random,jeong2020consistency} as well as \citep[Appendix D]{zhang2022boosting} for references. In summary, randomized smoothing is competitive for the $\epsilon=2/255$ case on CIFAR-10. However, when $\epsilon$ is relatively large, randomized smoothing cannot achieve non-trivial performance.

\begin{table}[h]
\small
\caption{Results of randomized smoothing in different settings.}
\label{tab:results_randomized_smoothing}
\vspace{2pt}
\centering
\begin{tabular}{c|c|c|c|c}
\Xhline{0.75pt}
Dataset                                      & $\epsilon$ & Clean & Certified & Reference\\ \Xhline{0.6pt}
\multicolumn{1}{c|}{\multirow{2}{*}{MNIST}} & 0.1  & $\approx$92 & $\approx$12 & \citep{jeong2020consistency}\\
& 0.3  & -- & 10.0 & -- \\\hline
\multicolumn{1}{c|}{\multirow{2}{*}{CIFAR-10}} & 2/255  & 78.8 & 62.6 & \citep{blum2020random}\\
& 8/255  & 52.3 & 25.2 & \citep{jeong2020consistency} \\
\Xhline{0.75pt}
\end{tabular}
\end{table}



\begin{table}[t]
\centering
\small
\vspace{-5pt}
\setlength\tabcolsep{3pt}
\caption{Full results of clean accuracy over 8 independent runs, sorted in descending order.}
\label{tab:results_full_clean}
\begin{tabular}{c|c|c|cccccccc|c}
\Xhline{0.75pt}
Dataset                                                                           & $\epsilon$                 & Model           & \multicolumn{8}{c|}{Clean} & Median \\ \Xhline{0.6pt}
\multirow{2}{*}{MNIST}
  & 0.1 & SortNet         & \textbf{99.08} & 99.07 & 99.04 & 99.01 & 99.01 & 98.97 & 98.95 & 98.94 & 99.01\\
  & 0.3 & SortNet         & \textbf{98.68} & 98.67 & 98.59 & 98.59 & 98.59 & 98.57 & 98.53 & 98.46 & 98.59\\ \hline
\multirow{4}{*}{CIFAR-10}
  & \multirow{2}{*}{2/255}
        & SortNet         & \textbf{66.00} & 65.98 & 65.96 & 65.86 & 65.79 & 65.78 & 65.69 & 65.23 & 65.83 \\
  &     & SortNet+MLP     & \textbf{67.72} & 67.72 & 67.64 & 67.59 & 67.57 & 67.11 & 66.99 & 66.99 & 67.58 \\ \cline{2-12} 
  & \multirow{2}{*}{8/255}
        & SortNet         & \textbf{54.84} & 54.77 & 54.70 & 54.59 & 54.30 & 54.30 & 54.28 & 54.12 & 54.45 \\
  &     & SortNet+MLP     & \textbf{54.80} & 54.39 & 54.33 & 54.13 & 54.07 & 54.02 & 53.71 & 53.46 & 54.10 \\ \hline
\multirow{2}{*}{\begin{tabular}[c]{@{}c@{}}Tiny\\ ImageNet\end{tabular}}
  & \multirow{2}{*}{1/255}
        & SortNet+MLP     & \textbf{24.17} & 24.07 & 24.01 & 23.97 & 23.95 & 23.85 & 23.75 & 23.52 & 23.96 \\
  &     & SortNet+MLP(2x) & \textbf{25.77} & 25.69 & 25.55 & 25.43 & 25.33 & 25.30 & 25.26 & 25.02 & 25.38 \\ \hline
\multirow{2}{*}{\begin{tabular}[c]{@{}c@{}}ImageNet\\ $64\times 64$\end{tabular}}
  & \multirow{2}{*}{1/255}
        & SortNet+MLP     & \textbf{13.48} & 13.37 & 13.35 & 13.33 & 13.30 & 13.26 & 13.22 & 13.21 & 13.32 \\
  &     & SortNet+MLP(2x) & \textbf{14.96} & 14.87 & 14.86 & 14.84 & 14.82 & 14.79 & 14.76 & 14.71 & 14.83 \\ \Xhline{0.75pt}
\end{tabular}
\vspace{10pt}
\caption{Full results of certified accuracy over 8 independent runs, sorted in descending order.}
\label{tab:results_full_certified}
\begin{tabular}{c|c|c|cccccccc|c}
\Xhline{0.75pt}
Dataset                                                                           & $\epsilon$                 & Model           & \multicolumn{8}{c|}{Certified} & Median \\ \Xhline{0.6pt}
\multirow{2}{*}{MNIST}
  & 0.1 & SortNet         & \textbf{98.14} & 98.13 & 98.11 & 98.07 & 98.06 & 98.05 & 98.04 & 97.93 & 98.07 \\
  & 0.3 & SortNet         & \textbf{93.40} & 93.40 & 93.39 & 93.39 & 93.36 & 93.25 & 93.22 & 93.19 & 93.38 \\ \hline
\multirow{4}{*}{CIFAR-10}
  & \multirow{2}{*}{2/255}
        & SortNet         & \textbf{56.67} & 56.65 & 56.29 & 56.23 & 56.20 & 56.07 & 56.06 & 55.97 & 56.22 \\
  &     & SortNet+MLP     & \textbf{56.94} & 56.80 & 56.70 & 56.63 & 56.30 & 56.25 & 56.16 & 56.14 & 56.47 \\ \cline{2-12} 
  & \multirow{2}{*}{8/255}
        & SortNet         & \textbf{40.39} & 40.14 & 40.13 & 40.10 & 40.05 & 39.96 & 39.61 & 39.29 & 40.08 \\
  &     & SortNet+MLP     & \textbf{39.99} & 39.76 & 39.70 & 39.56 & 39.45 & 39.40 & 39.25 & 38.90 & 39.51 \\ \hline
\multirow{2}{*}{\begin{tabular}[c]{@{}c@{}}Tiny\\ ImageNet\end{tabular}}
  & \multirow{2}{*}{1/255}
        & SortNet+MLP     & \textbf{17.92} & 17.52 & 17.51 & 17.48 & 17.43 & 17.33 & 17.25 & 17.21 & 17.46 \\
  &     & SortNet+MLP(2x) & \textbf{18.18} & 17.74 & 17.68 & 17.64 & 17.64 & 17.62 & 17.59 & 17.50 & 17.64 \\ \hline
\multirow{2}{*}{\begin{tabular}[c]{@{}c@{}}ImageNet\\ $64\times 64$\end{tabular}}
  & \multirow{2}{*}{1/255}
        & SortNet+MLP     & \textbf{9.02} & 9.00 & 8.97 & 8.93 & 8.92 & 8.89 & 8.89 & 8.89 & 8.93 \\
  &     & SortNet+MLP(2x) & \textbf{9.54} & 9.45 & 9.43 & 9.41 & 9.41 & 9.41 & 9.35 & 9.35 & 9.41 \\ \Xhline{0.75pt}
\end{tabular}
\end{table}

\section{Full Results}
\label{sec:full_result}
As is pointed out above, for all SortNet models we run 8 set of experiments independently and report the median of the accuracy. In this section, we release full results of clean accuracy and certified accuracy in Tables \ref{tab:results_full_clean} and \ref{tab:results_full_certified}. The best achieved numbers are indicated in boldface.

\section{Training SortNet with Shorter Epochs}
In our main results, we use quite long training epochs in order to achieve the best performance. For example, the number of epochs is 3000 on CIFAR-10, which is larger than several prior works, in particular, the work of \citep{zhang2022boosting}. In this section, we consider reducing the training budget so that the total training time is less than \citep{zhang2022boosting}. A direct calculation shows that using 1800 epochs, the total training cost of SortNet will be less than \citep{zhang2022boosting}. So we simply change the number of epochs to 1800, \emph{without turning any other hyper-parameters}. Results show that SortNet can achieve 56.05\% certified accuracy for $\epsilon=2/255$, which is still significantly higher than \citep{zhang2022boosting} (54.12\% certified accuracy). For the $\epsilon=8/255$ case, SortNet can achieve 39.81\% certified accuracy, which is slightly lower than \citep{zhang2022boosting} (with a gap of -0.25\%). However, note that we do not use several training tricks, such as the special initialization strategy or the special weight decay adopted in \citep{zhang2022boosting,zhang2021towards}. On the other hand, we find that increasing the training budget of \citep{zhang2022boosting} to 3000 epochs does not help or even leads to a worse certified accuracy possibly due to overfitting.

\end{document}

%% file: math_commands.tex
\usepackage{amsmath,amsfonts,bm}

\def\eqref#1{equation~\ref{#1}}

\def\1{\bm{1}}

\def\eps{{\epsilon}}

\def\va{{\bm{a}}}
\def\vb{{\bm{b}}}
\def\vc{{\bm{c}}}

\def\ve{{\bm{e}}}
\def\vf{{\bm{f}}}

\def\vh{{\bm{h}}}

\def\vs{{\bm{s}}}

\def\vu{{\bm{u}}}
\def\vv{{\bm{v}}}
\def\vw{{\bm{w}}}
\def\vx{{\bm{x}}}
\def\vy{{\bm{y}}}
\def\vz{{\bm{z}}}

\DeclareMathAlphabet{\mathsfit}{\encodingdefault}{\sfdefault}{m}{sl}
\SetMathAlphabet{\mathsfit}{bold}{\encodingdefault}{\sfdefault}{bx}{n}













%% file: main.bbl
\begin{thebibliography}{10}

\bibitem{ajtai19830}
Mikl{\'o}s Ajtai, J{\'a}nos Koml{\'o}s, and Endre Szemer{\'e}di.
\newblock An 0 (n log n) sorting network.
\newblock In {\em Proceedings of the fifteenth annual ACM symposium on Theory
  of computing}, pages 1--9, 1983.

\bibitem{anil2019sorting}
Cem Anil, James Lucas, and Roger Grosse.
\newblock Sorting out {L}ipschitz function approximation.
\newblock In {\em International Conference on Machine Learning}, pages
  291--301, 2019.

\bibitem{balunovic2020Adversarial}
Mislav Balunovic and Martin Vechev.
\newblock Adversarial training and provable defenses: Bridging the gap.
\newblock In {\em International Conference on Learning Representations}, 2020.

\bibitem{biggio2013evasion}
Battista Biggio, Igino Corona, Davide Maiorca, Blaine Nelson, Nedim
  {\v{S}}rndi{\'c}, Pavel Laskov, Giorgio Giacinto, and Fabio Roli.
\newblock Evasion attacks against machine learning at test time.
\newblock In {\em Joint European conference on machine learning and knowledge
  discovery in databases}, pages 387--402. Springer, 2013.

\bibitem{blum2020random}
Avrim Blum, Travis Dick, Naren Manoj, and Hongyang Zhang.
\newblock Random smoothing might be unable to certify $\ell_\infty$ robustness
  for high-dimensional images.
\newblock {\em arXiv preprint arXiv:2002.03517}, 2020.

\bibitem{chernodub2016norm}
Artem Chernodub and Dimitri Nowicki.
\newblock Norm-preserving orthogonal permutation linear unit activation
  functions (oplu).
\newblock {\em arXiv preprint arXiv:1604.02313}, 2016.

\bibitem{chrabaszcz2017downsampled}
Patryk Chrabaszcz, Ilya Loshchilov, and Frank Hutter.
\newblock A downsampled variant of imagenet as an alternative to the cifar
  datasets.
\newblock {\em arXiv preprint arXiv:1707.08819}, 2017.

\bibitem{cisse2017parseval}
Moustapha Cisse, Piotr Bojanowski, Edouard Grave, Yann Dauphin, and Nicolas
  Usunier.
\newblock Parseval networks: Improving robustness to adversarial examples.
\newblock In {\em International Conference on Machine Learning}, pages
  854--863, 2017.

\bibitem{cohen2019certified}
Jeremy Cohen, Elan Rosenfeld, and Zico Kolter.
\newblock Certified adversarial robustness via randomized smoothing.
\newblock In {\em International Conference on Machine Learning}, pages
  1310--1320. PMLR, 2019.

\bibitem{cohen2019universal}
Jeremy~EJ Cohen, Todd Huster, and Ra~Cohen.
\newblock Universal {L}ipschitz approximation in bounded depth neural networks.
\newblock {\em arXiv preprint arXiv:1904.04861}, 2019.

\bibitem{croce2019provable}
Francesco Croce, Maksym Andriushchenko, and Matthias Hein.
\newblock Provable robustness of {ReLU} networks via maximization of linear
  regions.
\newblock In {\em the 22nd International Conference on Artificial Intelligence
  and Statistics}, pages 2057--2066. PMLR, 2019.

\bibitem{cybenko1989approximation}
George Cybenko.
\newblock Approximation by superpositions of a sigmoidal function.
\newblock {\em Mathematics of control, signals and systems}, 2(4):303--314,
  1989.

\bibitem{devlin2019bert}
Jacob Devlin, Ming-Wei Chang, Kenton Lee, and Kristina Toutanova.
\newblock Bert: Pre-training of deep bidirectional transformers for language
  understanding.
\newblock In {\em NAACL-HLT (1)}, 2019.

\bibitem{ding2020mma}
Gavin~Weiguang Ding, Yash Sharma, Kry Yik~Chau Lui, and Ruitong Huang.
\newblock {MMA} training: Direct input space margin maximization through
  adversarial training.
\newblock In {\em International Conference on Learning Representations}, 2020.

\bibitem{krishnamurthy2018dual}
Krishnamurthy Dvijotham, Robert Stanforth, Sven Gowal, Timothy Mann, and
  Pushmeet Kohli.
\newblock A dual approach to scalable verification of deep networks.
\newblock {\em arXiv preprint arXiv:1803.06567}, 2018.

\bibitem{dvijotham2020efficient}
Krishnamurthy~Dj Dvijotham, Robert Stanforth, Sven Gowal, Chongli Qin, Soham
  De, and Pushmeet Kohli.
\newblock Efficient neural network verification with exactness
  characterization.
\newblock In {\em Uncertainty in Artificial Intelligence}, pages 497--507.
  PMLR, 2020.

\bibitem{farnia2019generalizable}
Farzan Farnia, Jesse Zhang, and David Tse.
\newblock Generalizable adversarial training via spectral normalization.
\newblock In {\em International Conference on Learning Representations}, 2019.

\bibitem{fazlyab2019efficient}
Mahyar Fazlyab, Alexander Robey, Hamed Hassani, Manfred Morari, and George
  Pappas.
\newblock Efficient and accurate estimation of {L}ipschitz constants for deep
  neural networks.
\newblock {\em Advances in Neural Information Processing Systems}, 32, 2019.

\bibitem{gehr2018ai2}
Timon Gehr, Matthew Mirman, Dana Drachsler-Cohen, Petar Tsankov, Swarat
  Chaudhuri, and Martin Vechev.
\newblock {AI2}: Safety and robustness certification of neural networks with
  abstract interpretation.
\newblock In {\em 2018 IEEE Symposium on Security and Privacy (SP)}, pages
  3--18. IEEE, 2018.

\bibitem{gouk2018regularisation}
Henry Gouk, Eibe Frank, Bernhard Pfahringer, and Michael Cree.
\newblock Regularisation of neural networks by enforcing {L}ipschitz
  continuity.
\newblock {\em arXiv preprint arXiv:1804.04368}, 2018.

\bibitem{gowal2018effectiveness}
Sven Gowal, Krishnamurthy Dvijotham, Robert Stanforth, Rudy Bunel, Chongli Qin,
  Jonathan Uesato, Relja Arandjelovic, Timothy Mann, and Pushmeet Kohli.
\newblock On the effectiveness of interval bound propagation for training
  verifiably robust models.
\newblock {\em arXiv preprint arXiv:1810.12715}, 2018.

\bibitem{he2016deep}
Kaiming He, Xiangyu Zhang, Shaoqing Ren, and Jian Sun.
\newblock Deep residual learning for image recognition.
\newblock In {\em Proceedings of the IEEE conference on computer vision and
  pattern recognition}, pages 770--778, 2016.

\bibitem{huang2021training}
Yujia Huang, Huan Zhang, Yuanyuan Shi, J~Zico Kolter, and Anima Anandkumar.
\newblock Training certifiably robust neural networks with efficient local
  {L}ipschitz bounds.
\newblock {\em Advances in Neural Information Processing Systems}, 34, 2021.

\bibitem{huster2018limitations}
Todd Huster, Cho-Yu~Jason Chiang, and Ritu Chadha.
\newblock Limitations of the {L}ipschitz constant as a defense against
  adversarial examples.
\newblock In {\em Joint European Conference on Machine Learning and Knowledge
  Discovery in Databases}, pages 16--29. Springer, 2018.

\bibitem{ioffe2015batch}
Sergey Ioffe and Christian Szegedy.
\newblock Batch normalization: Accelerating deep network training by reducing
  internal covariate shift.
\newblock In {\em International Conference on Machine Learning}, pages
  448--456, 2015.

\bibitem{jacot2018neural}
Arthur Jacot, Franck Gabriel, and Clement Hongler.
\newblock Neural tangent kernel: Convergence and generalization in neural
  networks.
\newblock In {\em Advances in Neural Information Processing Systems},
  volume~31, 2018.

\bibitem{jeong2020consistency}
Jongheon Jeong and Jinwoo Shin.
\newblock Consistency regularization for certified robustness of smoothed
  classifiers.
\newblock {\em Advances in Neural Information Processing Systems}, 33, 2020.

\bibitem{katz2017reluplex}
Guy Katz, Clark Barrett, David~L Dill, Kyle Julian, and Mykel~J Kochenderfer.
\newblock Reluplex: An efficient {SMT} solver for verifying deep neural
  networks.
\newblock In {\em International Conference on Computer Aided Verification},
  pages 97--117. Springer, 2017.

\bibitem{krizhevsky2009learning}
Alex Krizhevsky et~al.
\newblock Learning multiple layers of features from tiny images.
\newblock {\em Technical Report TR-2009}, 2009.

\bibitem{kumar2020curse}
Aounon Kumar, Alexander Levine, Tom Goldstein, and Soheil Feizi.
\newblock Curse of dimensionality on randomized smoothing for certifiable
  robustness.
\newblock In {\em International Conference on Machine Learning}, pages
  5458--5467. PMLR, 2020.

\bibitem{latorre2020lipschitz}
Fabian Latorre, Paul Rolland, and Volkan Cevher.
\newblock Lipschitz constant estimation of neural networks via sparse
  polynomial optimization.
\newblock In {\em International Conference on Learning Representations}, 2020.

\bibitem{le2015tiny}
Ya~Le and Xuan Yang.
\newblock Tiny imagenet visual recognition challenge.
\newblock {\em CS 231N}, 7(7):3, 2015.

\bibitem{lecun1998mnist}
Yann LeCun, Corinna Cortes, and Chris Burges.
\newblock {MNIST} handwritten digit database, 1998.

\bibitem{lecuyer2019certified}
Mathias Lecuyer, Vaggelis Atlidakis, Roxana Geambasu, Daniel Hsu, and Suman
  Jana.
\newblock Certified robustness to adversarial examples with differential
  privacy.
\newblock In {\em 2019 IEEE Symposium on Security and Privacy (SP)}, pages
  656--672. IEEE, 2019.

\bibitem{lee2020lipschitz}
Sungyoon Lee, Jaewook Lee, and Saerom Park.
\newblock Lipschitz-certifiable training with a tight outer bound.
\newblock {\em Advances in Neural Information Processing Systems}, 33, 2020.

\bibitem{leino21gloro}
Klas Leino, Zifan Wang, and Matt Fredrikson.
\newblock Globally-robust neural networks.
\newblock In {\em International Conference on Machine Learning (ICML)}, 2021.

\bibitem{leshno1993multilayer}
Moshe Leshno, Vladimir~Ya Lin, Allan Pinkus, and Shimon Schocken.
\newblock Multilayer feedforward networks with a nonpolynomial activation
  function can approximate any function.
\newblock {\em Neural networks}, 6(6):861--867, 1993.

\bibitem{li2019certified}
Bai Li, Changyou Chen, Wenlin Wang, and Lawrence Carin.
\newblock Certified adversarial robustness with additive noise.
\newblock {\em Advances in Neural Information Processing Systems},
  32:9464--9474, 2019.

\bibitem{li2019preventing}
Qiyang Li, Saminul Haque, Cem Anil, James Lucas, Roger~B Grosse, and
  Joern-Henrik Jacobsen.
\newblock Preventing gradient attenuation in {L}ipschitz constrained
  convolutional networks.
\newblock {\em Advances in Neural Information Processing Systems},
  32:15390--15402, 2019.

\bibitem{meunier2022dynamical}
Laurent Meunier, Blaise Delattre, Alexandre Araujo, and Alexandre Allauzen.
\newblock A dynamical system perspective for {L}ipschitz neural networks.
\newblock {\em arxiv:2110.12690}, 2022.

\bibitem{mirman2021fundamental}
Matthew Mirman, Maximilian Baader, and Martin Vechev.
\newblock The fundamental limits of interval arithmetic for neural networks.
\newblock {\em arXiv preprint arXiv:2112.05235}, 2021.

\bibitem{mirman2018differentiable}
Matthew Mirman, Timon Gehr, and Martin Vechev.
\newblock Differentiable abstract interpretation for provably robust neural
  networks.
\newblock In {\em International Conference on Machine Learning}, pages
  3578--3586. PMLR, 2018.

\bibitem{neumayer2022approximation}
Sebastian Neumayer, Alexis Goujon, Pakshal Bohra, and Michael Unser.
\newblock Approximation of {L}ipschitz functions using deep spline neural
  networks.
\newblock {\em arXiv preprint arXiv:2204.06233}, 2022.

\bibitem{o2007approximation}
Ryan O’Donnell and Karl Wimmer.
\newblock Approximation by {DNF}: examples and counterexamples.
\newblock In {\em International Colloquium on Automata, Languages, and
  Programming}, pages 195--206. Springer, 2007.

\bibitem{pinkus1999approximation}
Allan Pinkus.
\newblock Approximation theory of the {MLP} model in neural networks.
\newblock {\em Acta numerica}, 8:143--195, 1999.

\bibitem{qian2019lnonexpansive}
Haifeng Qian and Mark~N. Wegman.
\newblock L2-nonexpansive neural networks.
\newblock In {\em International Conference on Learning Representations}, 2019.

\bibitem{raghunathan2018certified}
Aditi Raghunathan, Jacob Steinhardt, and Percy Liang.
\newblock Certified defenses against adversarial examples.
\newblock In {\em International Conference on Learning Representations}, 2018.

\bibitem{raghunathan2018semidefinite}
Aditi Raghunathan, Jacob Steinhardt, and Percy~S. Liang.
\newblock Semidefinite relaxations for certifying robustness to adversarial
  examples.
\newblock In {\em Advances in Neural Information Processing Systems}, pages
  10900--10910, 2018.

\bibitem{salman2019provably}
Hadi Salman, Greg Yang, Jerry Li, Pengchuan Zhang, Huan Zhang, Ilya
  Razenshteyn, and S{\'e}bastien Bubeck.
\newblock Provably robust deep learning via adversarially trained smoothed
  classifiers.
\newblock In {\em Proceedings of the 33rd International Conference on Neural
  Information Processing Systems}, pages 11292--11303, 2019.

\bibitem{salman2019convex}
Hadi Salman, Greg Yang, Huan Zhang, Cho-Jui Hsieh, and Pengchuan Zhang.
\newblock A convex relaxation barrier to tight robustness verification of
  neural networks.
\newblock {\em Advances in Neural Information Processing Systems},
  32:9835--9846, 2019.

\bibitem{shi2022efficient}
Zhouxing Shi, Yihan Wang, Huan Zhang, Zico Kolter, and Cho-Jui Hsieh.
\newblock Efficiently computing local {L}ipschitz constants of neural networks
  via bound propagation.
\newblock In {\em Advances in Neural Information Processing Systems},
  volume~35, 2022.

\bibitem{shi2021fast}
Zhouxing Shi, Yihan Wang, Huan Zhang, Jinfeng Yi, and Cho-Jui Hsieh.
\newblock Fast certified robust training with short warmup.
\newblock {\em Advances in Neural Information Processing Systems}, 34, 2021.

\bibitem{singh2018fast}
Gagandeep Singh, Timon Gehr, Matthew Mirman, Markus P{\"u}schel, and Martin
  Vechev.
\newblock Fast and effective robustness certification.
\newblock In {\em Advances in Neural Information Processing Systems}, pages
  10802--10813, 2018.

\bibitem{singla2021skew}
Sahil Singla and Soheil Feizi.
\newblock Skew orthogonal convolutions.
\newblock In {\em International Conference on Machine Learning}, volume 139,
  pages 9756--9766. PMLR, 2021.

\bibitem{singla2022improved}
Sahil Singla, Surbhi Singla, and Soheil Feizi.
\newblock Improved deterministic l2 robustness on {CIFAR}-10 and {CIFAR}-100.
\newblock In {\em International Conference on Learning Representations}, 2022.

\bibitem{srivastava2014dropout}
Nitish Srivastava, Geoffrey Hinton, Alex Krizhevsky, Ilya Sutskever, and Ruslan
  Salakhutdinov.
\newblock Dropout: a simple way to prevent neural networks from overfitting.
\newblock {\em The journal of machine learning research}, 15(1):1929--1958,
  2014.

\bibitem{szegedy2013intriguing}
Christian Szegedy, Wojciech Zaremba, Ilya Sutskever, Joan Bruna, Dumitru Erhan,
  Ian Goodfellow, and Rob Fergus.
\newblock Intriguing properties of neural networks.
\newblock {\em arXiv preprint arXiv:1312.6199}, 2013.

\bibitem{tanielian2021approximating}
Ugo Tanielian and Gerard Biau.
\newblock Approximating {L}ipschitz continuous functions with groupsort neural
  networks.
\newblock In {\em International Conference on Artificial Intelligence and
  Statistics}, pages 442--450. PMLR, 2021.

\bibitem{tjeng2019evaluating}
Vincent Tjeng, Kai~Y Xiao, and Russ Tedrake.
\newblock Evaluating robustness of neural networks with mixed integer
  programming.
\newblock In {\em International Conference on Learning Representations}, 2019.

\bibitem{trockman2021orthogonalizing}
Asher Trockman and J~Zico Kolter.
\newblock Orthogonalizing convolutional layers with the cayley transform.
\newblock In {\em International Conference on Learning Representations}, 2021.

\bibitem{tsuzuku2018lipschitz}
Yusuke Tsuzuku, Issei Sato, and Masashi Sugiyama.
\newblock Lipschitz-margin training: Scalable certification of perturbation
  invariance for deep neural networks.
\newblock {\em Advances in Neural Information Processing Systems},
  31:6541--6550, 2018.

\bibitem{wang2018efficient}
Shiqi Wang, Kexin Pei, Justin Whitehouse, Junfeng Yang, and Suman Jana.
\newblock Efficient formal safety analysis of neural networks.
\newblock In {\em Advances in Neural Information Processing Systems},
  volume~31, 2018.

\bibitem{wang2021beta}
Shiqi Wang, Huan Zhang, Kaidi Xu, Xue Lin, Suman Jana, Cho-Jui Hsieh, and
  J~Zico Kolter.
\newblock Beta-crown: Efficient bound propagation with per-neuron split
  constraints for neural network robustness verification.
\newblock {\em Advances in Neural Information Processing Systems}, 34, 2021.

\bibitem{weng2018towards}
Lily Weng, Huan Zhang, Hongge Chen, Zhao Song, Cho-Jui Hsieh, Luca Daniel,
  Duane Boning, and Inderjit Dhillon.
\newblock Towards fast computation of certified robustness for {ReLU} networks.
\newblock In {\em International Conference on Machine Learning}, pages
  5276--5285. PMLR, 2018.

\bibitem{wong2018provable}
Eric Wong and Zico Kolter.
\newblock Provable defenses against adversarial examples via the convex outer
  adversarial polytope.
\newblock In {\em International Conference on Machine Learning}, pages
  5286--5295. PMLR, 2018.

\bibitem{wong2018scaling}
Eric Wong, Frank~R Schmidt, Jan~Hendrik Metzen, and J~Zico Kolter.
\newblock Scaling provable adversarial defenses.
\newblock In {\em Proceedings of the 32nd International Conference on Neural
  Information Processing Systems}, pages 8410--8419, 2018.

\bibitem{wu2021completing}
Yihan Wu, Aleksandar Bojchevski, Aleksei Kuvshinov, and Stephan G{\"u}nnemann.
\newblock Completing the picture: Randomized smoothing suffers from the curse
  of dimensionality for a large family of distributions.
\newblock In {\em International Conference on Artificial Intelligence and
  Statistics}, pages 3763--3771. PMLR, 2021.

\bibitem{xiao2019training}
Kai~Y Xiao, Vincent Tjeng, Nur Muhammad~Mahi Shafiullah, and Aleksander Madry.
\newblock Training for faster adversarial robustness verification via inducing
  {ReLU} stability.
\newblock In {\em International Conference on Learning Representations}, 2019.

\bibitem{xu2020automatic}
Kaidi Xu, Zhouxing Shi, Huan Zhang, Yihan Wang, Kai-Wei Chang, Minlie Huang,
  Bhavya Kailkhura, Xue Lin, and Cho-Jui Hsieh.
\newblock Automatic perturbation analysis for scalable certified robustness and
  beyond.
\newblock {\em Advances in Neural Information Processing Systems}, 33, 2020.

\bibitem{yang2020randomized}
Greg Yang, Tony Duan, J~Edward Hu, Hadi Salman, Ilya Razenshteyn, and Jerry Li.
\newblock Randomized smoothing of all shapes and sizes.
\newblock In {\em International Conference on Machine Learning}, pages
  10693--10705. PMLR, 2020.

\bibitem{yoshida2017spectral}
Yuichi Yoshida and Takeru Miyato.
\newblock Spectral norm regularization for improving the generalizability of
  deep learning.
\newblock {\em arXiv preprint arXiv:1705.10941}, 2017.

\bibitem{zhai2019macer}
Runtian Zhai, Chen Dan, Di~He, Huan Zhang, Boqing Gong, Pradeep Ravikumar,
  Cho-Jui Hsieh, and Liwei Wang.
\newblock {MACER}: Attack-free and scalable robust training via maximizing
  certified radius.
\newblock In {\em International Conference on Learning Representations}, 2020.

\bibitem{zhang2021towards}
Bohang Zhang, Tianle Cai, Zhou Lu, Di~He, and Liwei Wang.
\newblock Towards certifying {L}-infinity robustness using neural networks with
  {L}-inf-dist neurons.
\newblock In {\em International Conference on Machine Learning}, pages
  12368--12379. PMLR, 2021.

\bibitem{zhang2022boosting}
Bohang Zhang, Du~Jiang, Di~He, and Liwei Wang.
\newblock Boosting the certified robustness of {L}-infinity distance nets.
\newblock In {\em International Conference on Learning Representations}, 2022.

\bibitem{zhang2020towards}
Huan Zhang, Hongge Chen, Chaowei Xiao, Sven Gowal, Robert Stanforth, Bo~Li,
  Duane Boning, and Cho-Jui Hsieh.
\newblock Towards stable and efficient training of verifiably robust neural
  networks.
\newblock In {\em International Conference on Learning Representations}, 2020.

\bibitem{zhang2018efficient}
Huan Zhang, Tsui-Wei Weng, Pin-Yu Chen, Cho-Jui Hsieh, and Luca Daniel.
\newblock Efficient neural network robustness certification with general
  activation functions.
\newblock In {\em Advances in Neural Information Processing Systems},
  volume~31, pages 4939--4948, 2018.

\end{thebibliography}
